\documentclass[11pt]{article} 

\usepackage{amsmath, amsfonts, amssymb, amsthm,  graphicx, enumerate, bbm}
\graphicspath{{./figures/}}

\usepackage[letterpaper, left=1truein, right=1truein, top = 1truein, bottom = 1truein]{geometry}

\usepackage{bbold}

\usepackage[numbers, square]{natbib} 
\usepackage{comment}

\usepackage[usenames]{xcolor}

\usepackage{tikz} 
\usetikzlibrary{arrows}

\usepackage[
            CJKbookmarks=true,
            bookmarksnumbered=true, 
            bookmarksopen=true,
            colorlinks=true,
            citecolor=red,
            linkcolor=blue,
            anchorcolor=red,
            urlcolor=blue
            ]{hyperref}

\usepackage[ruled, vlined, lined, commentsnumbered]{algorithm2e}
\usepackage{algorithmic}

\usepackage{prettyref,soul,xspace}
\usepackage{booktabs}
\usepackage{multirow}

\usepackage[final]{showlabels}

\usepackage[T1]{fontenc}
\usepackage{lmodern}

\newcommand{\iid}{i.i.d.\xspace}

\newcommand{\rank}{\mathop{\sf rank}}

\newcommand{\diag}{\mathop{\text{diag}}}

\newcommand{\Expect}{\mathbb{E}}
\newcommand{\expect}[1]{\mathbb{E} #1 }

\newcommand{\Prob}{\mathbb{P}}

\newcommand{\Bern}{\text{Bern}}

\newcommand{\argmin}{\mathop{\rm arg\, min}}
\newcommand{\argmax}{\mathop{\rm arg\, max}}

\newrefformat{eq}{(\ref{#1})}
\newrefformat{chap}{Chapter~\ref{#1}}
\newrefformat{sec}{Section~\ref{#1}}
\newrefformat{algo}{Algorithm~\ref{#1}}
\newrefformat{fig}{Fig.~\ref{#1}}
\newrefformat{tab}{Table~\ref{#1}}
\newrefformat{rmk}{Remark~\ref{#1}}
\newrefformat{clm}{Claim~\ref{#1}}
\newrefformat{def}{Definition~\ref{#1}}
\newrefformat{cor}{Corollary~\ref{#1}}
\newrefformat{lmm}{Lemma~\ref{#1}}
\newrefformat{lemma}{Lemma~\ref{#1}}
\newrefformat{prop}{Proposition~\ref{#1}}
\newrefformat{app}{Appendix~\ref{#1}}
\newrefformat{ex}{Example~\ref{#1}}
\newrefformat{exer}{Exercise~\ref{#1}}
\newrefformat{soln}{Solution~\ref{#1}}

\newcommand{\lunder}[1]{{\underset{\raise0.3em\hbox{$\smash{\scriptscriptstyle-}$}}{#1}}}

\newcommand{\ceil}[1]{{\left\lceil {#1} \right \rceil}}

\newcommand{\norm}[1]{\left\|{#1} \right\|}

\newcommand{\fnorm}[1]{\|#1\|_{\rm F}}

\newcommand{\indc}[1]{{\mathbf{1}({#1})}}
\newcommand{\Indc}{\mathbf{1}}

\def\innergetnumber#1[#2]#3{#2}
\def\getnumber{\expandafter\innergetnumber\jobname}

\newcommand{\bbC}{{\mathbb{C}}}
\newcommand{\bbD}{{\mathbb{D}}}
\newcommand{\bbE}{{\mathbb{E}}}
\newcommand{\bbF}{{\mathbb{F}}}

\newcommand{\bbP}{{\mathbb{P}}}

\newcommand{\bbR}{{\mathbb{R}}}

\newcommand{\calB}{{\mathcal{B}}}
\newcommand{\calC}{{\mathcal{C}}}

\newcommand{\calL}{{\mathcal{L}}}
\newcommand{\calM}{{\mathcal{M}}}

\newcommand{\calP}{{\mathcal{P}}}

\newcommand{\hsigma}{{\hat{\sigma}}}

\newcommand{\sth}[1]{\left\{ #1 \right\}}

\newcommand{\wt}{\widetilde}
\newcommand{\wh}{\widehat}

\newcommand{\sss}{\scriptscriptstyle}

\newcommand{\alphaavg}{\overline{\alpha}}
\newcommand{\omegabar}{{\omega'}}
\newcommand{\indexset}{{J}}
\newcommand{\kmbound}{{\varepsilon}}

\newcommand{\countrate}{\nu_n}
\newcommand{\countrateupper}{\overline{\nu}_n^\epsilon}
\newcommand{\countratelower}{\underline{\nu}_n^\epsilon}
\newcommand{\countratelowereps}[1]{\underline{\nu}_n^{#1}}
\newcommand{\countrateuppereps}[1]{\overline{\nu}_n^{#1}}
\newcommand{\speclore}{{{SpecLoRe}}\xspace}
\newcommand{\omegarate}{\epsilon_1}

\newtheorem{theorem}{Theorem}[section]
\newtheorem{lemma}[]{Lemma}[section]
\newtheorem{proposition}[]{Proposition}[section]

\newtheorem{assumption}[]{Assumption}[section]

\theoremstyle{definition}

\newtheorem{remark}[]{Remark}[section] 

\DeclareMathOperator*{\ave}{ave}

\newcommand{\hs}[1]{{\color{red} #1}\xspace}
\renewcommand{\hs}[1]{{ #1}\xspace}

\title{Community detection in sparse latent space models}
\author{Fengnan Gao\footnote{F.G.~is at School of Data Science of Fudan University and Shanghai Center for Mathematical Sciences. E-mail: fngao@fudan.edu.cn.}
\and Zongming Ma\footnote{Z.M.~is at the Department of Statistics of the University of Pennsylvania. E-mail: zongming@wharton.upenn.edu.}
\and Hongsong Yuan\footnote{H.Y.~is at the School of Information Management and Engineering of Shanghai University of Finance and Economics.
E-mail: yuan.hongsong@shufe.edu.cn.}}
\date{\today}

\begin{document}  

\maketitle

\begin{abstract}
We show that a simple community detection algorithm originated from stochastic blockmodel literature achieves consistency, and even optimality, for a broad and flexible class of sparse latent space models.
The class of models includes latent eigenmodels \cite{hoff2008modeling}.
The community detection algorithm is based on spectral clustering followed by local refinement via normalized edge counting. 
\\

\end{abstract}



\section{Introduction}  

Network is a prevalent form of relational data.
%
A central theme in learning network data is community detection \cite{goldenberg2010survey,fortunato2010community}.
Community detection seeks to partition the nodes of a network into several disjoint subsets (a.k.a.~communities) upon observing the adjacency matrix \cite{girvan2002community}.
The underlying assumption is that nodes within the same community share some commonalities in their connection patterns.
To understand and to motivate algorithms for community detection, statisticians, probabilists and theoretical computer scientists have studied stochastic blockmodels (SBMs) \cite{holland1983} extensively.
To date, researchers have obtained a thorough understanding of the fundamental limits and the behavior of various algorithms under SBMs.
For more details, we refer interested readers to review papers \cite{abbe2017community,li2018convex,gao2018minimax} and the references therein.
A major shortcoming of SBMs is that nodes within the same community must have exactly the same degree profile, and hence SBMs cannot model degree heterogeneity which is commonly observed in real world networks.
To mitigate this issue, 
researchers have proposed degree-corrected blockmodels (DCBMs) where an extra sequence of degree correction parameters was used to lend more flexibility to individual node degrees \cite{karrer2011stochastic}.
In the regimes of strong consistency (when perfect recovery of community structure is possible) and weak consistency (when perfect recovery except for a vanishing proportion of nodes is possible), it is known that spectral clustering followed by certain local algorithm could achieve the best possible accuracy \cite{abbe2017community,gao2018minimax}.


In a separate line of literature, statisticians have proposed and studied a class of network models called latent space models \cite{hoff2002latent,hoff2003random,handcock2007model,hoff2008modeling,krivitsky2009representing}.
We may view this class of models as a natural extension of generalized linear models to network setting.
In this paper, we consider the following generative model for entries of the observed adjacency matrix $A$.
For any positive integer $m$, let $[m] = \{1,\dots, m\}$.
First, we exclude self-loops and so $A_{ii} = 0$ for all $i\in [n]$.
In addition, { conditional on unobserved values of $\{\alpha_i\}_{i=1}^n$ and $\{z_i\}_{i=1}^n$}, we assume that the Bernoulli random variables $\{A_{ij} = A_{ji}: 1\leq i<j\leq n \}$ are mutually independent, and for each pair $i< j$,
\begin{equation}
\label{eq:model}
\begin{aligned}
     P_{ij} & =\Prob(A_{ij} = 1| \{\alpha_i,z_i\}_{i=1}^n) 
	= 1 - \Prob(A_{ij} = 0| \{\alpha_i,z_i\}_{i=1}^n)
	= \frac{\exp({\alpha_i + \alpha_j + z_i^\top H z_j})}{ 1 + \exp({\alpha_i + \alpha_j + z_i^\top H z_j}) }. 
\end{aligned}
\end{equation}
Model \eqref{eq:model} is a generalization of the logistic regression model to the binary network setting.
Here $\{\alpha_i\}_{i=1}^n$ is a sequence of degree parameters.
Nodes with larger values of $\alpha_i$'s are expected to have larger degrees.
Furthermore, $\{z_i\}_{i=1}^n \subset \mathbb{R}^d$ are the latent positions of the nodes in a $d$-dimensional latent space (a.k.a.~``social space'' in the latent space model literature), and $H$ an unobserved $d\times d$ symmetric matrix that moderates how the latent positions affect edge formation.
To impose a community structure, let there be $k$ communities.
Let $\{\calL_{z,1},\dots, \calL_{z, k}\}$ be $k$ different probability distributions defined on the latent space $\mathbb{R}^d$.
We assume that there is an unknown deterministic community label vector $\sigma = (\sigma_1,\dots, \sigma_n)^\top \in [k]^n$.
For each node $i$, $\sigma_i = j$ means the $i$th node belongs to the $j$th community.
In this case $z_i$ is a random vector generated from $\calL_{z,\sigma_i}$, and all the $z_i$'s are mutually independent.
Our goal is to infer $\sigma$ from the observed adjacency matrix $A$.

The latent space model \eqref{eq:model} not only models community structures but is also flexible for modeling degree heterogeneity.
The particular form \eqref{eq:model} can be identified as the latent eigenmodel in \cite{hoff2008modeling} which was shown to possess more flexibility and modeling power than many other latent space models and various blockmodels.
\citet{ma2017exploration} studied fitting methods for this model when $H$ is the identity matrix and $\alpha_i$'s and $z_i$'s are considered deterministic.
See also \cite{levina2017}.
Their study also revealed appealing numerical properties for clustering estimated latent positions after fitting such a special case of \eqref{eq:model}, which has partially motivated the study reported in this manuscript.
Nevertheless, to the best of our limited knowledge, the literature of community detection for latent space models has been scarce.
A sound understanding of community detection is crucial to applications of such models in real-world network datasets. 
The present manuscript aims to take a first step along this direction.

\subsection{Main contributions} 

The main contributions of this manuscript are twofold. 

From an algorithmic viewpoint, we establish consistency of \speclore, a simple and intuitive community detection method for latent space model \eqref{eq:model} in a stylized setting. 
The method is based on spectral clustering followed by a local edge counting refinement step. 
It was first proposed for blockmodels and its properties for the broader class of latent space models, especially in the generality of latent eigenmodels, were previously unknown.
Our new consistency result suggests that the method may enjoy a certain level of universality on exchangeable network models.
The community detection method aims only at estimating community structure while not trying to find estimates of latent positions or their distributions.
Thus, it is different in nature from most algorithms developed for latent space models in the literature which fit specific latent space models and estimate model parameters. 
See, for instance, \cite{ma2017exploration,levina2017,zhang2018network}.
As estimation of latent positions usually involves solving a computationally expensive optimization problem,  
our method bypasses it and attains comparable or even better accuracy for community detection with considerably lower computational cost.

From a theoretical viewpoint, our consistency result sheds light on a better understanding of community detection for latent space models. 
Our explicit upper bounds on rates of convergence exhibit an interesting interplay between signal-to-noise ratio affected by network sparsity and that affected by latent positions and the quadratic form matrix $H$ in \eqref{eq:model}.
In a more restrictive setting, we could even show that the resulting estimator achieves nearly optimal rates of convergence in some minimax sense. 
The key insight comes from the investigation of a special simple vs.~simple hypothesis testing problem which underpins the local refinement step in our method.
We study error rates of a simple edge counting procedure for this testing problem.
By a seemingly intuitive yet elegant exploitation of symmetry inherent to our model, we are able to show that the simple testing method is equivalent to the optimal likelihood ratio test under mild assumptions.  
The equivalence, being the major novelty of our manuscript, paves the way for establishing the optimality of our algorithms.

\subsection{Relation to prior work}

The present manuscript is connected to \cite{ma2017exploration,levina2017} which studied efficient fitting methods for model \eqref{eq:model} when the $z_i$'s are treated as deterministic.
\citet{ma2017exploration} also touched community detection for  \eqref{eq:model}. 
However, the method was a ``plug-in'' one which ran $k$-means clustering to estimated latent positions.
As we shall show empirically, its computational efficiency is far inferior to the method we consider in this paper while community detection accuracies are comparable.

Moreover, \citet{handcock2007model} and \citet{krivitskyfitting} proposed Bayesian algorithms for community detection in a latent distance model which is different from \eqref{eq:model} but can be approximated by it  \cite{ma2017exploration}.
Their study emphasized the algorithmic and computational perspective, and theoretical properties of the proposed methods were not considered.

In addition to the community detection literature for blockmodels that we have mentioned earlier, there have been extensive studies of community detection for random dot-product graph models, especially via spectral methods.
See the review paper \cite{athreya2017statistical} and the references therein.
These models relax SBMs and their variants such as DCBMs and mixed membership blockmodels.
However, these studies have also mostly focused on ``plug-in'' methods and community detection is conducted through clustering estimated latent positions. 
There has been little investigation on methods designed specifically for community detection, and there is little understanding on fundamental limits of such an inference goal.

\paragraph{Notation}
Let $S(\cdot)$ be the sigmoid function $S: x \mapsto 1/(1+e^{-x})$, which is the inverse of the logit function $p \mapsto \log\bigl(p/(1-p)\bigr)$.
Let $\Indc(E)$ be the indicator function of $E$, where $E$ may be an event or a set.
$S_2$ contains the two permutations of $[2]$.
$\| A \|_2$ is the usual operator norm of $A$:
$\| A \|_2 = \sup_{x\neq 0}{\| A x \|_2}/{ \|x\|_2}$.
The Frobenius norm $\fnorm{A}$ of matrix $A = (A_{ij})_{i \in [n],  j \in [m]}$ is defined as
$\fnorm{A} = ( \sum_{i}\sum_{j} A_{ij}^2 )^{1/2}$.
For vector $v = (v_1, \cdots, v_d)^\top \in \bbR^d$, $\|x\|_p = \bigl( \sum_{i=1}^d |x_i|^p\bigr)^{1/p}$ for $p = 1, 2$.
$\boldsymbol{1}_d$ and $\boldsymbol{0}_d$ denote a $d$-dimensional column vector with all entries equal to $1$ and $0$, respectively.
For notational simplicity in asymptotics, for two deterministic sequences $a_n$ and $b_n$, we define the following notations: $a_n \lesssim (\gtrsim)\, b_n$ if and only if there exists a constant $C > 0$ such that $a_n \le (\ge)\, C b_n$; $a_n \ll (\gg)\, b_n$ if and only if $a_n/b_n \rightarrow 0\, (\infty)$ as $n \rightarrow \infty$. 
We also write $a_n = O(b_n)$ when $a_n \lesssim b_n$, and $a_n = o(b_n)$ when $a_n\ll b_n$.


\section{Method}
\label{sec:methods}


We consider a two-stage procedure, consisting of an initialization stage and a refinement stage.
The algorithm was first proposed in \cite{gao2018community} as a community detection method for DCBMs.
In what follows, we introduce the two stages separately for self-completeness.


\subsection{A practical version}
\label{subsec:practical}

We first introduce a practical version of our method which we shall refer to as
\speclore~(\underline{spec}tral clustering followed by \underline{lo}cal \underline{re}finement) in the rest of this paper.
It is obtained by running Algorithm \ref{alg:local} with initial value given by Algorithm \ref{alg:init}.
It relies on Algorithm~\ref{alg:init} to process the adjacency matrix for an initial guess $\widehat{\sigma}^0$ and on Algorithm~\ref{alg:local} to further refine the crude yet informative initial guess to obtain the final estimator.
Here and after, we assume the number of communities $k$ is known.

\paragraph{Initialization}
We summarize the initialization stage as Algorithm~\ref{alg:init}.
In this stage, we first compute the best rank-$k$ approximation $\widehat{P}$ to the observed adjacency matrix $A$ where $k$ is the number of clusters.
Then we apply weighted $k$-median clustering on normalized rows of $\widehat{P}$.
While running weighted $k$-median clustering, we only seek a constant-factor approximation solution to ensure that the output could be produced within polynomial time complexity \cite{charikar2002constant,chen2015convexified}.
Here $\varepsilon$ is required to be an absolute constant.

\begin{algorithm}[!th]
    \label{alg:init}
	\caption{Initialization}
	\begin{algorithmic}[1]
		\STATE {\bfseries Input:} Adjacency matrix: $A$; latent dimension $d$; number of clusters $k$.
        \STATE Find the solution to the following optimization problem
        \begin{equation}
            \widehat{P} = \argmin_{\rank(P) \le k} \fnorm{ A - P}^2.
            \label{eqn:def-P-hat}
        \end{equation}
        \STATE Let $\widehat{P}_i$ be the $i$th row. Define $\indexset_0 = \{ i \in [n] \mid \| \widehat{P}_i\|_1 = 0 \}$.  For $i \in J_0^c$, define $\widetilde{P}_i = \widehat{P}_i/\| \widehat{P}_i\|_1$.  Put $\widehat{\sigma}^0_i = 0$ for $ i \in J_0$.
        \STATE Find a $(1+\kmbound)$ approximate weighted $k$-median solution for clustering $(\widetilde{P}_i)_{i=1}^n$.
        That is, find labels $\widehat{\sigma}^0=\{\widehat{\sigma}^0_i\}_{i=1}^n \in [k]^n$ and centers $\widehat{v}_l\in \mathbb{R}^{k}, l=1,\cdots,k$, such that
		\begin{align*}
            \sum_{l=1}^k \min_{v_l \in \bbR^n} \sum_{\{i\in J_0^c:\hat{\sigma}_i^0=l\}}\| \widehat{P}_i\|_1  \| \widetilde{P}_i -\widehat{v}_l\|_1 \le (1+\kmbound)\min_{\sigma \in [k]^n} \sum_{l=1}^k \min_{v_l\in\mathbb{R}^k} \sum_{\{i:\sigma_i=l\}}\| \widehat{P}_i \|_1 \|\widetilde{P}_i-v_l\|_1.
		\end{align*}
		\STATE {\bfseries Output:} $\widehat{\sigma}^0$.
	\end{algorithmic}
\end{algorithm}

\paragraph{Refinement}
We then state the local refinement procedure in Algorithm \ref{alg:local}.
Starting with an initial estimator $\wh\sigma^{0}$, we refine it by the following simple and intuitive majority voting rule.
For node $i$, we look at all communities prescribed in $\wh\sigma^0$ and calculate the relative connecting frequency from $i$ to each community.
Then we recalibrate the community label of node $i$ to be that of the community to which it most likely connects.
Since the refinement is strictly local, it can be easily carried out in a parallel fashion on each node.
As the process only involves counting edges, a crude inspection of the algorithm puts the computational cost of one round of refinement at $O(n^2)$.
Moreover, as simulated and real world examples reported in Sections \ref{sec:simulation} and \ref{sec:realdata} suggest, one typically only needs to run an $O(1)$ round of refinement to arrive at a stable estimator.

\begin{algorithm}[!h]
	\caption{Local Refinement}
	\label{alg:local}
	\begin{algorithmic}[1]
        \STATE {\bfseries Input:} {Adjacency matrix: $A$; number of clusters $k$; an initial label vector $\widehat{\sigma}^{0}$; number of iterations $R$.}
        \STATE Initialize $\widehat{\sigma}^{\mathrm{old}} : =  \widehat{\sigma}^{0}$.
        \FOR{$t \leftarrow 1$ \KwTo $R$}
           \FOR{$i \leftarrow 1$ \KwTo  $n$}
           \STATE Update the labels \begin{align*}
               \widehat{\sigma}^{\mathrm{new}}_i :=\argmax_{u\in[k]} \frac{1}{|\{j:\widehat{\sigma}^{\mathrm{old}}_j=u\}|}\sum_{\{j:\widehat{\sigma}^{\mathrm{old}}_j=u\}} A_{ij}.
		   \end{align*}
           \ENDFOR
           \STATE $\widehat{\sigma}^{\mathrm{old}} :=  \widehat{\sigma}^{\mathrm{new}}$.
        \ENDFOR
        \STATE {\bfseries Output:} $\widehat{\sigma} :=\widehat{\sigma}^{\mathrm{new}}$.
	\end{algorithmic}
\end{algorithm}

\subsection{A theoretically justifiable variant}

In this part, we state a theoretically justifiable variant of \speclore, summarized as Algorithm~\ref{alg:provable}, for which we will establish an upper bound in Section \ref{sec:theory}.
As an artifact of our proof techniques (see the proof of Theorem~\ref{thm:upper}), we are unable to present a cleaner theory for \speclore. 
As a remedy, the new comprehensive Algorithm~\ref{alg:provable} has two stages as well and combines both Algorithms~\ref{alg:init} and \ref{alg:local}, albeit not in a simple consecutive fashion.

The first part of Algorithm~\ref{alg:provable} (lines 2--7) does a separate initialization on each node by performing Algorithm~\ref{alg:init} on the network excluding node $i$, leading to a vector $\widehat\sigma^{(-i,0)}$.
It then applies Algorithm \ref{alg:local} on $\widehat\sigma^{(-i,0)}$ to obtain a refined estimate for node $i$, denoted by $\widehat\sigma^{(-i,0)}_i$.
The separate initializations dissolve an issue in the proof.
However, since each initialization could end up with a different permutation of community labels, 
the second part of Algorithm \ref{alg:provable} (lines 8--11) aligns all label permutations with that of $\widehat\sigma^{(-1,0)}$.

Algorithm~\ref{alg:provable} has at most polynomial time complexity.
We do not emphasize its computational efficiency though, since we view it more as a proof device rather than a practical replacement of {\speclore} in the previous subsection.

\begin{algorithm}[!h]
\caption{A provable version of latent space model community detection method}
\label{alg:provable}
\begin{algorithmic}[1]
    \STATE {\bfseries Input:}
    Adjacency matrix: $A$; latent dimension $d$; number of clusters $k$.
    \FOR{$i \leftarrow 1$ \KwTo $n$} \label{alg:line-for}
	\STATE Let $A^{(-i)}\in \{0,1\}^{(n-1)\times (n-1)}$ be the matrix obtained from removing the $i$th row and the $i$th column of $A$;

	\STATE Apply Algorithm \ref{alg:init} on $A^{(-i)}$ to  obtain $\wh\sigma^{(-i,0)}\in [k]^{n-1}$;

	\STATE Augment $\wh\sigma^{(-i,0)}$ to a $n$-dimensional vector by inserting $0$ in the $i$th position;

	\STATE Update
	\begin{equation*}
		\wh\sigma_i^{(-i,0)}
		=\argmax_{u\in[k]} \frac{1}{|\{j:\widehat{\sigma}^{\sss (-i,0)}_j=u\}|}\sum_{\{j:\widehat{\sigma}^{\sss (-i,0)}_j=u\}} A_{ij}.
	\end{equation*}
\ENDFOR

\STATE Define $\wh\sigma_1 = \wh\sigma_1^{(-1,0)}$.

\FOR{$i \leftarrow 2$ \KwTo $n$}
	\STATE Let
	\begin{equation*}
		\wh\sigma_i = \argmax_{u\in [k]} \left| \{j:\wh\sigma^{(-1,0)}_j = u\}
		\cap \{j:\wh\sigma^{(-i,0)}_j = \wh\sigma_i^{(-i,0)} \}  \right|.
	\end{equation*}
\ENDFOR

		%
\STATE {\bfseries Output:} $\wh{\sigma}= (\wh{\sigma}_1,\dots, \wh\sigma_n)^\top \in [k]^n$.
	\end{algorithmic}
\end{algorithm}


\section{Theoretical results}
\label{sec:theory}

In this section, we present decision theoretic results for Algorithm \ref{alg:provable} on model \eqref{eq:model}.
We focus on the balanced two community case.
In other words, we consider the case where $k=2$ and the two communities have roughly equal sizes.
The need to consider Algorithm \ref{alg:provable} is due to proof technique, and we show in later sections that there is little numerical difference between its accuracy and that of \speclore in Section \ref{subsec:practical}.

\subsection{A decision-theoretic framework}

We shall establish uniform high probability error bounds for Algorithm \ref{alg:provable}.
To this end, we first define classes of models for which uniform error bounds are to be obtained.

\paragraph{Uniformity class}
Let the adjacency matrix be $A = (A_{ij}) = A^\top \in \{0,1\}^{n\times n}$.
Given a \emph{deterministic} community label vector $\sigma\in [2]^n$,
we suppose that the edges are generated in the following way:
\begin{equation}
	\label{eq:model-1}
\begin{aligned}
	\alpha_i & \stackrel{iid}{\sim} F_\alpha, \quad
	z_i      \stackrel{ind}{\sim} F_{z,\sigma_i},
	\quad i\in [n],\\
	A_{ij} = A_{ji} \,|\, \alpha_i,\alpha_j,z_i, z_j & \stackrel{ind}{\sim} \text{Bernoulli}(P_{ij}),\quad  i,j\in [n], \quad \text{where}\\
	& ~~~~~\text{logit}(P_{ij}) = \alpha_i + \alpha_j +  z_i^\top H z_j .
\end{aligned}
\end{equation}
Here $F_\alpha$ is a distribution from which the $\alpha_i$'s are generated,
and $H$ is a symmetric $n\times n$ matrix.
The two distributions $\{F_{z,j}: j=1,2\}$ generate each latent position $z_i$ depending on the value of $\sigma_i$.
For most of theoretical results below, we further assume that
\begin{equation}
	\label{eq:model-2}
	F_{z,j} \stackrel{d}{=} N_d( (-1)^{j-1} \mu, \tau^2 I_d),
	\quad j=1,2.
\end{equation}
In other words, we assume that the latent positions within each community are generated according to an isotropic multivariate Gaussian distribution with shared covariance structure\footnote{If we start with $F_{z,j} \stackrel{d}{=} N_d( (-1)^{j-1} \mu, \tau^2 \Sigma)$ for some positive definite matrix $\Sigma$. Then we can rewrite model \eqref{eq:model-1} -- \eqref{eq:model-2} with $H$ replaced by $\widetilde{H} = \Sigma^{1/2}H \Sigma^{1/2}$, $\mu$ by $\widetilde\mu = \Sigma^{-1/2}\mu$ and $F_{z,j}$ by $\widetilde{F}_{z,j} \stackrel{d}{=} N_d( (-1)^{j+1} \widetilde\mu, \tau^2 I_d)$. Therefore, the assumption of a covariance matrix proportional to identity does not impose any more restriction than that the two latent variable distributions corresponding to the two communities share the same covariance structure.}\ and different mean vector depending on the community label.
Here and after, $I_d$ is the $d\times d$ identity matrix.
For identifiability of $\mu$, $\tau$ and $H$, we assume that
\begin{equation}
	\label{eq:H-eig}
	\|H\|_2 = 1.
\end{equation}

In what follows, we denote such a model by $\calM_n(\sigma,H,\mu,\tau,F_\alpha)$.
For each $\sigma\in [2]^n$ and each $j\in [2]$, let $n_j = n_j(\sigma) = |\{i:\sigma_i = j\}|$.
The uniformity classes of interest are of the form
\begin{equation}
\label{eq:para-space}
\calP_n(H, \mu,\tau, F_\alpha) =
\left\{
\calM_n(\sigma, H,\mu,\tau, F_\alpha):
n_j(\sigma)\in \left[(1 - \delta_n)\frac{n}{2}, (1+\delta_n)\frac{n}{2} \right],
~~~j=1,2
\right\},
\end{equation}
where $\delta_n = o(1)$ is some vanishing sequence.
In the rest of this section, we treat $H$ and $\mu$ as fixed parameters, while $\tau$ and $F_\alpha$ scale with $n$.

\paragraph{Estimation and loss function}
Our goal is to estimate the community labels $\sth{\sigma_i:i\in [n]}$ based on the observed adjacency matrix $A$.
Since permutation of community labels does not change the partition of nodes,
we use the following misclustering proportion as the loss function:
\begin{equation}
\label{eq:loss}
\ell({\sigma}, \widehat\sigma) = \min_{\pi \in S_2} \frac{1}{n}\sum_{i=1}^n \indc{\widehat\sigma_i \neq \pi(\sigma_i)}.
\end{equation}

\subsection{Assumptions on model parameters}


For convenience of reference, we collect and explain various assumptions used in main results here.



\begin{assumption}
	\label{assump:alpha}
For $i\in [n]$,
$\alpha_i={\alphaavg}+\omega_i$,
with $\alphaavg$ deterministic, $\omega_i$ {\iid}~with $\expect[\omega_1] = 0$, \hs{$\expect[e^{2\omega_1}]\leq C$} for some constant $C>0$, and
\begin{equation}
	\label{eq:omegabar}
	-\underline{\omega}\le \omega_i\le \omegabar,
\end{equation}
where $\underline{\omega}>0$ is a constant but $\omegabar$ is allowed to grow to $\infty$ with $n$.
As $n\to \infty$, ${\alphaavg}$ and $\omegabar$ jointly satisfy all of the following conditions
\begin{gather}
{\alphaavg}+\omegabar  \to  -\infty, \label{eq:alpha1} \\
\hs{n e^{2\alphaavg}/ \sqrt{\log n} \to \infty, \label{eq:alpha4}} \\
\hs{e^{\omegabar} \big/ \min \big\{n e^{2\alphaavg}, n/\log n \big\} \to 0. \label{eq:alpha5} }
\end{gather}
Furthermore, for some constants $\overline{L}>0$ and $C_1>0$, the empirical fourth moment of $e^{\omega_i}$ satisfies the condition
\begin{align}
    \Prob\left(\Bigl(\frac{1}{n_u}\sum_{\sigma_i=u}e^{4\omega_i}\Bigr)^{1/4} >  \overline{L}\right) & \le n^{-(1+C_1)}, \quad \text{\ for\ } u \in [2]. \label{eq:alpha3}
\end{align}

\end{assumption}

    In this overarching assumption on $F_\alpha$, equation \eqref{eq:alpha1} ensures that the network is sparse in the sense that the maximum degree scales at an $o(n)$ rate.
Equations \eqref{eq:omegabar} and \eqref{eq:alpha4} jointly imply that the minimum degree grows at a rate no slower than $\sqrt{\log n}$.
Equation \eqref{eq:alpha5} guarantees that the maximum degree grows at a slower rate than squared minimum degree.
Moreover, it imposes the restriction that the ratio of maximum over minimum degrees grows at a slower rate than $n/\log n$.
Finally, \eqref{eq:alpha3} puts some technical tail bounds on the empirical fourth moments of $e^{\omega_i}$ within each community.

\begin{assumption}\label{assump:tau}
	There exists a positive constant $c$ such that $\tau \sqrt{\log n}\le c$.
\end{assumption}

Even if we directly observe the latent positions $\{z_i\}_{i=1}^n$, we always suffer the Bayes error for clustering two normal distributions with identical covariance structure.
Write $\bar{\Phi}(t) = \bbP( N(0,1) \ge t)$.
Under model \eqref{eq:model-1}--\eqref{eq:model-2}, simple calculation shows that the Bayes error  is at the rate 
\( \bar{\Phi}( \| \mu\|_2 / \tau) \lesssim \exp\bigl(- \|\mu\|_2^2/(2\tau^2)\bigr)\tau/\| \mu\|_2    \)
as $n\to\infty$.
Since $\mu$ is fixed,
by varying $c$, Assumption \ref{assump:tau} allows us to consider any case where the Bayes error scales at an $O(n^{-a})$ rate for any $a > 0$.

\begin{assumption}
	\label{assump:mu-H-mu}
For $H$ in \eqref{eq:model} and $\mu$ in \eqref{eq:model-2}, $\mu^\top H\mu>0$.
\end{assumption}

This is an assortativity assumption.
With this assumption, we make certain that, given the same $\alpha_i$ values, nodes within the same community are more likely to be connected than nodes from two different communities.

\begin{assumption}
For $H$ in \eqref{eq:model} and $\mu$ in \eqref{eq:model-2},
$\mu$ is an eigenvector of $H$ associated with some positive eigenvalue.
    \label{assump:mu-Sigma-H}
\end{assumption}

This assumption is a strengthened version of Assumption \ref{assump:mu-H-mu}.
It is trivially true when $H = I_d$ is the identity matrix.
We only need this assumption when minimax lower bounds are concerned.
%
{
\begin{remark}
\label{rem:example-mu-eigenvector}
We take the following simple example to see what Assumption~\ref{assump:mu-Sigma-H} entails.
Let
\(
H = \diag(\boldsymbol{1}_{d_1}^\top, -\boldsymbol{1}_{d - d_1}^\top).
\)
The inner product defined by $H$ results in
\(
    P_{ij} = S(\alpha_i + \alpha_j + z_i^{\sss (1)} z_j^{\sss (1)} - z_i^{\sss (2)} z_j^{\sss (2)}  ),
\)
where the superscript $(1)$ and $(2)$ indicate the vector made of the first $d_1$ coordinates and the last $d - d_1$ coordinates of $z$, respectively.
Possible $\mu$'s, allowing the above argument to work, can take value in the $d_1$-dim.\ subspace such as 
\(
    \mu = (\bigl(\mu^{\sss (1)}\bigr)^\top,\boldsymbol{0}_{d-d_1}^\top)^\top.
\)
This means the latent variable $z$ can be decomposed into two components, the \textit{signal} component $z^{\sss (1)}$ and the \textit{noise} component $z^{\sss (2) }$,
\[
    z = \begin{pmatrix}
        z^{\sss (1)} \sim \mu^{\sss (1)} + N_{d_1}(0, I_{d_1})\\
        z^{\sss (2) } \sim N_{d - d_1}(0, I_{d - d_1})
    \end{pmatrix}.
\]
The signal component enhances the clustering and the noise reduces signal-to-noise ratio.
In effect, this allows some additional flexibility in adding some noise in the latent variable. 
\end{remark}
}

\subsection{A closely related testing problem} 
\label{subsec:testing}


We first consider the following testing problem, 
which applies to slightly more general settings than the model setup that we usually take in the rest of the manuscript.

Suppose that we observe a network of size $2m+1$, with $m$ nodes $1, 2, \dots, m$ having known labels $1$ ($+$) and $m$ nodes ${m+1}, \dots, {2m}$ having  labels $2$ ($-$).
Suppose that node $0$ has the only unknown label $\sigma_0$. 
Suppose that we have some base distribution $F$ with density $f$ and write $F_{\nu}$ as its shifted version by $\nu$ with density $f_{\nu}$, i.e., $f_{\nu}(z) = f(z - \nu)$.
In addition, we assume that for nodes in the first community, $z_i\stackrel{\text{iid}}{\sim} F_{\mu}$ and for those in the second, $z_i\stackrel{\text{iid}}{\sim} F_{-\mu}$.
We proceed to consider testing the following hypotheses
\begin{equation}
	\label{eq:onenode-testing}
    H_0: \sigma_0 = 1,  \quad\text{versus} \quad H_1: \sigma_0 = 2.
\end{equation}
Let $A_{0,i} = 1$ if there is an edge between nodes $0$ and $i$, and otherwise $0$.
Under our modeling assumption, conditional on the realization of the $\alpha$'s and the $z$'s, $\{A_{0,i}:i=1,\dots, 2m\}$ are independent Bernoulli random variables with success probability
\(
    P_{0i}
	= S(z_0^\top H z_i + \alpha_0 + \alpha_i).
\)
Define $A_{0,+} = \sum_{i=1}^m A_{0,i}$ and $A_{0,-} = \sum_{i=m+1}^{2m} A_{0,i}$. 

\subsubsection{Likelihood ratio test and edge counting}

The following lemma connects the likelihood ratio test for \eqref{eq:onenode-testing} and edge counting.

\begin{lemma}
    \label{lemma:likelihood-ratio-test-general-equiv}
    Consider the hypothesis testing problem \eqref{eq:onenode-testing}.
    Suppose that $f$ is symmetric about the origin, i.e., $f(z) = f(-z)$, and  that $f_\mu(z) > f_{-\mu}(z) $ on $\{ z : z^\top H \mu > 0\}$.
    Suppose that $\{ \alpha_i: 0 \le i \le n\}$ are i.i.d.
    Then the likelihood ratio test {which reject $H_0$ when the likelihood ratio of alternative over null is larger than $1$} is equivalent to the simple edge counting test where we reject $H_0$ when $A_{0,+} < A_{0,-} $.
\end{lemma}


\begin{proof}
To simplify notation, write $F_{-}(\cdot)$ and $F_{+}(\cdot)$ as shorthands of $F_{-\mu}$ and $F_\mu$, respectively, and $f_{-}(\cdot)$ and $f_{+}(\cdot)$ the corresponding densities.
Let $F_\alpha$ be the generating distribution of $\alpha$'s.
Define the following quantities: 
\begin{gather}
    \label{eqn:def-p-z0}
    p(\alpha_0, z_0) := \iint S(z_0^\top H z + \alpha_0 + \alpha) dF_\alpha(\alpha) dF_+(z), \\
    q(\alpha_0, z_0) := \iint S(z_0^\top H z + \alpha_0 + \alpha) dF_\alpha(\alpha) dF_-(z).
    \label{eqn:def-q-z0}
\end{gather}
Applying that $F_+$ and $F_-$ are symmetric about the origin, i.e., \( dF_+(z) = f(z-\mu) = f(-z + \mu) = dF_-(-z) \), we have
\(
    q(\alpha_0, z_0  = \iint S( - z_0^\top H z + \alpha_0 + \alpha ) dF_\alpha(\alpha) dF_+(z)  = p(\alpha_0, -z_0).
\)
Conditioned on $z_0$ and $\alpha_0$, by Fubini's theorem, we obtain the conditional likelihood
\begin{equation*}
    g(\alpha_0, z_0) : = \bigl( p(\alpha_0, z_0)\bigr)^{A_{0,+}} \bigl( 1 - p(\alpha_0, z_0) \bigr)^{m - A_{0, +} }
    \bigl( q(\alpha_0, z_0)\bigr)^{A_{0,-}} \bigl( 1 - q(\alpha_0, z_0) \bigr)^{m - A_{0, -} }.
\end{equation*}
We may obtain $g(\alpha_0, -z_0)$ by plugging in $-z_0$ in the last display and noticing $p(\alpha_0, -z_0) = q(\alpha_0, z_0)$
\begin{equation*}
    g(\alpha_0, - z_0) = \bigl(q(\alpha_0, z_0)\bigr)^{A_{0,+}} \bigl( 1-q(\alpha_0, z_0) \bigr)^{m - A_{0,+}} \bigl( p(\alpha_0, z_0) \bigr)^{A_{0,-}} \bigl(1 - p(\alpha_0, z_0)\bigr)^{m - A_{0,-}}.
\end{equation*}
The full likelihood under $H_0$, denoted by $I_+$, minus the full likelihood under $H_1$, $I_-$, is
\begin{equation}
    \begin{aligned}
        I_+ - I_-
		& =  \iint g(\alpha_0, z_0) dF_\alpha(\alpha_0) dF_{+}(z_0) - \iint g(\alpha_0, z_0)dF_\alpha(\alpha_0) dF_{-}(z_0) \\
        & = \int\left[\int (g(\alpha_0, z_0) - g(\alpha_0, -z_0)) dF_\alpha(\alpha_0)\right] dF_{+}(z_0). 
    \end{aligned}
    \label{eqn:full-lhd-diff}
\end{equation}
We define the above integrand inside the square brackets to be $G(z_0)$ and write $p$ and $q$ as shorthands of $q(\alpha_0, z_0)$ and $p(\alpha_0, z_0)$, respectively. So
\begin{align*}
    G(z_0) & : =  \int  \left\{ ( 1 - p )^{m}  ( 1 - q )^{m } \left[  \left(\frac{p}{  1 - p}\right)^{A_{0, + } } \left(\frac{q}{  1 - q}\right)^{A_{0, - }}  -  \left(\frac{q}{  1 - q}\right)^{A_{0, + } } \left(\frac{p}{  1 - p}\right)^{A_{0, - }}\right]\right\} dF_\alpha(\alpha_0).
\end{align*}
Moreover, since $p(\alpha_0, -z_0) = q(\alpha_0, z_0)$, we have
$G(-z_0) = - G(z_0)$.
If $A_{0,+}=A_{0, -}$, the preceding display is $0$ and \(I_+ = I_-\),
whence we may not differentiate between $H_0$ and $H_1$.
For the rest of this proof, we consider $A_{0,+} > A_{0,-}$.

We first note that on \( \{ z_0 : z_0^\top H \mu >0 \} \), by the monotonicity of \( S: x \mapsto e^x/(1+e^x)\),
\begin{align*}
    p(\alpha_0, z_0) & = \iint S\bigl(z_0^\top H (\mu + w) + \alpha_0 + \alpha \bigr) dF_\alpha(\alpha) dF(w) \\
    &  \ge \iint S\bigl(- z_0^\top H (\mu + w) + \alpha_0 + \alpha \bigr) dF_\alpha(\alpha) dF(w) = q(\alpha_0, z_0).
\end{align*}
The last equality comes from the symmetry of $f$ about the origin.
We further note that by the monotonicity of the mapping \( x \mapsto x/(1-x)  \) for $x \in (0,1)$, $p/(1-p) > q/(1-q)$ on $\{ z_0: z_0^\top H \mu >0\}$.
We obtain
\(
    \bigl(({p/(1-p))}/(({q/(1-q))}\bigr)^{A_{0,+} - A_{0,-}} > 1,
\)
whence we conclude that $G(z_0) > 0$ for $z_0$ such that $z_0^\top H \mu >0$.
Finally we have
\begin{equation*}
    \begin{aligned}
        I_+ - I_- & = \int_{z_0^\top H \mu >0} G(z_0) dF_{+}(z_0) + \int_{z_0^\top H \mu <0} G(z_0) dF_+(z_0) \\
            & = \int_{z_0^\top H \mu >0} G(z_0) dF_{+}(z_0) - \int_{z_0^\top H \mu > 0} G(z_0) dF_-(z_0) \\
            & = \int_{z_0^\top H \mu >0} G(z_0)\bigl(f_+(z_0) - f_-(z_0)\bigr) dz_0 > 0.
    \end{aligned}
\end{equation*}
The last inequality holds as $f_+(z_0) > f_-(z_0)$ on $\{z_0 : z_0^\top H \mu>0 \}$.
The proof is complete after applying the same argument to the case $A_{0,+} < A_{0,-}$, which implies $I_+ < I_-$. 
\end{proof}

\begin{remark}
If \( \mu \) is an eigenvector of $H$ associated with a positive eigenvalue $\lambda$ {as in Assumption \ref{assump:mu-Sigma-H}}, then the two hyperplanes $\{ z: z^\top H \mu = 0 \}$ and $\{ z: z^\top \mu =0 \}$ coincide, and for all \(z\) such that $z^\top  \mu >0$, $z^\top H \mu = \lambda z^\top \mu  >0 $.
\label{rem:more-on-eigenvector-mu}
\end{remark}

\begin{remark}
    If we can write the density as
    \( f(z) = r( \|z\|_2 )\)
    for 
    some monotone decreasing function $r: \bbR_+ \rightarrow \bbR_+$ and $\mu$ is an eigenvector of $ H$ associated with some positive eigenvalue, the conditions on the density in Lemma~\ref{lemma:likelihood-ratio-test-general-equiv} are satisfied.
    \label{rem:more-on-f}
\end{remark}

In light of the above remarks, we arrive at the fundamental testing lemma for our setup. 
Note that we only need $\alpha$'s being i.i.d.\ for Lemma~\ref{lemma:likelihood-ratio-test-equiv} to hold; here the distributional restrictions of $\alpha$ in \eqref{eq:omegabar}--\eqref{eq:alpha3} of Assumption~\ref{assump:alpha} are superflous.
\begin{lemma}
    Consider the testing problem in \eqref{eq:onenode-testing} with $F$ being $N_d(0, \tau^2 I_d )$. Suppose that Assumptions~\ref{assump:alpha} and \ref{assump:mu-Sigma-H} hold. 
    Then the likelihood ratio test for the above hypothesis testing problem \eqref{eq:onenode-testing} is equivalent to the simple edge counting test where we reject $H_0$ when $A_{0,+} < A_{0,-} $.
    \label{lemma:likelihood-ratio-test-equiv}
\end{lemma}

\subsubsection{Error rates for edge counting}

Let $\countrate$ be the probability of making Type I+II errors of the test that rejects $H_0$ in \eqref{eq:onenode-testing} when $A_{0,+} < A_{0,-}$ with $F$ being $N_d(0, \tau^2 I_d)$.
For any fixed $\alpha_0$ and $z_0$, let $p(\alpha_0, z_0)$ and $q(\alpha_0, z_0)$ be defined as in \eqref{eqn:def-p-z0} and \eqref{eqn:def-q-z0} respectively, and let
\begin{equation}
	\label{eq:renyi-div}
I(\alpha_0,z_0)=-2\log \left(\sqrt{p(\alpha_0,z_0)q(\alpha_0,z_0)}+\sqrt{(1-p(\alpha_0,z_0))(1-q(\alpha_0,z_0))}\right)
\end{equation}
be the R\'{e}nyi divergence of order $\frac{1}{2}$ between two Bernoulli distribution Bernoulli$(p(\alpha_0,z_0))$ and Bernoulli$(q(\alpha_0,z_0))$.
The projection distance from $\mu $ to the hyperplane $\{ z: z^\top H \mu =0\}$ is then 
\begin{align}
    \rho :=\frac{\mu^\top H\mu}{\sqrt{\mu^\top H^2\mu}}.
\label{eq:rho}
\end{align}
Furthermore, for any positive integer $n$ and any fixed $\epsilon > 0$, define
\begin{align}
\countrateupper & =
\Expect^{\alpha_0,z_0}_{H_0}
\left[\indc{ z_0\in \calB_\epsilon } \exp\left\{-\frac{n}{2}(1-\epsilon)I(\alpha_0,z_0)\right\}\right] +
\exp\left\{-(1-\epsilon)\frac{\rho^2}{2\tau^2}\right\},
\label{eq:rate-upper}
\\
\countratelower & =
\Expect^{\alpha_0,z_0}_{H_0}
\left[\indc{ z_0\in \calB_\epsilon } \exp\left\{-\frac{n}{2}(1+\epsilon)I(\alpha_0,z_0)\right\}\right] +
\exp\left\{-(1+\epsilon)\frac{\rho^2}{2\tau^2}\right\},
\label{eq:rate-lower}
\end{align}
where $\mathcal{B}_\epsilon=\{z_0: \|z_0-\mu\|_{2}\le \sqrt{1-\frac{\epsilon}{4}}\rho\}$ and the notation $\Expect^{\alpha_0,z_0}_{H_0}$ means taking expectation over $\alpha_0$ and $z_0$ when the null hypothesis in \eqref{eq:onenode-testing} is true.
{Note that we have $\overline{\nu}_n^{0} = \underline{\nu}_n^{0}$ if we generalize both \eqref{eq:rate-upper} and \eqref{eq:rate-lower} to allow $\epsilon = 0$.}
{
There are two terms in both \eqref{eq:rate-upper} and \eqref{eq:rate-lower}. 
The first term involving the R\'{e}nyi divergence has previously appeared in the blockmodel community detection literature. 
It reflects the average influence on signal-to-noise ratio from the difference in Bernoulli sampling probabilities of edges connecting nodes within the same or between two different communities.
Since the Bernoulli sampling probabilities depend on the realized latent positions, the term collects indirect influence on signal-to-noise ratio from the latent space.
The second term depends on the distributions of $z$'s and the quadratic form matrix $H$ only, and it sums up the direct influence on signal-to-noise from the latent space.
}

With the foregoing definitions, the following lemma controls $\countrate$ from both sides.

\begin{lemma}
\label{lem:edgecount-rate}
Suppose that Assumptions \ref{assump:alpha} and \ref{assump:mu-H-mu} hold.
Let $n=2m+1$ and that $z_i\stackrel{\text{iid}}{\sim} N_d(\mu, \tau^2 I_d)$ for $i=1,\dots, m$ and $z_i\stackrel{\text{iid}}{\sim} N_d(-\mu, \tau^2 I_d)$ for $i=m+1,\dots, 2m$, where $\tau\to 0$ as $n \to \infty$.
Further assume that $\countrateuppereps{0}\to 0$ as $n\to\infty$,
then for any $\epsilon \in (0,1/2)$,
there is an $n_\epsilon$ such that for all $n>n_\epsilon$,
\begin{equation}
	\label{eq:rate-sandwich}
	\countratelower \le
	\nu_n \le
	\countrateupper.
\end{equation}
\end{lemma}


\subsection{Rates of convergence}

In this subsection, we present rates of convergence on errors of our initial and refined estimators.

\paragraph{Upper bounds}
The following proposition gives upper bounds for estimators obtained from Algorithm \ref{alg:init}.

\hs{
\begin{proposition}
    Suppose that Assumptions \ref{assump:alpha}, \ref{assump:tau} and \ref{assump:mu-H-mu} hold. 
	Assume that the $n$ nodes have true labels $\sigma$, where $\sigma_i=1$ for $i=1,\cdots,n_1$, $\sigma_i=2$ for $i=n_1+1,\cdots, n$, 
	and $n_1, n_2 \in \bigl[ (1- \delta_n)/2, (1+\delta_n)/2\bigr]$.
	Let $\hat{\sigma}^0$ be the output of Algorithm \ref{alg:init}. 
    Then for any $\gamma>0$, some constant $C>0$ and all sufficiently large $n$, we have
	\begin{align*}
        \Prob(\ell(\sigma,\hat{\sigma}^0)\le \gamma) \geq \Prob\bigg(\sum_{\{i:\sigma_i\neq \hsigma^0_i\}} e^{\omega_i}\le e^{-\underline{\omega}} \gamma n \bigg)\ge 1-n^{-(1+2C)}.
	\end{align*}
	\label{prop:init-error}
\end{proposition}
}

The following theorem gives our main upper bounds on the output of Algorithm \ref{alg:provable}.
\begin{theorem}
	\label{thm:upper}
Let $k = 2$ and $\calP_n = \calP_n(H,\mu,\tau, F_\alpha)$.
Suppose that Assumptions~\ref{assump:alpha}, \ref{assump:tau} and \ref{assump:mu-H-mu} hold. 
For any $\epsilon \in (0,1/2)$, let $\countrateuppereps{\epsilon}$ be defined as in \eqref{eq:rate-upper}.
Suppose $\countrateuppereps{0}\to 0$ as $n\to\infty$.
Then for any fixed $\epsilon > 0$, the output $\wh{\sigma}$ of Algorithm \ref{alg:provable} satisfies
\begin{equation*}
\limsup_{n\to\infty} \sup_{\calP_n}
\Prob\left\{ \ell(\sigma,\wh{\sigma}) >  \countrateuppereps{\epsilon} \right\} = 0.
\end{equation*}
\end{theorem}

The high probability upper bound in Theorem \ref{thm:upper} consists of two terms as on the righthand side of \eqref{eq:rate-upper}. 
In view of the discussion following \eqref{eq:rate-upper}, the first term summarizes influence on the clustering error from the network signal, averaged over realizations of degree sequence and latent positions.
Hence we regard it as the network term.
The second term collects immediate influence on clustering error by signal from latent space as it depends only on $H$ and the latent position distributions, which could be viewed as the latent space term.

\paragraph{Lower bounds}
We conclude this section with the following minimax lower bounds when Assumption \ref{assump:mu-Sigma-H} holds, which implies Assumption \ref{assump:mu-H-mu}.
The lower bounds match the upper bounds in Theorem \ref{thm:upper} up to some arbitrarily small perturbation of the exponents.

\begin{theorem}
	\label{thm:lowbd}
Let $k=2$ and $\calP_n = \calP_n(H,\mu,\tau, F_\alpha)$.
Suppose that Assumptions \ref{assump:alpha}, \ref{assump:mu-H-mu} and \ref{assump:mu-Sigma-H} hold.
Suppose $\countrateuppereps{0}\to 0$ as $n\to\infty$.
For any $\epsilon \in (0,1/2)$, define $\countratelower$ as in \eqref{eq:rate-lower}, then the minimax risk satisfies
\begin{equation}
	\label{eq:minimax-lower}
	\inf_{\widehat\sigma} \sup_{\mathcal{P}_n} \Expect[\ell(\sigma,\widehat\sigma)]
	\gtrsim \countratelower.
\end{equation}
\end{theorem}


\section{Real data examples} 
\label{sec:realdata}

We now demonstrate performance of the proposed algorithm on some real data examples. 
More detailed comparison of Algorithm \ref{alg:provable} with Algorithms \ref{alg:init}+\ref{alg:local} and other methods on carefully constructed simulated examples can be found in Section \ref{sec:simulation} of the appendices.

We consider five datasets.
The first three datasets are Political Blog ($1222$ nodes, $16714$ edges, and $2$ communities) \cite{adamic2005political}, Simmons College ($1137$ nodes, $24257$ edges, and $4$ communities) and Caltech data ($590$ nodes, $12822$ edges, and $8$ communities) \cite{traud2011comparing,traud2012social}. 
For Simmons College and Caltech data, we followed the same pre-processing steps as in \cite{chen2015convexified}.
These datasets have been studied extensively in the blockmodel community detection literature.

The fourth dataset is 
a manufacturing company network from \cite{cross2004social},  
which was studied in \cite{feng2018nodal}.
Questions were asked to pairs of employees on their ties in work, and weights were assigned on a $0$--$6$ scale where higher weights correspond to closer ties.
Following \cite{feng2018nodal}, we used the weights to create an adjacency matrix: We set $A_{ij}=A_{ji}=1$ if and only if both edges from $i$ to $j$ and from $j$ to $i$ have weights larger than $3$. Otherwise, $A_{ij} = A_{ji} = 0$.
This resulted in an undirected network with $74$ nodes and $235$ edges.
Four communities were formed according to the ``location'' value of each node which is the most assortative among three available node attributes in this data.
%

The fifth dataset is
a French high school friendship network \cite{barrat2015highschool}. 
This dataset recorded friendship relations and contacts among $329$ students in a Marseilles high school.
To construct an adjacency matrix, we took the first contact information which recorded active contacts between students during 20-second intervals of the data collection process over a measuring infrastructure. 
We set $A_{ij}=A_{ji}=1$ if and only if there were contacts recorded between $i$ and $j$.
The resulting network has $5818$ edges. 
Each student belonged to one of nine classes which we regarded as nine true communities.

In this section, we compare Algorithm \ref{alg:init} + one-round Algorithm \ref{alg:local} refinement ({\speclore}$_{R=1}$) and Algorithm \ref{alg:init} + ten-round Algorithm \ref{alg:local} refinement ({\speclore}$_{R=10}$) to LSCD in \cite{ma2017exploration} (initialized by Algorithm 3 in \cite{ma2017exploration} followed by Algorithm 1 in \cite{ma2017exploration} with $800$ iterations).
Algorithm \ref{alg:provable} has essentially the same level of accuracy as \speclore with $R=1$, which we illustrate in detail in Section \ref{sec:simulation}.
The LSCD methods functioned as the benchmark.
Comparison of LSCD to several other state-of-the-art methods on the first three datasets was already  conducted in \cite{ma2017exploration}. 
LSCD was shown to be a top performer, and so we omit comparison to other methods. 
We set latent space dimension equal to number of communities for LSCD.

Table \ref{tab:three-datasets} presents performances of both versions of {\speclore} and those of LSCD in terms of accuracy and speed. 
For reported speed of \speclore, we have included time spent on spectral initialization.
In addition, it also reports accuracy of spectral initialization (Algorithm \ref{alg:init}).
On these five datasets, {\speclore}$_{R=10}$ and LSCD were comparable in terms of accuracy while {\speclore}$_{R=10}$ was significantly faster (and also slightly more accurate in most examples).
This is not surprising since it aims only at clustering nodes while LSCD fits all parameters.
\speclore$_{R=1}$ was the fastest due to a single round of refinement which incurred the cost of slightly inferior accuracy. 
However, it still notably improved the accuracy of spectral clustering.
All reported results were obtained on a Windows 7 PC with two Intel Xeon Processors (E5-2630 v3@2.40GHz) and $64$G RAM.




\begin{table}[htb]
	\begin{center}
		\begin{tabular}{c|c|rr|r|rr|rr}
			\hline
			\hline
			& & \multicolumn{2}{|c|}{LSCD} & Initial~ & \multicolumn{2}{|c|}{{\speclore}$_{R=1}$} & \multicolumn{2}{|c}{{\speclore}$_{R=10}$}  \\
			\hline
			Dataset & \# Clusters & error & time & error & error & time &  error & time \\
			\hline
			\hline
			Political blog & 2 & 4.91\% & 43.31 & 5.32\% & 4.66\% & 0.62 & 4.66\% & 0.97 \\
			
			Simmons& 4 & 11.87\% & 39.90 & 13.54\% & 11.61\% & 1.94 & 11.17\%  & 2.65 \\
			
			Caltech & 8 & 18.14\% & 11.85 & 21.69\% & 17.46\% & 0.87 & 14.58\% & 1.29 \\
			Company & 4 & 1.35\% & 0.83  & 5.41\%  & 2.70\%  &  0.01 & 1.35\%  & 0.02 \\
			High school & 9 & 0.61\% & 5.29  & 0.61\%  & 0.61\%  & 0.13  & 0.61\%  & 0.24 \\			
			\hline
			\hline
		\end{tabular}
	\end{center}
	\vspace{-0.15in} 
	\caption{A summary of performances on five datasets.
	Each ``error'' column reports proportions of misclustered nodes.
	Each ``time'' column reports runtime of the corresponding method in seconds (including initialization). 
	}
	\label{tab:three-datasets}
\end{table}

\section{Discussions}
\label{sec:discuss}

In this paper, we study theoretical and empirical performances of a simple community detection algorithm in the context of sparse latent space models.
We establish consistency and derive rates of convergence of the method for sparse latent eigenmodels with two balanced communities. 
Under an additional eigenvector assumption (Assumption \ref{assump:mu-Sigma-H}), we further argue that our rate has sharp exponent in some minimax sense.
Although we have centered our theoretical investigations on balanced two community case, the method performs well empirically in more general scenarios.

We have focused on the case where one only observes a network structure among $n$ nodes. 
An important advantage of latent space models is the convenience to further include node and/or edge covariates \cite{hoff2002latent,ma2017exploration,levina2017}.
Though it is beyond the scope of the present paper, it is nonetheless desirable to understand how the presence of covariates could affect community detection on nodes.
Furthermore, whether there is covariate or not, it is of interest to explore information-theoretic limits and optimal algorithms for community detection when Assumption \ref{assump:mu-Sigma-H} fails.

%

\bibliographystyle{abbrvnat} 
\bibliography{network}

\newpage
\appendix

\section{Simulation studies}
\label{sec:simulation}

In this section,   
we evaluate numerical performance of both {\speclore} and Algorithm~\ref{alg:provable} on simulated examples generated according to different parameter specifications of the latent space model.
All reported results were obtained on a Windows 7 PC with two Intel Xeon Processors (E5-2630 v3@2.40GHz) and $64$G RAM.

\paragraph{Specification 1}
We first consider the case where $H$ is positive semi-definite. 
In this case, we compare both {\speclore} and Algorithm~\ref{alg:provable} with the LSCD method in Section 6.1 of \cite{ma2017exploration}.

We set up model \eqref{eq:model} with latent space dimension $d=3$ and size $n=1000$. The nodes were split into two clusters of sizes $n_1=n_2=500$.
For $i=1,\cdots,n_1$, we generated i.i.d.~$z_i \sim N_d(\mu,\tau^2 I_d)$, where $\mu=(0.5,1,0)^\top$, and for $i=n_1+1,\cdots, n$, we generated i.i.d.~$z_i\sim N_d(-\mu,\tau^2 I_d)$. 
We varied $\tau\in\{0.75, 0.5, 0.25\}$. 
In addition, we let 
$H=\mathrm{diag}(1,1,0.5)$, and generated $\alpha_i=\alphaavg + \omega_i$, where $\alphaavg=-2.49$ (so that {the median degree} $ne^{2\alphaavg} = \log n$) and $\omega_i\stackrel{iid}{\sim} N(0,1)$. 
We have designed the setting so that $\mu$ is an eigenvector of $H$ with positive eigenvalue $1$. 
In each repetition, we generated one copy of the adjacency matrix $A$ with diagnoals $A_{ii} = 0$ for $ i \in [n]$. 
Then we applied the {\speclore} method with $R=1$ and $R=10$ rounds of local refinement to cluster nodes. 
We also ran Algorithm \ref{alg:provable} to investigate its numerical difference from {\speclore}.
For LSCD, we used Algorithm 3 in \cite{ma2017exploration} as the initializer, then applied Algorithm 1 in \cite{ma2017exploration} with $800$ iterations followed by $k$-means clustering.  

Table \ref{tab:simu1} reports
average misclustering proportions \eqref{eq:loss} over $100$ repetitions and average runtimes (in seconds) of {\speclore} (denoted ``{\speclore}'' with subscripts $R=1$ and $R=10$), Algorithm \ref{alg:provable} and LSCD. 
The runtime of {\speclore} included time spent on spectral initialization by Algorithm \ref{alg:init}.
It also reports average degrees (namely the average of  $\frac{1}{n}\sum_{i=1}^n\sum_{j=1}^n A_{ij}$ over $100$ repetitions). 
Furthermore, it reports theoretical Bayes risks, which are best possible misclustering errors if we observe the latent positions directly and know the underlying distributions that generated the $z_i$'s.  
Bayes risk is only attainable by reconstructing the underlying distributions based on infinite samples directly observed from the latent variable distributions. 
Finally, the ``Initial'' column reports the average errors of the initial estimates obtained from  Algorithm \ref{alg:init}. 

\begin{table}[htb] 
	\begin{center}
		\begin{tabular}{c|c|c|cc|c|c|cc|cc}
			\hline
			\hline
			\multirow{2}{*}{$\tau$} & Avg & Bayes & \multicolumn{2}{|c|}{LSCD} & Algo3 & Initial & \multicolumn{2}{|c|}{{\speclore}$_{R=1}$} & \multicolumn{2}{|c}{{\speclore}$_{R=10}$}\\
			\cline{4-11}
			 & degree & risk & error & time & error & error & error & time & error & time \\
			\hline
			\hline
			0.75 & 47.68 & 6.80\% & 8.03\% & 179.29 & 8.27\% & 8.33\% & 8.21\% & 2.10 & 8.20\% & 2.72 \\
			0.5 & 35.28 & 1.27\% & 2.93\% & 184.31 & 3.20\% & 3.44\% & 3.18\% & 2.07 & 3.18\% & 2.63 \\
			0.25 & 29.51 & 3.87E-4\% & 0.82\% & 182.72 & 0.84\% & 1.36\% & 0.85\% & 2.02 & 0.83\% & 2.63 \\
			\hline
			\hline
		\end{tabular}
	\end{center}
\vspace{-0.15in} 
\caption{Misclustering proportions and runtimes in Specification 1.
}
\label{tab:simu1}
\end{table}

For all three values of $\tau$, misclustering errors of {\speclore} with $R=10$ and LSCD were close, but runtimes of the former method were only tiny proportions of those of the latter.
We also observe that misclustering errors of {\speclore} with $R=1$ were nearly identical to those of Algorithm \ref{alg:provable}. 
{This reassures that repeated initializations in Algorithm \ref{alg:provable} were only needed for technical reasons in proofs,}
and justifies the use of {\speclore} in practice. 
Furthermore, for $\tau=0.75$, the misclustering errors of {\speclore} were close to Bayes risk, while for $\tau=0.25$ the misclustering errors of {\speclore} were much larger than Bayes risk. 
This suggests that when $\tau$ is large, the signal-to-noise ratio affected by the latent positions dominates the error rate, while when $\tau$ is small, the signal-to-noise ratio affected by the network sparsity dominates. 

\paragraph{Specification 2}
In the second study, we kept the same settings as in the first case except that we set $H = \mathrm{diag}(1,1,-0.5)$ which is no longer positive semi-definite, while $\mu$ is still an eigenvector of $H$ with eigenvalue $1$. 
In this case, the LSCD method cannot be directly applied, and so we did not report its results in this case. 
Table \ref{tab:simu2} reports all the other columns in Table \ref{tab:simu1} in the present setting. 
Overall, misclustering errors and runtimes of various algorithms in this setting were almost identical to those in the first study. 

\begin{table}[htb]
	\begin{center}
		\begin{tabular}{c|c|c|c|c|cc|cc}
			\hline
			\hline
			\multirow{2}{*}{$\tau$} & Avg & Bayes & Algo3 & Initial & \multicolumn{2}{|c|}{{\speclore}$_{R=1}$} & \multicolumn{2}{|c}{{\speclore}$_{R=10}$} \\
			\cline{4-9}
			 & degree & risk & error & error & error & time & error & time \\
			\hline
			\hline
			0.75 & 47.85 & 6.80\% & 8.25\% & 8.28\% & 8.18\% & 2.13 & 8.16\% & 2.68 \\
			0.5 & 35.41 & 1.27\% & 3.16\% & 3.44\% & 3.16\% & 2.18 & 3.14\% & 2.73  \\
			0.25 & 29.51 & 3.87E-4\% & 0.82\% & 1.31\% & 0.85\% & 2.12 & 0.79\% & 2.65  \\
			\hline
			\hline
		\end{tabular}
	\end{center}
	\vspace{-0.15in} 
	\caption{Misclustering proportions and runtimes in Specification 2.
	}
	\label{tab:simu2}
\end{table}

\paragraph{Specification 3}
In the third study, the settings remained the same as in the first study except that we fixed $\tau=0.5$ and let $\alphaavg\in \{-2.14, -2.49, -2.83\}$, which calibrated the median degree of networks to be around $\{2, 1, 0.5 \}\times \log n$, respectively. 
Table \ref{tab:simu3} reports the results for all three different $\alphaavg$'s.  
As $|\alphaavg|$ grows, the average degree decreases significantly. 
Misclustering errors of {\speclore} with $R=10$ were slightly worse than those of the LSCD method, but were always within $110\%$ of the LSCD errors. 
On the other hand, runtimes of {\speclore} with $R=10$ were of smaller order of magnitude than those of LSCD. 
Misclustering errors of {\speclore} were comparable to Bayes risk when ${\alphaavg}=-2.14$, and became more sizeable relative to Bayes risk for larger $\alphaavg$. 
This suggests that network sparsity becomes the dominating factor in error rate as $|{\alphaavg}|$ grows.

\begin{table}[htb]
	\begin{center}
		\begin{tabular}{c|c|c|cc|c|c|cc|cc}
			\hline
			\hline
			\multirow{2}{*}{$\alphaavg$} & Avg & Bayes & \multicolumn{2}{|c|}{LSCD} & Algo3 & Initial & \multicolumn{2}{|c|}{{\speclore}$_{R=1}$} & \multicolumn{2}{|c}{{\speclore}$_{R=10}$} \\
			\cline{4-11}
			 & degree & risk & error & time & error & error & error & time & error & time \\
			\hline
			\hline
			-2.14 & 58.86 & 1.27\% & 2.04\% & 219.92 & 2.24\% & 2.27\% & 2.25\% & 2.04 & 2.23\% & 2.59  \\
			-2.49 & 35.28 & 1.27\% & 2.93\% & 211.29 & 3.20\% & 3.44\% & 3.18\% & 2.31 & 3.17\% & 2.86  \\
			-2.83 & 20.30 & 1.27\% & 4.58\% & 213.31 & 4.94\% & 6.04\% & 4.91\% & 2.26 & 4.88\% & 2.85  \\
			\hline
			\hline
		\end{tabular}
	\end{center}
	\vspace{-0.15in} 
	\caption{Misclustering proportions and runtimes in Specification 3.
	}
	\label{tab:simu3}
\end{table}

\paragraph{Specification 4}
Finally, we repeated the last two studies with $H=\mathrm{diag}(1,1,-0.5)$ and $\mu = \sqrt{{1.25}/{1.29}}\left(0.5,1,0.2\right)^\top$. In this case, $\mu$ is no longer an eigenvector of $H$ but $\|\mu\|_2$ is the same as in Specifications 1--3 to make the results more comparable.
Table \ref{tab:simu4} summarizes the relevant results for all different combinations of $\tau$ and $\alphaavg$ values. 
We observe that the first three rows had slightly larger misclustering errors than those in Tables \ref{tab:simu1} and \ref{tab:simu2}, and the last three rows had slightly larger misclustering errors than those in Table \ref{tab:simu3}.
Such a difference conforms with our theory since quantity $\rho$ (defined in \eqref{eq:rho}) in \eqref{eq:rate-upper}--\eqref{eq:rate-lower}  becomes smaller when $\mu$ is no longer an eigenvector of $H$ with maximum possible eigenvalue $1$ under \eqref{eq:H-eig}, resulting in larger error rates.

\begin{table}[htb]
	\begin{center}
		\begin{tabular}{c|c|c|c|c|c|cc|cc}
			\hline
			\hline 
			\multirow{2}{*}{$\tau$} & \multirow{2}{*}{$\alphaavg$} & Avg & Bayes & Algo3 & Initial & \multicolumn{2}{|c|}{{\speclore}$_{R=1}$} & \multicolumn{2}{|c}{{\speclore}$_{R=10}$} \\
			\cline{5-10}
			 &  & degree & risk & error & error & error & time & error & time \\
			\hline
			\hline
			0.75 & -2.49 & 46.34 & 6.80\% & 8.89\% & 8.89\% & 8.83\% & 2.27 & 8.80\% & 2.82  \\
			0.5 & -2.49 & 34.09 & 1.27\% & 3.63\% & 3.93\% & 3.62\% & 2.16 & 3.62\% & 2.71 \\
			0.25 & -2.49 & 28.55 & 3.87E-4\% & 0.97\% & 1.56\% & 1.01\% & 2.11 & 1.00\% & 2.68  \\
			\cline{1-10} 
			0.5 & -2.14 & 57.64 & 1.27\% & 2.55\% & 2.60\% & 2.53\% & 2.07 & 2.53\% & 2.63  \\
			0.5 & -2.49 & 34.09 & 1.27\% & 3.51\% & 3.93\% & 3.62\% & 2.16 & 3.62\% & 2.71  \\
			0.5 & -2.83 & 19.72 & 1.27\% & 5.35\% & 6.45\% & 5.33\% & 2.15 & 5.27\% & 2.73  \\
			\hline
			\hline
		\end{tabular}
	\end{center}
	\vspace{-0.15in} 
	\caption{Misclustering proportions and runtimes in Specification 4.
	}
	\label{tab:simu4}
\end{table}


\section{Proof of Lemma \ref{lem:edgecount-rate}}
\label{sec:proofedgecountapp}

    We note that, by Jensen's inequality, for any fixed $\epsilon \in (0, 1/2)$,
    \begin{equation*}
        \countrateuppereps{\epsilon} \le (\countrateuppereps{0})^{1-\epsilon} \rightarrow 0, \quad \text{ as } n \rightarrow \infty.
    \end{equation*}

    By symmetry, we have
\begin{align*}
    \nu_n & =  \Prob_{H_0}(A_{0,+}<A_{0,-})+\Prob_{H_1}(A_{0,+}\ge A_{0,-}) \\
& =  \Prob_{H_0}(A_{0,+}<A_{0,-})+\Prob_{H_0}(A_{0,+}\le A_{0,-}).
\end{align*}
Hence
\begin{align}\label{eq:nu_inequalities}
\Prob_{H_0}(A_{0,+}\le A_{0,-}) \le \nu_n  \le 2\Prob_{H_0}(A_{0,+}\le A_{0,-}).
\end{align}

\paragraph{Upper bound}
By law of total expectation,
\begin{align*}
\Prob_{H_0}(A_{0,+}\le A_{0,-}) = \Expect_{H_0}^{\alpha_0, z_0} \big[ \Prob(A_{0,+}\le A_{0,-}|\alpha_0,z_0) \big].
\end{align*}
Let
\[
\Omega=\left\{\{a_{0,i}\}_{i=1}^{2m}: a_{0,i}\in\{0,1\} \mathrm{\ for \ } 1\le i\le 2m, \sum_{i=1}^m a_{0,i}\le \sum_{i=m+1}^{2m} a_{0,i}\right\}.
\]
We then have
\begin{align*}
\Prob(A_{0,+}\le A_{0,-}|\alpha_0, z_0) & =
\sum_{\{a_{0,i}\}_{i=1}^{2m}\in\Omega} \Prob(A_{0,1}=a_{0,1},\cdots,A_{0,2m}=a_{0,2m}|\alpha_0, z_0) \\
& =
\sum_{\{a_{0,i}\}_{i=1}^{2m}\in\Omega}
\Expect^{\{\alpha_i, z_i \}_{i=1}^{2m}}
\big[\Prob(A_{0,1}=a_{0,1},\cdots,A_{0,2m}=a_{0,2m}|\alpha_0, z_0,\{\alpha_i, z_i \}_{i=1}^{2m}) \big] \\
& =
\sum_{\{a_{0,i}\}_{i=1}^{2m}\in\Omega}
\Expect^{\{\alpha_i, z_i\}_{i=1}^{2m}}
\left[\prod_{i=1}^{2m} \Prob(A_{0,i}=a_{0,i}|\alpha_0,z_0,\alpha_i, z_i) \right] \\
& =
\sum_{\{a_{0,i}\}_{i=1}^{2m}\in\Omega}\prod_{i=1}^{2m}
\Expect^{\alpha_i, z_i}
\big[ \Prob(A_{0,i}=a_{0,i}|\alpha_0,z_0,\alpha_i, z_i) \big].
\end{align*}
Here $\Expect^{\alpha_i, z_i}$ means the expectation over $\alpha_i$ and $z_i$ (under $H_0$).
In the last equality, we have used the mutual independence of $\{\alpha_i,z_i \}$ for $1\le i\le 2m$.
By the discussion preceding \eqref{eq:onenode-testing} and the definition in \eqref{eqn:def-p-z0} and \eqref{eqn:def-q-z0},
we have
\begin{align*}
\Expect^{\alpha_i,z_i} \big[
\Prob(A_{0,i}=1|\alpha_0,z_0,\alpha_i, z_i)
\big]
=
\begin{cases}
p(\alpha_0,z_0), \mathrm{\ for \ } 1\le i\le m, \\
q(\alpha_0,z_0), \mathrm{\ for \ } m+1\le i\le 2m.
\end{cases}
\end{align*}
By definition, $p(\alpha_0,z_0)$ and $q(\alpha_0,z_0)$ can be written as
\begin{align}
\label{eqn:def-p-z0-new}
p(\alpha_0,z_0) & = \Expect^{\alpha_1,z_1} S(z_0^\top H z_1 + \alpha_0 + \alpha_1 ), \\
q(\alpha_0,z_0) & = \Expect^{\alpha_{m+1},z_{m+1}} S(z_0^\top H z_{m+1} + \alpha_0 + \alpha_{m+1} ) \nonumber \\
 & = \Expect^{\alpha_1,z_1} S(-z_0^\top H z_1 + \alpha_0 + \alpha_1).
\label{eqn:def-q-z0-new}
\end{align}
Here $\alpha_i \stackrel{iid}{\sim} F_\alpha$, $z_1\sim N(\mu,\tau^2 I_d)$ and $z_{m+1}\sim N(-\mu,\tau^2 I_d)$, and they are mutually independent.
Define $\mathcal{L}_+=\{z_0: z_0^\top H\mu\ge 0\}$ and $\mathcal{L}_-=\{z_0: z_0^\top H\mu< 0\}$. Conditional on $\alpha_0$ and $z_0$, the distribution of $z_0^\top H(z_1-\mu)$ is symmetric about zero and is independent of $\alpha_1$. Since $S$ is a monotone increasing function, together with \eqref{eqn:def-p-z0-new} and \eqref{eqn:def-q-z0-new}, this observation implies that $p(\alpha_0,z_0)\ge q(\alpha_0,z_0)$ when $z_0\in \mathcal{L}_+$ and $p(\alpha_0,z_0)<q(\alpha_0,z_0)$ when $z_0\in \mathcal{L}_-$.


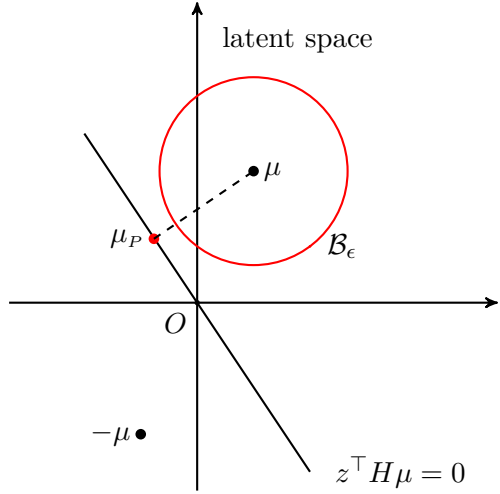
\begin{figure}[!tb]
	\centering
\begin{tikzpicture}[
    scale=5,
    axis/.style={thick, ->, >=stealth'},
    important line/.style={thick},
    dashed line/.style={dashed, thin},
    pile/.style={thick, ->, >=stealth', shorten <=2pt, shorten
    >=2pt},
    every node/.style={color=black}
    ]
    \draw[axis] (0,0.5)  -- (1.3,0.5) node(xline)[right]
        {$~$};
    \draw[axis] (0.5,0) -- (0.5,1.3) node(yline)[below right = 0.5em] {latent space};
	\fill[black] (0.5,0.5) circle (.2pt) node[below left] {$O$};
	\fill[black] (0.65,0.85) coordinate (M) circle (.4pt) node[right] {$\mu$};
	\fill[black] (0.35,0.15) coordinate (N) circle (.4pt) node[left] {$-\mu$};
    \draw[thick] (.2,.95) coordinate (A) -- (.8,.05)
        coordinate (B) node[right = 0.5em] {$z^\top H\mu = 0$};
	\fill[red] (0.385,0.67) coordinate (P) circle (.4pt) node[left] {$\mu_{\sss P}$};
	\draw[dashed,thick] (P) -- (M);
	\draw[red,thick] (M) circle (0.25cm); 
	\fill[red] (0.8, 0.65) coordinate (B) circle (0.1pt) node[right = 0.25em] {$\mathcal{B}_\epsilon$};
\end{tikzpicture}
\caption{An illustration of a $\mathcal{B}_\epsilon$-ball in the latent space: $\mu_{\sss P}$ is the orthogonal projection of $\mu$ onto the hyperplane $\{z: z^\top H\mu = 0\}$ with the distance between $\mu$ and $\mu_{\sss P}$ equal to $\rho$ defined in \eqref{eq:rho}. Given $\epsilon > 0$, $\mathcal{B}_\epsilon$ is the ball in red with radius $\sqrt{1-\frac{\epsilon}{4}}\rho$. } 
\label{fig:ball}	
\end{figure}

For any $z_0\in \mathcal{B}_\epsilon$, we have
\begin{align}
z_0^\top H \mu = & \mu^\top H \mu + (z_0-\mu)^\top H \mu \nonumber \\
\ge & \mu^\top H \mu - |(z_0-\mu)^\top H \mu| \nonumber \\
\ge & \mu^\top H \mu - \|H \mu\|_2 \|z_0-\mu\|_2  \nonumber \\
\ge & \mu^\top H \mu - \sqrt{\mu^\top H^2 \mu} \sqrt{1-\frac{\epsilon}{4}}\rho \nonumber \\
= & \left(1-\sqrt{1-\frac{\epsilon}{4}}\right) \mu^\top H \mu \nonumber \\
\ge & \frac{\epsilon}{8} \mu^\top H \mu. \label{eq:z0-H-mu-bound}
\end{align}
Here the second equality holds due to \eqref{eq:rho}.
Thus $\mathcal{B}_\epsilon\subset \mathcal{L}_+$.
See Figure \ref{fig:ball} for a graphical illustration.

Next, we derive uniform bounds of $p(\alpha_0,z_0)$, $q(\alpha_0,z_0)$ and $I(\alpha_0,z_0)$ for all $z_0\in\calB_\epsilon$.
To this end, define
\begin{align*}
D_p(\omega_0,z_0)= & \Expect^{\omega_1,z_1}\left[ \frac{e^{z_0^\top H (z_1-\mu)+\omega_0+\omega_1 }}{1+e^{z_0^\top H z_1+2\alphaavg+\omega_0+\omega_1 }}\right], \quad D_q(\omega_0,z_0)= \Expect^{\omega_1,z_1}\left[ \frac{e^{-z_0^\top H (z_1-\mu)+\omega_0+\omega_1 }}{1+e^{-z_0^\top H z_1+2\alphaavg+\omega_0+\omega_1 }}\right].
\end{align*}
By \eqref{eqn:def-p-z0-new} and \eqref{eqn:def-q-z0-new}, we have
\begin{align}
p(\alpha_0,z_0)= & e^{2\alphaavg}e^{z_0^\top H\mu} D_p(\omega_0,z_0) \label{eq:p-decomposition} \\
q(\alpha_0,z_0)= & e^{2\alphaavg}e^{-z_0^\top H\mu} D_q(\omega_0,z_0). \label{eq:q-decomposition}
\end{align}
To find upper bounds for $D_p(\omega_0,z_0)$ and $D_q(\omega_0,z_0)$, we define
\begin{align*}
D(\omega_0,z_0)= & \Expect^{\omega_1,z_1}\left[ e^{z_0^\top H (z_1-\mu)+\omega_0+\omega_1 }\right]=e^{\omega_0}\Expect(e^{\omega_1})\Expect^{z_1}[e^{z_0^\top H (z_1-\mu)}].
\end{align*}
Then we have
\begin{align}
D_p(\omega_0,z_0)\le & \Expect^{\omega_1,z_1}\left[ e^{z_0^\top H (z_1-\mu)+\omega_0+\omega_1}\right] = D(\omega_0,z_0) \label{eq:Dp-upper-bound}\\
D_q(\omega_0,z_0)\le & \Expect^{\omega_1,z_1}\left[ e^{-z_0^\top H (z_1-\mu)+\omega_0+\omega_1}\right] = D(\omega_0,z_0). \label{eq:Dq-upper-bound}
\end{align}
where the last equality holds since the distribution of $z_1-\mu$ is symmetric about zero. By Assumption \ref{assump:alpha}, $\Expect[e^{\omega_1}]\le \left(\Expect[e^{2\omega_1}]\right)^{1/2}\le C^{1/2}$. This inequality, combined with the boundedness of $z_0$ for $z_0\in \calB_\epsilon$ and \eqref{eq:omegabar} of Assumption \ref{assump:alpha} implies that
\begin{align}\label{eq:D-bound}
0< e^{-2\underline{\omega}} \underline{D}\le D(\omega_0,z_0)\le e^{\omegabar}\overline{D},
\end{align}
where $\overline{D}$ and $\underline{D}$ are constants.

On the other hand, to find lower bounds for $D_{p}(z_0, \omega_0)$ and $D_{q}(z_0, \omega_0)$, we define
\begin{align*}
D_2(\omega_0,z_0) = & \Expect^{\omega_1,z_1}\left[ e^{2z_0^\top H (z_1-\mu)+2\omega_0+2\omega_1}\right] = e^{2\omega_0}\Expect(e^{2\omega_1}) \Expect^{z_1}\left[ e^{2z_0^\top H (z_1-\mu)}\right].
\end{align*}
By Assumption \ref{assump:alpha}, $\Expect[e^{2\omega_1}]\le C$. Further by \eqref{eq:omegabar} of Assumption \ref{assump:alpha} and boundedness of $z_0$, $D_2(\omega_0,z_0)$ also has an upper bound $e^{2\omegabar}\overline{D}_2$ where $\overline{D}_2$ is a constant. Then
\begin{align}
D(\omega_0,z_0)-D_p(\omega_0,z_0)= & \Expect^{\omega_1, z_1}\left[ e^{z_0^\top H (z_1-\mu)+\omega_0+\omega_1} \left(1-\frac{1}{1+e^{z_0^\top H z_1+2\alphaavg+\omega_0+\omega_1}}\right)\right] \nonumber \\
= & e^{2\alphaavg}e^{z_0^\top H \mu}\Expect^{\omega_1,z_1}\left[ \frac{e^{2z_0^\top H (z_1-\mu)+2\omega_0+2\omega_1}} {1+e^{z_0^\top H z_1+2\alphaavg+\omega_0+\omega_1 }}\right] \nonumber \\
\le & e^{2\alphaavg}e^{z_0^\top H \mu} \Expect^{\omega_1,z_1}\left[ e^{2z_0^\top H (z_1-\mu)+2\omega_0+2\omega_1 }\right] \nonumber \\
= & e^{2\alphaavg}e^{z_0^\top H \mu} D_2(\omega_0,z_0) \nonumber \\
\le & e^{2\alphaavg+2\omegabar}e^{z_0^\top H \mu} \overline{D}_2. \label{eq:Dp-lower-bound-temp1}
\end{align}
Let $0<\kappa<1$ be any fixed constant. By \eqref{eq:alpha1} of Assumption \ref{assump:alpha} and the boundedness of $z_0$ within $\mathcal{B}_\epsilon$, the inequality $e^{2\alphaavg+2\omegabar}e^{z_0^\top H \mu} \overline{D}_2\le \kappa e^{-2\underline{\omega}} \underline{D}$ holds for all sufficiently large $n$.
By \eqref{eq:D-bound},
\begin{align}
e^{2\alphaavg+2\omegabar}e^{z_0^\top H \mu} \overline{D}_2 \le \kappa e^{-2\underline{\omega}}\underline{D} \le  \kappa D(\omega_0,z_0). \label{eq:Dp-lower-bound-temp2}
\end{align}
Combining \eqref{eq:Dp-lower-bound-temp1} and \eqref{eq:Dp-lower-bound-temp2}, we have
\begin{align}
D_p(\omega_0,z_0)\ge (1-\kappa)D(\omega_0,z_0).\label{eq:Dp-lower-bound}
\end{align}
By the same argument, we can also get
\begin{align}
D_q(\omega_0,z_0)\ge (1-\kappa)D(\omega_0,z_0).\label{eq:Dq-lower-bound}
\end{align}

We now derive a lower bound for $I(\alpha_0,z_0)$. By definition, we have
\begin{align*}
I(\alpha_0,z_0)= & -2 \log\left(\sqrt{p(\alpha_0,z_0)q(\alpha_0,z_0)}+ \sqrt{(1-p(\alpha_0,z_0))(1-q(\alpha_0,z_0))}\right) \\
\ge & -2\log\left(\sqrt{p(\alpha_0,z_0)q(\alpha_0,z_0)}+ 1-\frac{1}{2}\left[p(\alpha_0,z_0)+q(\alpha_0,z_0)\right] \right) \\
\ge & -2\sqrt{p(\alpha_0,z_0)q(\alpha_0,z_0)}+ p(\alpha_0,z_0)+q(\alpha_0,z_0) \\
= & e^{2\alphaavg}e^{z_0^\top H \mu} \left(\sqrt{D_p(\omega_0,z_0)}-e^{-z_0^\top H \mu} \sqrt{D_q(\omega_0,z_0)}\right)^2,
\end{align*}
where the last inequality is due to $\log(1-x)\le -x$ for $0<x<1$. We let
\begin{align*}
C(\omega_0,z_0)=e^{z_0^\top H \mu} \left(\sqrt{D_p(\omega_0,z_0)}-e^{-z_0^\top H \mu} \sqrt{D_q(\omega_0,z_0)}\right)^2,
\end{align*}
and let $\kappa=1-\frac{1}{4}(1+e^{-\frac{\epsilon}{8}\mu^\top H\mu})^2$. Then by \eqref{eq:z0-H-mu-bound}, \eqref{eq:Dq-upper-bound} and \eqref{eq:Dp-lower-bound} we get
\begin{align*}
C(\omega_0,z_0) \ge & e^{\frac{\epsilon}{8} \mu^\top H \mu} \left(\sqrt{(1-\kappa)D(\omega_0,z_0)}-e^{-z_0^\top H \mu} \sqrt{D(\omega_0,z_0)}\right)^2 \\
= & e^{\frac{\epsilon}{8} \mu^\top H \mu} D(\omega_0,z_0) \left(\sqrt{1-\kappa}-e^{-z_0^\top H \mu} \right)^2 \\
\ge & e^{\frac{\epsilon}{8} \mu^\top H \mu} D(\omega_0,z_0) \left[\frac{1}{2}\left(1+e^{-\frac{\epsilon}{8}\mu^\top H\mu}\right)-e^{-z_0^\top H \mu} \right]^2 \\
\ge & e^{\frac{\epsilon}{8} \mu^\top H \mu} D(\omega_0,z_0) \left[\frac{1}{2}\left(1+e^{-\frac{\epsilon}{8}\mu^\top H\mu}\right)-e^{-\frac{\epsilon}{8}\mu^\top H \mu} \right]^2 \\
= & e^{\frac{\epsilon}{8} \mu^\top H \mu} D(\omega_0,z_0) \left[\frac{1}{2}\left(1-e^{-\frac{\epsilon}{8}\mu^\top H\mu}\right) \right]^2 \\
\ge & e^{\frac{\epsilon}{8} \mu^\top H \mu} e^{-2\underline{\omega}} \underline{D} \left[\frac{1}{2}\left(1-e^{-\frac{\epsilon}{8}\mu^\top H\mu}\right) \right]^2 :=\underline{C}.
\end{align*}
Since $\underline{D}$ and $\underline{\omega}$ are both constants, $\underline{C}>0$ is also a constant.
In summary, for $z_0\in \calB_\epsilon$, we have established
\begin{align}\label{eq:I-lower-bound}
I(\alpha_0,z_0)\ge e^{2\alphaavg} \underline{C},
\end{align}
where $\underline{C}$ is some constant depending on $\epsilon$. 

In view of the foregoing discussion, we can write
\begin{align}\label{eq:prob_break}
\Prob_{H_0}(A_{0,+}\le A_{0,-})
= &
\Expect_{H_0}^{\alpha_0, z_0}
\left[
\indc{z_0 \in \mathcal{B}_\epsilon }\Prob(A_{0,+}\le A_{0,-}|\alpha_0, z_0 )
\right]
\nonumber \\
& +
\Expect_{H_0}^{\alpha_0, z_0}
\left[
\indc{ z_0\in \mathcal{B}^c_\epsilon }\Prob(A_{0,+}\le A_{0,-}|\alpha_0, z_0 )
\right].
\end{align}
Conditional on $\alpha_0$ and $z_0$, we can generate independent random variables $W_i\sim \mathrm{Bernoulli}(p(\alpha_0,z_0))$ for $i=1,\cdots, m$ and $W_i\sim \mathrm{Bernoulli}(q(\alpha_0,z_0))$ for $i=m+1,\cdots, 2m$. Then we have
$$
\Prob(A_{0,+}\le A_{0,-}|\alpha_0, z_0)=
\Prob\left(\sum_{i=1}^m W_i\le  \sum_{i=m+1}^{2m} W_i \right).
$$

For any $\alpha_0$ and any $z_0\in \calB_\epsilon$, aside from $p(\alpha_0,z_0)>q(\alpha_0,z_0)$, we can also get from \eqref{eq:p-decomposition}, \eqref{eq:q-decomposition}, \eqref{eq:Dp-upper-bound}, \eqref{eq:Dq-upper-bound}, \eqref{eq:D-bound}, $z_0$ bounded, and $\eqref{eq:alpha1}$ of Assumption \ref{assump:alpha} that as $n\to\infty$,
$$
p(\alpha_0,z_0)\to 0,\quad \mbox{and} \quad
q(\alpha_0,z_0)\to 0.
$$
We then obtain from the calculation in \cite{gao2015achieving,gao2018minimax} that
\begin{align*}
\Prob\left(\sum_{i=1}^m W_i\le \sum_{i=m+1}^{2m} W_i \right) \le \exp\left\{-m(1+\eta_1(\alpha_0,z_0))I(\alpha_0,z_0)\right\},
\end{align*}
in which $\eta_1(\alpha_0,z_0)=O(1/\sqrt{mI(\alpha_0,z_0)})$. By \eqref{eq:I-lower-bound} and \eqref{eq:alpha4} of Assumptions \ref{assump:alpha}, we have $1/\sqrt{mI(\alpha_0,z_0)}\le 1/\sqrt{m e^{2\alphaavg}\underline{C}}\to 0$. Then $-\eta_1(\alpha_0,z_0)\le \frac{\epsilon}{2}$ for all sufficiently large $n$. Therefore,
\begin{align*}
\Prob\left(\sum_{i=1}^m W_i\le \sum_{i=m+1}^{2m} W_i\right)
\le \exp\left\{-m \left(1-\frac{\epsilon}{2}\right)I(\alpha_0,z_0)\right\}.
\end{align*}

Note that $z_0\sim N(\mu,\tau^2 I)$ under $H_0$, we have $\|z_0-\mu\|_2^2/\tau^2\sim \chi^2(d)$. Since $\tau\to 0$ as $n\to \infty$, the inequality below holds for all sufficiently large $n$:
\begin{align*}
\left(1-\frac{\epsilon}{4} \right)\frac{\rho^2}{\tau^2}\ge d+2\sqrt{d\left(1-\frac{\epsilon}{2}\right)\frac{\rho^2}{2\tau^2}}
+\left(1-\frac{\epsilon}{2} \right)\frac{\rho^2}{\tau^2}.
\end{align*}
Then by Lemma 1 of \cite{massart2000chisquare}, we can get
\begin{align}
\Prob_{H_0}(z_0\in \calB_\epsilon^c)
= & \Prob_{H_0}\left(\frac{1}{\tau^2}\|z_0-\mu\|_2^2 > \left(1-\frac{\epsilon}{4} \right)\frac{\rho^2}{\tau^2}\right) \nonumber \\
\le & \Prob_{H_0}\left(\frac{1}{\tau^2}\|z_0-\mu\|_2^2 \ge d+2\sqrt{d\left(1-\frac{\epsilon}{2}\right)\frac{\rho^2}{2\tau^2}}
+\left(1-\frac{\epsilon}{2}\right)\frac{\rho^2}{\tau^2}\right) \nonumber \\
\le & \exp\left\{-\left(1-\frac{\epsilon}{2}\right)\frac{\rho^2}{2\tau^2}\right\}. \label{eq:ball_bound}
\end{align}
Therefore by \eqref{eq:prob_break},
\begin{align}\label{eq:edge_upperbound}
& \Prob_{H_0}(A_{0,+}\le A_{0,-})  \nonumber \\
&  \le \Expect^{\alpha_0,z_0}_{H_0}
\left[\indc{ z_0\in \calB_\epsilon} \exp\left\{-m\left(1-\frac{\epsilon}{2}\right)I(\alpha_0,z_0)\right\}\right] + \Prob_{H_0}(z_0\in \calB_\epsilon^c) \nonumber \\
& \le \Expect^{\alpha_0,z_0}_{H_0}
\left[\indc{ z_0\in \calB_\epsilon }
\exp\left\{-m\left(1-\frac{\epsilon}{2} \right)I(\alpha_0,z_0)\right\}\right] +
 \exp\left\{-\left(1-\frac{\epsilon}{2}\right)\frac{\rho^2}{2\tau^2}\right\}.
\end{align}
Combining \eqref{eq:edge_upperbound} with the second inequality of \eqref{eq:nu_inequalities}, we get
\begin{align*}
\nu_n\le & 2 \Expect^{\alpha_0,z_0}_{H_0}
\left[\indc{ z_0\in \calB_\epsilon } \exp\left\{-m\left(1-\frac{\epsilon}{2}\right)I(\alpha_0,z_0)\right\}\right] +
2\exp\left\{-\left(1-\frac{\epsilon}{2}\right)\frac{\rho^2}{2\tau^2}\right\} \\
\le & \Expect^{\alpha_0,z_0}_{H_0}
\left[\indc{ z_0\in \calB_\epsilon } \exp\left\{-m(1-\epsilon)I(\alpha_0,z_0)\right\}\right] +
\exp\left\{-(1-\epsilon)\frac{\rho^2}{2\tau^2}\right\}.
\end{align*}
Here the last inequality holds because $\frac{\epsilon}{2}>\frac{\log 2}{m e^{2\alphaavg}\underline{C}}\ge \frac{\log 2}{mI(\alpha_0,z_0)}$ by \eqref{eq:alpha4} of Assumption~\ref{assump:alpha} and $\frac{\epsilon}{2}>\frac{2\tau^2}{\rho^2}\log 2$ for all sufficiently large $n$.

\paragraph{Lower bound}
For the lower bound, when $z_0\in \calB_\epsilon$, we apply the Chernoff argument in \cite{gao2015achieving,gao2018minimax}
to get
\begin{align*}
\Prob(A_{0,+}\le A_{0,-}|\alpha_0, z_0) \ge \exp\{-m(1+\eta_2(\alpha_0,z_0))I(\alpha_0,z_0)\}.
\end{align*}
in which $\eta_2(\alpha_0,z_0)=O(1/\sqrt{mI(\alpha_0,z_0)})$. By \eqref{eq:I-lower-bound} and \eqref{eq:alpha4} of Assumption \ref{assump:alpha}, we get $\eta_2(\alpha_0,z_0)\le \epsilon$ for all sufficiently large $n$. Therefore,
\begin{align*}
\Prob\left(\sum_{i=1}^m W_i\le \sum_{i=m+1}^{2m} W_i \right) \ge \exp\left\{-m(1+\epsilon)I(\alpha_0,z_0)\right\}.
\end{align*}
Note $\calL_-\subset \calB_\epsilon^c$. When $z_0\in \mathcal{L}_-$, we have $p(\alpha_0,z_0)<q(\alpha_0,z_0)$, so
$$\Prob(A_{0,+}\le A_{0,-}|\alpha_0, z_0)\ge \frac{1}{2}.$$
Also,
\begin{align*}
\Prob_{H_0} (z_0\in \mathcal{L}_-)= & \Prob_{H_0} \left((z_0-\mu)^\top H \mu<-\mu^\top H \mu\right) \\
= & \Phi\left(-\frac{\mu^\top H \mu}{\tau\sqrt{\mu^\top H^2\mu}}\right)
\ge  \exp\left\{-(1+\frac{\epsilon}{2})\frac{\rho^2}{2\tau^2}\right\},
\end{align*}
where the last inequality is due to Mill's ratio. Therefore by \eqref{eq:prob_break} again,
\begin{align}\label{eq:edge_lowerbound}
& \Prob_{H_0}(A_{0,+}\le A_{0,-})
\nonumber
\\
& \ge \Expect_{H_0}^{\alpha_0, z_0}
\left[
\indc{z_0\in \mathcal{B}_\epsilon} \exp\left\{-m(1+\epsilon)I(\alpha_0,z_0)\right\}
\right]+ \frac{1}{2}\Prob_{H_0} (z_0\in \mathcal{L}_-) \nonumber\\
& \ge
\Expect_{H_0}^{\alpha_0,z_0}
\left[\indc{ z_0\in \mathcal{B}_\epsilon } \exp(-m(1+\epsilon)I(\alpha_0,z_0))\right]
+ \frac{1}{2}
 \exp\left\{-\left(1+\frac{\epsilon}{2} \right)\frac{\rho^2}{2\tau^2}\right\}
\nonumber \\
& \ge
\Expect_{H_0}^{\alpha_0, z_0}
\left[\indc{z_0\in \mathcal{B}_\epsilon} \exp\left\{-m(1+\epsilon)I(\alpha_0,z_0)\right\}
\right]+ \exp\left\{-(1+\epsilon)\frac{\rho^2}{2\tau^2}\right\}.
\end{align}
Here the last inequality holds because $\frac{\epsilon}{2}\ge \frac{2\tau^2}{\rho^2}\log 2$ for sufficiently large $n$. Combining \eqref{eq:edge_lowerbound} and the first inequality in \eqref{eq:nu_inequalities}, we obtain the first inequality in \eqref{eq:rate-sandwich}.


\section{Proof of Proposition~\ref{prop:init-error}}
The following lemma will be useful in the proof of Proposition \ref{prop:init-error}.
\begin{lemma}
Suppose a $d$-dimensional random vector $z \sim N(\mu, \tau^2 I_d)$. 
Let $M$ be a positive constant. 
Conditional on the event $\| z - \mu\|_2 \le \eta$ with $\eta/\tau \rightarrow \infty$ and $\tau \rightarrow 0$, we have, for $\|t \|_2 \le M$,
	\[
	\wt{\bbE}[ \exp(z^\top t)] 
	= \exp\left(\mu^\top t + \frac{\tau^2 t^\top t}{2} \right) \bigl( 1 - o(1)\bigr),
	\]  
	where $C$ is a constant and $\wt{\bbE}$ denotes the expectation taken over the conditional measure of $z$ on $\| z - \mu \|_2 \le \eta$.
	\label{lem:limit-mgf}
\end{lemma}
\begin{proof}
	Without loss of generality, we assume $\mu = 0$.
	We calculate
	\begin{align*}
	\wt{\bbE}[\exp(z^\top t)] & = \frac{\int_{\|z \|_2 \le \eta} \exp( z^\top t) \exp\bigl(- z^\top z/(2\tau^2)\bigr)/\bigl(\sqrt{2 \pi \tau^2}\bigr)^d dz }
	{\int_{\|z \|_2 \le \eta} \exp\bigl(- z^\top z/(2\tau^2)\bigr)/\bigl(\sqrt{2 \pi \tau^2}\bigr)^d dz} 
	\\
	& = \exp(\tau^2 t^\top t/2) 
	\frac{  \int_{\| z + \tau t\|_2 \le \eta/\tau} \exp\bigl(-z ^\top z/2\bigl)/\bigl( \sqrt{2 \pi }\bigr)^d dz} 
	{ \int_{\| z\|_2 \le \eta/\tau} \exp(-z^\top z/2)/\bigl( \sqrt{2 \pi} \bigr)^d dz }. 
	\end{align*}
	Denote the probability measure of $N(0, I_d)$ by $\bbP_0$ and we define
	\begin{gather*}
	A  : =  \int_{\| z + \tau t\|_2 \le \eta/\tau} \exp\bigl(-z^\top z/2\bigl)/\bigl( \sqrt{2 \pi }\bigr)^d dz = \bbP_0( \| z + \tau t \|_2 \le \eta/\tau),\\ 
	B  :=  \int_{\| z\|_2 \le \eta/\tau} \exp(-z^\top z/2)/\bigl( \sqrt{2 \pi} \bigr)^d dz = \bbP_0( \| z \|_2 \le \eta/\tau) = \bbP( \chi_d^2 \le (\eta/\tau)^2).
	\end{gather*}
	We note that 
	\[
	\bbP\bigl(\chi^2_d \le (\eta/\tau - \tau \|t\|_2 )^2\bigr) = \bbP_0( \|z\|_2 \le \eta/\tau - \tau \| t\|_2 ) \le A \le \bbP_0( \| z \|_2 \le \eta/\tau ) = \bbP\bigl( \chi^2_d \le (\eta/\tau)^2 \bigr).
	\]
	As a result, we bound
	\begin{align*}
	1 \ge \frac{A}{B} \ge  \frac{ \bbP\bigl( \chi^2_d \le (\eta/\tau - \tau \|t\|_2)^2 \bigr) }
	{  \bbP\bigl( \chi_d^2 \le (\eta/\tau)^2 \bigr)} = 1 - o(1).
	\end{align*}
	The last equlity comes from the trivial bound of $\chi^2$ distribution after choosing $\eta/\tau$ sufficiently large such that
	\begin{equation*}
	\frac{ \bbP\bigl((\eta/\tau)^2 \le \chi^2_d \le (\eta/\tau - \tau \|t\|)^2 \bigr) }
	{  \bbP\bigl( \chi_d^2 \le (\eta/\tau)^2 \bigr)} \le 2 f_d\bigl( (\eta/\tau - \tau M) \bigr) \tau M \le C \tau M,
	\end{equation*}
where $f_d$ is the density function of $\chi_d$ and 
$C =2 \sup_x f_d(x)$.
\end{proof}

\begin{proof}[{Proof of Proposition~\ref{prop:init-error}}]
	First of all, by law of total expectation,
	\begin{align*}
	\Prob(\ell(\sigma,\hsigma^0)> \gamma) = \Expect^{\{\alpha_i, z_i\}_{i=1}^n} \left[\Prob \left(\ell(\sigma,\hsigma^0) > \gamma | \{\alpha_i, z_i\}_{i=1}^n\right) \right].
	\end{align*}
Given $\{\alpha_i, z_i\}_{i=1}^n$, the probability matrix $P$ is deterministic. 
Let $\mu_i$ be the mean value of $z_i$, that is, $\mu_i=\mu$ for $i=1,\cdots,n_1$ and $\mu_i=-\mu$ for $i=n_1+1,\cdots,n$. Let $\xi_{ij}=\Expect [e^{z_i^\top H z_j}]$ for $i\neq j$ and 
$\xi_{ii}=\Expect [e^{z_1^\top H z_2}]$. 
Define 
	\begin{align}
	B_{ij}= e^{\alpha_i+\alpha_j}\xi_{ij}. \label{eq:def-B}
	\end{align}
We further denote $\xi_{+}=\Expect[e^{z_1^\top H z_2}]$ 
and 
$\xi_{-}=\Expect[e^{z_1^\top H z_{n_1+1}}]=\Expect[e^{-z_1^\top H z_2}]$, 
then $B_{ij}=e^{\alpha_i+\alpha_j}\xi_{+}$ if $\sigma_i=\sigma_j$ and $B_{ij}=e^{\alpha_i+\alpha_j}\xi_{-}$ otherwise. 
Note that $B$ is a matrix of rank $2$, 
and we will show the proximity of $B$ and $\widehat{P}$ on a high-probability event.

\medskip
   
\paragraph{Step 1: Finding high probability event.}
Define $\mathbb{D}=\{(\omega_1,\cdots,\omega_n): (1/n_u)\sum_{\{i:\sigma_i=u\}}e^{4\omega_i} \le \overline{L}^4 \mathrm{\ for\ } u=1,2\}$. By \eqref{eq:alpha3} of Assumption \ref{assump:alpha},
    \begin{align}
	\Prob((\omega_1,\cdots,\omega_n)\in \mathbb{D}^c)\le 2n^{-(1+C_1)} \le n^{-(1+C_1/2)}. \label{eq:Dc_bound}
    \end{align}
\noindent
Let $\eta=\tau\sqrt{12 \log n}$, then by Assumption \ref{assump:tau}, $\eta\le \sqrt{12}c$. 
Define 
\begin{equation*}
\mathbb{B}_\eta=\{(z_1,\cdots,z_n):\|z_i-\mu_i\|_2\le \eta,1\le i\le n\}.	
\end{equation*}
Since $\frac{\eta^2}{\tau^2} > d+ 2\sqrt{d \frac{\eta^2}{4\tau^2}}+ \frac{\eta^2}{2\tau^2}$ when $n$ is large, by Lemma 1 of \cite{massart2000chisquare}, 
\begin{align*}
	\Prob(\|z_i-\mu_i\|_2 > \eta) & = \Prob\left(\frac{1}{\tau^2}\|z_i-\mu_i\|_2^2 > \frac{\eta^2}{\tau^2}\right) \\
	& < \Prob\left(\frac{1}{\tau^2}\|z_i-\mu_i\|_2^2-d > 2\sqrt{d \frac{\eta^2}{4\tau^2}}+ \frac{\eta^2}{2\tau^2} \right) 
	\le \exp\left\{-\frac{\eta^2}{4\tau^2}  \right\}.
\end{align*}
Therefore,
\begin{align}
\Prob(\mathbb{B}_\eta^c) \le n \exp\left\{ -\frac{\eta^2}{4\tau^2}\right\}
=n^{-2}. 
\label{eq:B_bound}
\end{align}
Assume $(z_1,\cdots,z_n)\in \mathbb{B}_\eta$, then $z_i^\top H z_j\le \mu_i^\top H\mu_j+ \eta \|H\mu_i\|_2+\eta \|H\mu_j\|_2+\eta^2 \|H\|_2\le \mu_i^\top H\mu_j+ \sqrt{12} c \|H\mu_i\|_2+\sqrt{12} c \|H\mu_j\|_2+12 c^2 \|H\|_2$ which is a constant. Hence there is a positive constant $\overline{\xi}$ such that
\begin{equation}
	\label{eq:xibar}
	e^{z_i^\top H z_j} \leq \overline{\xi}\quad \mbox{on $\mathbb{B}_\eta$}.
\end{equation}
	
Let $f_{ij}= \bigl(e^{z_i^\top H z_j} - \xi_{ij} \bigr)^2$, and define the set
\begin{equation*}
\mathbb{C}_r= \Big\{(z_1,\cdots,z_n): \sum_{1\le i\neq j\le n} f_{ij}\le 4 r^2 n (n-1)/(\log n)^{1-\omegarate}
\Big\}.	
\end{equation*}
for any small constant $\epsilon_1 \in (0,0.01)$ and some fixed constant $r>0$. 
We will specify the choice of $r$ later. 
Since $\xi_{ij}$, $\eta$ and $\|H\|_2$ are all constants, by \eqref{eq:xibar},
$f_{ij}$ has a uniform constant upper bound for all $1\le i\neq j\le n$ on $\mathbb{B}_\eta$, which we denote by $\overline{f}$. 
Write $\Phi_\eta^{+}$ as the measure of $z_i$ conditioned on $\| z_i - \mu\|_2 \le \eta$ for $i\in [n_1]$, and $\Phi_\eta^{-}$ for $n_1+1 \le i \le n$. The conditional distribution of $\{z_i\}_{1\le i \le n}$ on $\mathbb{B}_\eta$ is  
	\[
	\underbrace{\Phi_\eta^{+} \times \cdots \times \Phi_\eta^{+}}_{n_1} \times 
	\underbrace{\Phi_\eta^{-} \times \cdots \times \Phi_\eta^{-}}_{n_2},
	\]
where $\times $ denotes the product measure.
In particular, $z_i$'s are still mutually independent conditioned on $\mathbb{B}_\eta$. 
Hence, for any particular $i\in [n]$, $(f_{ij})_{j\neq i}$ are independent, and follow one of two distributions, depending on whether node $j$ is in the same community as node $i$. 
Thus we define
	\begin{align*}
	f_{i+}= & \wt{\Expect}^{z_j} ( f_{ij} |z_i) \mathrm{\ for\ any\ } 1\le j\le n_1, j\neq i, \\
	f_{i-}= & \wt{\Expect}^{z_j} ( f_{ij} |z_i) \mathrm{\ for\ any\ } n_1+1\le j\le n, j\neq i, \\
	f_{++} = & \wt{\Expect}^{z_i} (f_{i+}), \mathrm{\ for\ any\ } 1\le i\le n_1, \\
	f_{-+} = & \wt{\Expect}^{z_i} (f_{i+}), \mathrm{\ for\ any\ } n_1+1\le i\le n, \\
	f_{+-} = & \wt{\Expect}^{z_i} (f_{i-}), \mathrm{\ for\ any\ } 1\le i\le n_1, \\
	f_{--} = & \wt{\Expect}^{z_i} (f_{i-}), \mathrm{\ for\ any\ } n_1+1\le i\le n,
	\end{align*}
where $\wt{\Expect}^{z_j} (\cdot | z_i)$ in the first two equations denotes expectation with respect to the distribution of $z_j$ conditional on $z_i$ and $\|z_j-\mu_j\|_2\le \eta$, and $\wt{\Expect}^{z_i} (\cdot)$ in the last four equalities means expectation with respect to the distribution of $z_i$ conditional on $\|z_i-\mu_i\|_2\le \eta$. By Bernstein's inequality, we obtain
	\begin{align*}
	& \wt{\Prob}\left(\sum_{j\neq i}^n f_{ij} - \sum_{j\neq i}^n \wt{\Expect}^{z_j}(f_{ij}|z_i) > r^2 \frac{n-1}{(\log n)^{1-\omegarate}} \bigg| z_i \right) \\
	&~~~ \le \exp\left\{-\frac{r^4 (n-1)^2/(\log n)^{2(1-\omegarate)} } {2\sum_{j\neq i}\wt{\mathrm{Var}}^{z_j} (f_{ij}|z_i)+\frac{2}{3}\overline{f} r^2 (n-1) /(\log n)^{1-\omegarate}} \right\},
	\end{align*}
where $\wt{\Prob}$ and $\wt{\mathrm{Var}}^{z_j}(\cdot)$ are taken over the distribution of $z_j$ conditional on $\|z_j-\mu_j\|_2\le \eta$. By direct calculation we have
	\begin{align*}
	\wt{\mathrm{Var}}^{z_j}(f_{ij}|z_i) = \wt{M}_{ij}^{(4)}-4\xi_{ij} \wt{M}_{ij}^{(3)} +4\xi_{ij}^2 \wt{M}_{ij}^{(2)} + 4\xi_{ij} \wt{M}_{ij}^{(1)} \wt{M}_{ij}^{(2)}-(\wt{M}_{ij}^{(2)})^2 - 4\xi_{ij}^2 (\wt{M}_{ij}^{(1)})^2,
	\end{align*}
	where $\wt{M}_{ij}^{(l)}=\wt{\Expect}^{z_j}(e^{l z_i^\top H z_j}|z_i)$. Let $\zeta_{ij}=z_i^\top H \mu_j$ and $\iota_{i}=z_i^\top H^2 z_i$. Since $\|H z_i\|_2$ is upper bounded by a constant, by Lemma \ref{lem:limit-mgf}, $\wt{M}_{ij}^{(l)}=\exp\left(l \zeta_{ij}+\frac{\tau^2}{2} l^2 \iota_{i}\right)(1-o(1))$. Further calculation leads to
	\begin{align*}
	& \wt{\mathrm{Var}}^{z_j}(f_{ij}|z_i) \\
	= & (e^{\tau^2 \iota_{i}}-1) e^{2\zeta_{ij}+\tau^2 \iota_{i}}\left[e^{2\zeta_{ij}+3\tau^2 \iota_{i}}(e^{\tau^2 \iota_{i}}+1)(e^{2\tau^2 \iota_{i}}+1)- 4\xi_{ij}e^{\zeta_{ij}+\frac{3}{2}\tau^2 \iota_{i}}(e^{\tau^2 \iota_{i}}+1)+4\xi_{ij}^2\right](1+o(1)),
	\end{align*}
	which is upper bounded by $c_1 \tau^2$ with some constant $c_1>0$, since $\xi_{ij}$, $\zeta_{ij}$ and $\iota_i$ are upper bounded by constants. By Assumption \ref{assump:tau}, we have $2\sum_{j\neq i}\wt{\mathrm{Var}}^{z_j} (f_{ij}|z_i)\le 2c^2 c_1 (n-1)/\log n\le \frac{1}{3}\overline{f}r^2 (n-1)/(\log n)^{1-\omegarate}$ for large $n$. Consequently,
	\begin{align}
	\wt{\Prob}\left(\sum_{j\neq i}^n f_{ij} -\sum_{j\neq i}^n \wt{\Expect}^{z_j} (f_{ij}|z_i) > r^2 \frac{n-1}{(\log n)^{1-\omegarate}} \bigg| z_i \right) & \le \exp\left\{-\frac{r^2 (n-1)} {\overline{f} (\log n)^{1-\omegarate}} \right\} \le n^{-(2+C_2)} \label{eq:bernstein1}
	\end{align}
	for some constant $C_2>0$.
	
Recall that 
\begin{equation*}
\sum_{j\neq i}^n \wt{\Expect}^{z_j} (f_{ij}|z_i)
= 
\begin{cases}
(n_1-1)f_{i+} + n_2 f_{i-}, & 	1\le i\le n_1,\\
n_1 f_{i+} + (n_2-1) f_{i-}, &  n_1+1\le i\le n.
\end{cases}
\end{equation*}
Since $f_{i+}, f_{i-}\le \overline{f}$ on $\mathbb{B}_\eta$ for any $1\le i\le n$, by Bernstein's inequality again, we obtain 
	\begin{align*}
	\wt{\Prob}\left(\sum_{i=1}^{n_1} f_{i+}-n_1 f_{++}> r^2 \frac{n_1}{(\log n)^{1-\omegarate}}\right)\le &  \exp\left\{-\frac{r^4 n_1^2/ (\log n)^{2(1-\omegarate)}} {2n_1 \wt{\mathrm{Var}}^{z_1}(f_{1+})+\frac{2}{3} \overline{f} r^2 n_1/(\log n)^{1-\omegarate}} \right\}. 
	\end{align*}
We further bound the righthand side of the above display. 
By definition we have
	\begin{align*}
	f_{1+}= \wt{\Expect}^{z_j} ( f_{1j} |z_1) = \wt{M}_{1j}^{(2)} -2\xi_{+} \wt{M}_{1j}^{(1)}+\xi_{+}^2 \qquad \text{for } 1 \le j \le n_1, 
	\end{align*}
the variance of which is 
	\(
	\wt{\mathrm{Var}}^{z_1}(\wt{M}_{1j}^{(2)})+4\xi_{+}^2 \wt{\mathrm{Var}}^{z_1}(\wt{M}_{1j}^{(1)})- 4\xi_{+}\wt{\mathrm{Cov}}^{z_1}(\wt{M}_{1j}^{(2)}, \wt{M}_{1j}^{(1)}).
	\)
Since $z_1$ is bounded by constants and $\tau^2\to 0$, 
we can find a constant $c_1'>0$ such that 
$1 \le e^{4\tau^2\iota_{1}}\le 1+c_1'\tau^2$. 
Then we get
\begin{align*}
\wt{\mathrm{Var}}^{z_1}(\wt{M}_{1j}^{(2)}) 
& = \left[\wt{\Expect}^{z_1} e^{4\zeta_{1j} + 4\tau^2 \iota_{1}} - (\wt{\Expect}^{z_1} e^{2\zeta_{1j} + 2\tau^2 \iota_{1}})^2\right](1+o(1)) 
\\
& \le 
2\left[(1+c_1'\tau^2) \wt{\Expect}^{z_1} e^{4\zeta_{1j}} 
-  (\wt{\Expect}^{z_1} e^{2\zeta_{1j}})^2\right] \\
& = 2\left[ (1+c_1'\tau^2) e^{4\mu_1^\top H \mu_j+8\tau^2 \mu_j^\top H^2 \mu_j} - e^{4\mu_1^\top H \mu_j+4\tau^2 \mu_j^\top H^2 \mu_j}\right](1+o(1)) \\
& \leq 4e^{4\mu_1^\top H \mu_j+4\tau^2 \mu_j^\top H^2 \mu_j} \left[e^{4\tau^2 \mu_j^\top H^2 \mu_j}-1 +\tau^2 c_1' e^{4\tau^2 \mu_j^\top H^2 \mu_j} 
\right] \\
& \le c_2' \tau^2,
\end{align*}
for some constant $c_2'>0$. 
The last inequality is again due to $\tau^2\to 0$. 
We can use similar argument to get $\wt{\mathrm{Var}}^{z_1}(\wt{M}_{1j}^{(1)})\le c_3'\tau^2$ and $\wt{\mathrm{Cov}}^{z_1}(\wt{M}_{1j}^{(2)}, \wt{M}_{1j}^{(1)})\le c_4'\tau^2$. Therefore, we have $\wt{\mathrm{Var}}(f_{1+})\le c_2 \tau^2\le c^2 c_2/\log n\le \frac{1}{6}\overline{f}r^2/(\log n)^{1-\omegarate}$, where $c_2>0$ is a constant. This implies
	\begin{align}
	\wt{\Prob}\left(\sum_{i=1}^{n_1} f_{i+}-n_1 f_{++}> r^2 \frac{n_1}{(\log n)^{1-\omegarate}}\right)\le & \exp\left\{-\frac{r^2 n_1} {\overline{f}(\log n)^{1-\omegarate}}\right\} \le n^{-(3+C_{3})} \label{eq:bernstein21}
	\end{align}	
	for some constant $C_{3}>0$.
Similarly, we also obtain 
	\begin{align}
	\wt{\Prob}\left(\sum_{i=1}^{n_1} f_{i-}-n_1 f_{+-}> r^2 \frac{n_1}{(\log n)^{1-\omegarate}}\right) \le n^{-(3+C_{3})} \label{eq:bernstein22} \\
	\wt{\Prob}\left(\sum_{i=n_1+1}^{n} f_{i+}-n_2 f_{-+}> r^2 \frac{n_2}{(\log n)^{1-\omegarate}}\right)\le n^{-(3+C_{3})} \label{eq:bernstein23} \\
	\wt{\Prob}\left(\sum_{i=n_1+1}^{n} f_{i-}-n_2 f_{--}> r^2 \frac{n_2}{(\log n)^{1-\omegarate}}\right)\le n^{-(3+C_{3})}. \label{eq:bernstein24}
	\end{align}

Next we bound $f_{+-}, f_{+-}, f_{-+}$, and $f_{--}$. Since $z_1$ is bounded by constants and $\tau^2\to 0$, we can find constants $c_1''>0$ such that $e^{2\tau^2 \iota_{1}}\le 1+c_1''\tau^2$.
Then 
\begin{align}
f_{++} & =\wt{\Expect}^{z_1} [\wt{M}_{12}^{(2)} -2\xi_{+} \wt{M}_{1j}^{(1)}+\xi_{+}^2] 
\nonumber \\
& = 
\left[\wt{\Expect}^{z_1} [e^{2\zeta_{12}+ 2\tau^2 \iota_{1}}] 
-2\xi_{+} \wt{\Expect}^{z_1} [e^{\zeta_{12}+\frac{\tau^2}{2} \iota_{1}}] 
+ \xi_{+}^2 \right]
(1+o(1)) 
\nonumber \\
& \le 2 \left[(1+c_1''\tau^2) \wt{\Expect}^{z_1} [e^{2\zeta_{12}} ]
- 2\xi_{+}  \wt{\Expect}^{z_1} [e^{\zeta_{12}}]
+ \xi_{+}^2\right] 
\nonumber \\
& = 2\left[(1+c_1''\tau^2) e^{2\mu^\top H \mu+2\tau^2 \mu^\top H^2 \mu} 
-2\xi_{+} e^{\mu^\top H \mu+\frac{\tau^2}{2}\mu^\top H^2 \mu}
+ \xi_{+}^2 \right](1+o(1)) 
\label{eq:f++-bound-1}
\\
& \le 
4\left[(1+c_1''\tau^2) e^{2\mu^\top H \mu+2\tau^2 \mu^\top H^2 \mu} 
-2\xi_{+} e^{\mu^\top H \mu+\frac{\tau^2}{2}\mu^\top H^2 \mu}
+ \xi_{+}^2 \right]. 
\label{eq:f++-bound}
\end{align}
Here equality \eqref{eq:f++-bound-1} is due to Lemma \ref{lem:limit-mgf}.
By the definition of $\xi_{+}$, 
we have
\begin{align*}
\xi_{+} & = 
\Expect^{z_1}\left[\Expect^{z_2} \left[e^{z_1^\top H z_2}|z_1\right]\right] 
= \Expect^{z_1}\left[ 
e^{z_1^\top H \mu+\frac{\tau^2}{2} z_1^\top H^2 z_1}
\right].
\end{align*}
Let $z_1=\mu+\tau y_1$. 
Direct calculation leads to
\begin{align*}
\xi_{+} & = 
e^{\mu^\top H \mu +\frac{\tau^2}{2} \mu^\top H^2 \mu} 
\,\Expect\left[ e^{\frac{\tau^4}{2}y_1^\top H^2 y_1 + \tau \mu^\top H (I+ \tau^2 H) y_1}\right] \\
& = 
\left[\det(I-\tau^4 H^2)\right]^{-\frac{1}{2}} \exp\left\{\mu^\top H \mu 
+\frac{\tau^2}{2} \left[\mu^\top H^2 \mu + \mu^\top H (I+\tau^2 H)(I-\tau^4 H^2)^{-1}(I+\tau^2 H) H\mu\right]\right\}.
\end{align*}
By Taylor expansion, we have $\det(I-\tau^4 H^2)=1-\tau^4 \mathrm{Tr}(H^2)+o(\tau^4)$. 
Further note that $H^2$ is p.s.d., and so 
$1  \le \left[\det(I-\tau^4 H^2)\right]^{-\frac{1}{2}}\le 1+c_2''\tau^2$ for some constant $c_2''>0$. 
In addition, as $\tau\to 0$,
\begin{equation*}
\mu^\top H^2 \mu + \mu^\top H (I+\tau^2 H)(I-\tau^4 H^2)^{-1}(I+\tau^2 H) H\mu\to 2\mu^\top H^2 \mu.	
\end{equation*}
Therefore, we have 
\begin{align*}
	1 \le \exp\left\{ \frac{\tau^2}{2} \left[\mu^\top H^2 \mu + \mu^\top H (I+\tau^2 H)(I-\tau^4 H^2)^{-1}(I+\tau^2 H) H\mu\right]\right\} 
	\le (1+c_3''\tau^2)  
\end{align*}
for some constant $c_3''>0$. 
Therefore, we can find a constant $c_4''>0$ such that 
\begin{equation*}
e^{\mu^\top H \mu} \le \xi_{+}\le (1+c_4''\tau^2)e^{\mu^\top H \mu}. 	
\end{equation*}
Plugging this into \eqref{eq:f++-bound}, we get
\begin{align*}
f_{++} 
&
\le 4\left[
(1+c_1''\tau^2) e^{2\mu^\top H \mu+2\tau^2 \mu^\top H^2 \mu} 
-2 e^{\mu^\top H \mu}\,e^{\mu^\top H \mu+\frac{\tau^2}{2}\mu^\top H^2 \mu}
+ (1+c_4'' \tau^2)^2 e^{2\mu^\top H \mu} 
\right] 
\\ 
& \le 4 e^{2\mu^\top H \mu} 
\left(
e^{2\tau^2 \mu^\top H^2 \mu} 
-2e^{\frac{\tau^2}{2}\mu^\top H^2 \mu}
+1
+c_5'' \tau^2  
\right) \\
& \le c_3\tau^2,
\end{align*}
where $c_5''>0$, $c_3>0$ are constants. 
The last two inequalities are both due to $\tau^2\to 0$. We can bound $f_{+-}, f_{-+}, f_{--}$ in similar ways. Assumption \ref{assump:tau} then ensures that for sufficiently large values of $n$, 
\begin{equation}
	\label{eq:fpm-bounds}
\max\{f_{++}, f_{+-}, f_{-+}, f_{--}\} \le 2r^2/(\log n)^{1-\omegarate}. 	
\end{equation}
In view of the decomposition
	\begin{align*}
	\sum_{1\le i\neq j\le n} f_{ij} = & \sum_{i=1}^{n_1} \left[\sum_{j\neq i}f_{ij} -(n_1-1) f_{i+}-n_2 f_{i-}\right] + \sum_{i=n_1+1}^{n} \left[ \sum_{j\neq i} f_{ij} -n_1 f_{i+}-(n_2-1) f_{i-}\right] \\
	& + (n_1-1)\sum_{i=1}^{n_1} (f_{i+}-f_{++}) + n_2\sum_{i=1}^{n_1} (f_{i-}-f_{+-})\\
	& + n_1 \sum_{i=n_1+1}^{n} (f_{i+}-f_{-+}) 
	 + (n_2-1)\sum_{i=n_1+1}^{n} (f_{i-}-f_{--}) 
	\\
	& + n_1(n_1-1)f_{++} + n_1 n_2 f_{+-} + n_1 n_2 f_{-+} +n_2 (n_2-1) f_{--}
	\end{align*}
and that \eqref{eq:fpm-bounds} implies
\begin{equation*}
n_1(n_1-1)f_{++} + n_1 n_2 f_{+-} + n_1 n_2 f_{-+} +n_2 (n_2-1) f_{--}\le 
\frac{2r^2 n(n-1)}{(\log n)^{1-\omegarate}},	
\end{equation*}
 we obtain
	\begin{align}
	& \Prob\left((z_1,\cdots,z_n)\in \mathbb{C}_{r}^c \bigg|(z_1,\cdots,z_n)\in \mathbb{B}_\eta \right) = \wt{\Prob}\left(\sum_{1\le i\neq j\le n} f_{ij}> 4r^2 n(n-1)/(\log n)^{1-\omegarate} \right) \nonumber \\
	& \le \wt{\Prob} \left(\sum_{1\le i\neq j\le n} f_{ij}- n_1(n_1-1)f_{++} - n_1 n_2 f_{+-} - n_1 n_2 f_{-+} - n_2 (n_2-1) f_{--}> 2r^2 n(n-1)/(\log n)^{1-\omegarate} \right) \nonumber \\
	& \le \sum_{i=1}^{n_1} \wt{\Expect}^{z_i}\left[\wt{\Prob}\left(\sum_{j\neq i} f_{ij}-(n_1-1) f_{i+}-n_2 f_{i-}> r^2 (n-1)/(\log n)^{1-\omegarate} \bigg| z_i\right)\right] \nonumber \\
	&~~~ + \sum_{i=n_1+1}^{n} \wt{\Expect}^{z_i}\left[ \wt{\Prob}\left(\sum_{j\neq i} f_{ij}-n_1 f_{i+}-(n_2-1) f_{i-}> r^2 (n-1)/(\log n)^{1-\omegarate} \bigg| z_i\right) \right] \nonumber \\
	&~~~ + \wt{\Prob}\left(\sum_{i=1}^{n_1} f_{i+}-n_1 f_{++}>r^2 n_1 /(\log n)^{1-\omegarate} \right) + \wt{\Prob}\left(\sum_{i=1}^{n_1} f_{i-}-n_1 f_{+-}> r^2 n_1 /(\log n)^{1-\omegarate} \right) \nonumber \\ 
	&~~~ + \wt{\Prob}\left(\sum_{i=n_1+1}^{n} f_{i+}-n_2 f_{-+}>r^2 n_2 /(\log n)^{1-\omegarate} \right) + \wt{\Prob} \left(\sum_{i=n_1+1}^{n} f_{i-}-n_2 f_{--}>r^2 n_2 /(\log n)^{1-\omegarate} \right) \nonumber \\	
	& \le n^{-(1+C_2)}+ 4 n^{-(3+C_3)}\le n^{-(1+C_2)}+ n^{-(1+C_3)}. \nonumber 
	\end{align}
The penultimate inequality is due to \eqref{eq:bernstein1}--\eqref{eq:bernstein24}.
We then have for large $n$
\begin{equation}
\begin{aligned}
	\Prob\bigl((z_1,\cdots,z_n)\in \mathbb{B}_\eta \cap \mathbb{C}_{r}^c\bigr) & \le \Prob\bigl((z_1,\cdots,z_n)\in \mathbb{C}_{r}^c|(z_1,\cdots,z_n)\in \mathbb{B}_\eta \bigr) \\
& \le  n^{-(1+C_2)}+ n^{-(1+C_3)}. 
\end{aligned}	
\label{eq:B_Cc_bound}
\end{equation}

\medskip

\paragraph{Step 2: Bounding initialization error.}	
The next part of the proof is in line with the proofs of Lemma 1 and Corollary 2 in \cite{gao2018community}.  
Let $B_i$ denote the $i$th row of $B$, which is defined by \eqref{eq:def-B}, and define $\bar{B}_i=\|B_i\|_1^{-1} B_i$. 
Throughout this part, we conduct all the calculation on the intersection of the events $\{ (z_1,\cdots, z_n)\in \mathbb{B}_\eta \cap \mathbb{C}_{r}^c \}$ and $\{ (\omega_1,\cdots,\omega_n)\in \mathbb{D} \}$. 
	
\subparagraph
{Step 2.1: Establishing the separation condition for the rows of $\bar{B}$.} Note that $\bar{B}_i=\bar{B}_j$ when $\sigma_i=\sigma_j$, we only need to lower bound $\|\bar{B}_1-\bar{B}_n\|_1$. Let $L_u=\sum_{\sigma_i=u} e^{\omega_i}$ for $u=1,2$. When $L_1 \xi_{+}+L_2\xi_{-}\le L_1 \xi_{-}+L_2\xi_{+}$, we have 
	\begin{align*}
	\|\bar{B}_1-\bar{B}_n\|_1 \ge & \sum_{i=1}^{n_1} |\bar{B}_{1i}-\bar{B}_{ni}| = \sum_{i=1}^{n_1} \left|\frac{e^{\omega_i}\xi_{+}}{L_1 \xi_{+}+L_2 \xi_{-}}- \frac{e^{\omega_i}\xi_{-}}{L_1 \xi_{-}+L_2 \xi_{+}}\right| \\
	= & \frac{1}{L_1 \xi_{-}+L_2 \xi_{+}}\sum_{i=1}^{n_1} e^{\omega_i} \left| \frac{L_1 \xi_{-}+L_2 \xi_{+}}{L_1 \xi_{+}+L_2 \xi_{-}}\xi_{+}-\xi_{-} \right| \\
	\ge & \frac{L_1(\xi_{+}-\xi_{-})}{L_1 \xi_{+}+L_2 \xi_{-}}.
	\end{align*}
	Since $L_u\le (n_u \sum_{i=\sigma_u}e^{2\omega_i})^{1/2}\le \left(n_u\sqrt{n_u \sum_{i=\sigma_u}e^{4\omega_i}}\right)^{1/2}\le n_u \overline{L}$ for $u=1,2$ and $L_1\ge n_1 e^{-\underline{\omega}}\ge \frac{1}{3}n e^{-\underline{\omega}}$, we obtain
	\begin{align*}
	\|\bar{B}_1-\bar{B}_n\|_1 \ge \frac{\frac{1}{3}n e^{-\underline{\omega}}(\xi_{+}-\xi_{-})}{n \overline{L} \xi_{+}}=\frac{\xi_{+}-\xi_{-}}{3 e^{\underline{\omega}} \overline{L} \xi_{+}}.
	\end{align*}
	A similar argument holds when $L_1 \xi_{+}+L_2\xi_{-}> L_1 \xi_{-}+L_2\xi_{+}$ by using $\|\bar{B}_1-\bar{B}_n\|_1 \ge \sum_{i=n_1+1}^{n} |\bar{B}_{1i}-\bar{B}_{ni}|$ at the beginning of the sequence of inequalities. Therefore, the separation condition holds for $\bar{B}$
	\begin{align*}
	\min_{\sigma_i\neq \sigma_j} \|\bar{B}_i-\bar{B}_j\|_1 \ge \frac{\xi_{+}-\xi_{-}}{3 e^{\underline{\omega}} \overline{L} \xi_{+}}.
	\end{align*}
    
\subparagraph{Step 2.2: Bounding $\sum_{\hsigma^0_i\neq \sigma_i} e^{\omega_i}$.} Let $\widehat{v}_1$ and $\widehat{v}_2$ be the centroids from the $k$-median step of Algorithm~\ref{alg:init}. 
    Recall $\indexset_0 := \{i: \widehat{\sigma}_i = 0  \}$ from Algorithm~\ref{alg:init}.
    Fill matrix $\widehat{V}\in \mathbb{R}^{n\times n}$ with  $\widehat{V}_i=\widehat{v}_{\hsigma^0_i}$ being its $i$th row, if $i \in \indexset_0^c$ and $\widehat{V}_i=(0,\cdots,0)$ if $i\in \indexset_0$. Let $\indexset=\{i\in \indexset_0^c: \|\widehat{V}_i-\bar{B}_i\|_1\ge \frac{\xi_{+}-\xi_{-}}{6 e^{\underline{\omega}} \overline{L} \xi_{+}}\}$. 
    As in Lemma 5 of \cite{gao2018community} we define
    \begin{align*}
    \calC_u & =\{i\in \indexset_0^c: \sigma_i=u, \|\widehat{V}_i-\bar{B}_i\|_1< \frac{\xi_{+}-\xi_{-}}{6 e^{\underline{\omega}} \overline{L} \xi_{+}}\}, \\
    R_1 & =\{u\in\{1,2\}: \calC_u=\emptyset\}, \\
    R_2 & =\{u\in\{1,2\}: \calC_u\neq \emptyset, \forall i,j\in \calC_u, \hsigma^0_i=\hsigma_j^0\}, \\
    R_3 & =\{u\in\{1,2\}: \calC_u\neq \emptyset, \exists i,j\in \calC_u, \mathrm{s.t.\ } i\neq j, \hsigma^0_i\neq \hsigma_j^0\}.
    \end{align*}
The counting argument in Lemma 5 of \cite{gao2018community} implies $|R_3|\leq |R_1|$. 
Therefore, 
    \begin{align*}
    \sum_{i \in \cup_{u\in R_3} \calC_u} e^{\omega_i} \le |R_3| n \overline{L} \le |R_1| n \overline{L} \le 3 e^{\underline{\omega}} \overline{L} \sum_{i\in \indexset} e^{\omega_i}.
    \end{align*}
    Here the last inequality holds because $\sum_{i\in \indexset} e^{\omega_i}\ge \sum_{u\in R_1}\sum_{i\in C_u^c} e^{\omega_i}=\sum_{u\in R_1}\sum_{\sigma_i=u} e^{\omega_i}\ge |R_1|\frac{1}{3}n e^{-\underline{\omega}}$. Hence we have obtained
    \begin{align}
    \sum_{\hsigma_i^0\neq \sigma_i} e^{\omega_i} \le \sum_{i \in \indexset_0} e^{\omega_i} + \sum_{i \in \indexset} e^{\omega_i} + \sum_{i \in \cup_{u\in R_3} \calC_u} e^{\omega_i}\le \sum_{i \in \indexset_0} e^{\omega_i} + (1+3 e^{\underline{\omega}} \bar{L}) \sum_{i \in \indexset} e^{\omega_i}. \label{eq:geometry-cluster}
    \end{align}
    
\subparagraph{Step 2.3: Bounding $\sum_{i \in \indexset_0} e^{\omega_i}$ and $\sum_{i \in \indexset} e^{\omega_i}$.} By definition of $\widehat{P}$ from Algorithm~\ref{alg:init}, we have 
    \begin{align*}
    \sum_{i=1}^n \|\widehat{P}_i\|_1 \|\widehat{V}_i-\tilde{P}_i\|_1\le (1+\kmbound)\sum_{i=1}^n \|\widehat{P}_i\|_1 \|\bar{B}_i-\tilde{P}_i\|_1.
    \end{align*}	
    Then a bound for $\sum_{i\in \indexset}\|\widehat{P}_i\|_1$ can be established by
    \begin{align*}
    \sum_{i\in \indexset}\|\widehat{P}_i\|_1 & \le \frac{6 e^{\underline{\omega}} \overline{L} \xi_{+}}{\xi_{+}-\xi_{-}} \sum_{i\in \indexset}  \|\widehat{P}_i\|_1 \|\widehat{V}_i-\bar{B}_i\|_1 \\
    & \le \frac{6 e^{\underline{\omega}} \overline{L} \xi_{+}}{\xi_{+}-\xi_{-}} \sum_{i\in \indexset}  \left( \|\widehat{P}_i\|_1  \|\widehat{V}_i-\tilde{P}_i\|_1 + \|\widehat{P}_i\|_1  \|\tilde{P}_i-\bar{B}_i\|_1\right) \\
    & \le (2+\kmbound)\frac{6 e^{\underline{\omega}} \overline{L} \xi_{+}}{\xi_{+}-\xi_{-}} \sum_{i=1}^n \|\widehat{P}_i\|_1  \|\tilde{P}_i-\bar{B}_i\|_1 \\
    & \le (2+\kmbound)\frac{6 e^{\underline{\omega}} \overline{L} \xi_{+}}{\xi_{+}-\xi_{-}} \sum_{i=1}^n 2 \|\widehat{P}_i-B_i\|_1 \frac{\|\widehat{P}_i\|_1}{\|\widehat{P}_i\|_1 \vee \|B_i\|_1} \\
    & \le (2+\kmbound)\frac{12 e^{\underline{\omega}} \overline{L} \xi_{+}}{\xi_{+}-\xi_{-}} \sum_{i=1}^n \|\widehat{P}_i-B_i\|_1 \\
    & \le (2+\kmbound)\frac{12 e^{\underline{\omega}} \overline{L} \xi_{+}}{\xi_{+}-\xi_{-}} n \fnorm{\widehat{P}-B}.
    \end{align*}
    Since $\|B_i\|_1=e^{\alpha_i}\sum_{j=1}^n e^{\alpha_j}\xi_{ij}\ge e^{\omega_i} \frac{n}{3} e^{2\alphaavg-\underline{\omega}}\xi_{+}$, we can bound $\sum_{i\in \indexset} e^{\omega_i}$ by 
    \begin{align}
    \sum_{i\in \indexset} e^{\omega_i} & \le \frac{3}{ne^{2\alphaavg-\underline{\omega}} \xi_{+}} \sum_{i\in \indexset} \|B_i\|_1 \le \frac{3}{ne^{2\alphaavg- \underline{\omega}} \xi_{+}} \sum_{i\in \indexset} \left(\|\widehat{P}_i\|_1+ \|\widehat{P}_i-B_i\|_1 \right) \nonumber \\
    & \le \frac{3}{ne^{2\alphaavg-\underline{\omega}} \xi_{+}}\left[ (2+\kmbound)\frac{12 e^{\underline{\omega}} \overline{L} \xi_{+}}{\xi_{+}-\xi_{-}} n \fnorm{\widehat{P}-B} + n \fnorm{\widehat{P}-B} \right] \nonumber \\
    & = \frac{3}{e^{2\alphaavg-\underline{\omega}} \xi_{+}}\left[ (2+\kmbound)\frac{12 e^{\underline{\omega}} \overline{L} \xi_{+}}{\xi_{+}-\xi_{-}}+1 \right] \fnorm{\widehat{P}-B}. \label{eq:set-bound}
    \end{align} 
    We can also bound $\sum_{i\in \indexset_0} e^{\omega_i}$ by
    \begin{align}
    \sum_{i\in \indexset_0} e^{\omega_i} & \le \frac{3}{n e^{2\alphaavg-\underline{\omega}} \xi_{+}} \sum_{i\in \indexset_0} \|B_i\|_1 \le \frac{3}{n e^{2\alphaavg-\underline{\omega}} \xi_{+}} \sum_{i\in \indexset_0} \|\widehat{P}_i-B_i\|_1  \le \frac{3}{e^{2\alphaavg-\underline{\omega}} \xi_{+}} \fnorm{\widehat{P}-B}. \label{eq:set0-bound}
    \end{align}
    Combining \eqref{eq:geometry-cluster}, \eqref{eq:set-bound} and \eqref{eq:set0-bound}, we obtain
    \begin{align}
    \sum_{\{i:\hsigma^0_i\neq \sigma_i\}} e^{\omega_i} \le \sum_{i\in \indexset_0} e^{\omega_i} + (1+3 e^{\underline{\omega}} \bar{L}) \sum_{i\in \indexset} e^{\omega_i} \le C' e^{3\underline{\omega}} \overline{L}^2 e^{-2\alphaavg} \fnorm{\widehat{P}-B} \label{eq:theta-error}
    \end{align}
    for some constant $C'>0$.
    
\subparagraph{Step 2.4: Bounding $\|\widehat{P}-B\|_F$.} We follow the argument of Lemma 6 in \cite{gao2018community}. By definition of $\widehat{P}$, $\fnorm{\widehat{P}-A}^2 \le \fnorm{B-A}^2$. Then
    \begin{align*}
    \fnorm{\widehat{P}-B}^2 & =\fnorm{\widehat{P}-A}^2-\fnorm{B-A}^2-2\langle \widehat{P}-B,B-A\rangle \\
    & \le 2 |\langle \widehat{P}-B,B-A\rangle| 
    \le 2 \fnorm{\widehat{P}-B} \sup_{K: \fnorm{K}=1: \rank(K)\le 4} |\langle K, A-B\rangle| \\
    & \le \frac{1}{4}\fnorm{\widehat{P}-B}^2 + 4 \sup_{K: \fnorm{K}=1: \rank(K)\le 4} |\langle K, A-B\rangle|^2. 
    \end{align*}
    By rearranging terms we obtain
    \begin{align*}
    \fnorm{\widehat{P}-B}^2 & \le \frac{16}{3} \sup_{K: \fnorm{K}=1: \rank(K)\le 4} |\langle K, A-B\rangle|^2.
    \end{align*}
    Suppose $K$ has singular value decomposition $K=\sum_{l=1}^{4} \lambda_l u_l u_l^\top$, then
    \begin{align*}
    |\langle K, A-B\rangle|\le \sum_{l=1}^{4} |\lambda_l| |u_l^\top (A-B) u_l|\le \norm{A-B}_2 \sum_{l=1}^4 |\lambda_l| \le 2 \norm{A-B}_2.
    \end{align*}
    Therefore, we have
    \begin{align}
    \fnorm{\widehat{P}-B} \le \frac{8}{\sqrt{3}} \norm{A-B}_2. \label{eq:Phat-B-bound}
    \end{align}
    Define $Q_{ij}=e^{\alpha_i+\alpha_j+z_i^\top H z_j}$ for $1\le i\neq j\le n$ and $Q_{ii}=0$ for $1\le i\le n$. By the triangle inequality,
    \begin{align}
    \norm{A-B}_2 \le \norm{A-P}_2 + \norm{P-Q}_2 +\norm{Q-B}_2. \label{eq:triangle}
    \end{align}
    We bound the three terms on the right hand side separately. First by Example 4.1 in \cite{latala2018dimension}, for any $u \ge 1$ and $t>0$, we bound 
    \begin{equation}
    \bbP \bigl( \| A - P\|_2 > 2 e^{1/(2u)} \sqrt{b} +  C_4 e^{1/u}\sqrt{u\log n}  + t \bigr|P) < \exp\Bigl( -\frac{t^2}{C_4 } \Bigr) \label{eq:latala}
    \end{equation}
    with some constant $C_4>0$, where $b : = \max_{i} \sum_{j=1}^n P_{ij}$. Observe that $\sum_{j=1}^n P_{ij}= e^{\alpha_i}\sum_{j=1}^n e^{\alpha_j} e^{z_i^\top H z_j}\le e^{2\alphaavg+\omegabar}\sum_{j=1}^n e^{\omega_j} \overline{\xi}\le \overline{\xi} \overline{L} n e^{2\alphaavg+\omegabar}$ for all $i\in [n]$. 
    Take $t = \sqrt{C_4(1+C_4) \log n}$ in \eqref{eq:latala}, then conditional on $P$, with probability at least $1-n^{-(1+C_4)}$,
    \begin{align}
    \| A - P\|_2 \le C_1' \sqrt{\overline{L} n e^{2\alphaavg+\omegabar}} + C_2'\sqrt{\log n} \label{eq:K1_bound}
    \end{align}
    for constants $C_1'>0$ and $C_2'>0$.
    
	By definition, for $i\neq j$,
	\begin{align*}
	|P_{ij}-Q_{ij}|& = e^{2\alphaavg+\omega_i+\omega_j + z_i^\top H z_j} \frac{ e^{2\alphaavg+\omega_i+\omega_j + z_i^\top H z_j}}{1+ e^{2\alphaavg+ \omega_i+\omega_j + z_i^\top H z_j}} \le e^{4\alphaavg+2\omega_i+2\omega_j + 2z_i^\top H z_j} \le e^{4\alphaavg} e^{2\omega_i+2\omega_j} \overline{\xi}^2,
	\end{align*}
	and $P_{ii}-Q_{ii}=0$. Then we obtain
	\begin{align}
	\norm{P-Q}_2\le \fnorm{P-Q}\le \left(\sum_{i,j=1}^n e^{8\alphaavg} e^{4\omega_i+4\omega_j} \overline{\xi}^4\right)^{1/2} = e^{4\alphaavg}\sum_{i=1}^n e^{4\omega_i} \overline{\xi}^2 \le \overline{\xi}^2 \overline{L}^4 n e^{4\alphaavg} . \label{eq:K2_bound}
	\end{align}
	
	By definition, $(Q_{ij}-B_{ij})^2 = e^{4\alphaavg+2\omega_i+ 2\omega_j} f_{ij}$ for $i\neq j$, and $(Q_{ii}-B_{ii})^2=e^{4\alphaavg+4\omega_i+2z_i^\top H z_i}$. By Cauchy-Schwarz inequality,
	\begin{align*}
	\sum_{1\le i\neq j\le n} e^{4\alphaavg+2\omega_i+ 2\omega_j} f_{ij} \le \left(\sum_{1\le i\neq j\le n} e^{8\alphaavg+4\omega_i+ 4\omega_j}\right)^{1/2} \left(\sum_{1\le i\neq j\le n} f_{ij}^2\right)^{1/2}.
	\end{align*}
	It is straightforward to obtain the bound $\left(\sum_{1\le i\neq j\le n} e^{8\alphaavg+4\omega_i+ 4\omega_j}\right)^{1/2}\le e^{4\alphaavg} \sum_{i=1}^n e^{4\omega_i}\le \overline{L}^4 n e^{4\alphaavg}$. Since $f_{ij}\le \overline{f}$, we have
	\begin{align*}
	\left(\sum_{1\le i\neq j\le n} f_{ij}^2\right)^{1/2} & \le \overline{f}^{1/2} \left(\sum_{1\le i\neq j\le n} f_{ij}\right)^{1/2} \\
	& \le \overline{f}^{1/2} \left(4r^2(n-1)n/(\log n)^{1-\omegarate} \right)^{1/2} \\
	& \le 2 r \overline{f}^{1/2} n/(\log n)^{\frac{1-\omegarate}{2}}.
	\end{align*}
	Hence we obtain
	\begin{align*}
	\sum_{1\le i\neq j\le n} (Q_{ij}-B_{ij})^2 \le 2 r \overline{f}^{1/2} \overline{L}^4 n^2 e^{4\alphaavg}/(\log n)^{\frac{1-\omegarate}{2}}.
	\end{align*}
	On the other hand,
	\begin{align*}
	\sum_{i=1}^n (Q_{ii}-B_{ii})^2 & = \xi_{+}^2 \sum_{i=1}^n e^{4\alphaavg+4\omega_i} \le \xi_{+}^2 \overline{L}^4 n e^{4\alphaavg} \le 2 r \overline{f}^{1/2} \overline{L}^4 n^2 e^{4\alphaavg}/(\log n)^{\frac{1-\omegarate}{2}}.
	\end{align*}
	Then we can bound $\norm{Q-B}_2$ by
	\begin{align} 
	\norm{Q-B}_2\le \fnorm{Q-B} \le &  2\sqrt{r} \overline{f}^{1/4}\overline{L}^2 n e^{2\alphaavg}/(\log n)^{\frac{1-\omegarate}{4}}. \label{eq:K3_bound}
	\end{align}

\subparagraph{Step 2.5: Bounding $\sum_{\hsigma^0_i\neq \sigma_i} e^{\omega_i}$.} Combining \eqref{eq:theta-error}, \eqref{eq:Phat-B-bound}, \eqref{eq:triangle}, \eqref{eq:K1_bound}, \eqref{eq:K2_bound} and \eqref{eq:K3_bound}, we obtain that conditional on $P$, with probability at least $1-n^{-(1+C_4)}$
	\begin{align*}
	\sum_{\{i:\hsigma^0_i\neq \sigma_i\}} e^{\omega_i} & \le \frac{8}{\sqrt{3}} C' e^{3\underline{\omega}} \overline{L}^2 e^{-2\alphaavg} \left[ C_1' \sqrt{\overline{L} n e^{2\alphaavg+\omegabar}} + C_2'\sqrt{\log n}+ \overline{\xi}^2 \overline{L}^4 n e^{4\alphaavg} + 2\sqrt{r} \overline{f}^{1/4}\overline{L}^2 n e^{2\alphaavg}/(\log n)^{\frac{1-\omegarate}{4}} \right] \\
	& \le n\left( C_1'' \frac{1}{\sqrt{n e^{2\alphaavg-\omegabar}}}+C_2''\frac{\sqrt{\log n}}{ne^{2\alphaavg}} + C_3'' e^{2\alphaavg}+ C_4''\sqrt{r} \frac{1}{(\log n)^{\frac{1-\omegarate}{4}}} \right)
	\end{align*}
	for constants $C_1'', C_2'', C_3'', C_4''>0$. 
    By \eqref{eq:alpha4} and \eqref{eq:alpha5} of Assumption \ref{assump:alpha}, we have $1/\sqrt{n e^{2\alphaavg-\omegabar}}\to 0$ and $(\log n)/\sqrt{ne^{2\alphaavg}} \to 0$. For any $\gamma>0$, we can then make $r$ small enough such that $\sum_{\{i:\hsigma^0_i\neq \sigma_i\}} e^{\omega_i} \le e^{-\underline{\omega}} \gamma n$. Note that when $\gamma$ is fixed, $r$ is can still be a constant bounded away from 0.
	
	At last, putting \eqref{eq:Dc_bound}, \eqref{eq:B_bound} and \eqref{eq:B_Cc_bound} together with the conclusion from the previous paragraph, we obtain
	\begin{align*}
	& \Prob\left(\sum_{\{i:\hsigma^0_i\neq \sigma_i\}} e^{\omega_i}> e^{-\underline{\omega}} \gamma n\right) \\
	& \le \Expect^{\{\alpha_i,z_i\}_{i=1}^n}\left[\Prob\left(\sum_{\{i: \hsigma^0_i\neq \sigma_i\}} e^{\omega_i}> e^{-\underline{\omega}} \gamma n \bigg|\{\alpha_i,z_i\}_{i=1}^n\right)\indc{(z_1,\cdots,z_n)\in \mathbb{B}_\eta\cap \mathbb{C}_r, (\omega_1,\cdots,\omega_n)\in \mathbb{D}} \right] \\
	& + \Prob\left((z_1,\cdots,z_n)\in \mathbb{B}_\eta\cap \mathbb{C}_r^c\right)  + \Prob\left((z_1,\cdots,z_n)\in \mathbb{B}_\eta^c \right) + \Prob\left( (\omega_1,\cdots, \omega_n)\in \mathbb{D}^c \right)\\
	& \le n^{-(1+C_4)} + n^{-(1+C_2)} + n^{-(1+C_3)}+ n^{-2} + n^{-(1+C_1/2)} \\
	& < n^{-(1+2C)}
	\end{align*}
	with $0<C<\min\{\frac{C_1}{4},\frac{C_2}{2},\frac{C_3}{2},\frac{C_4}{2}, \frac{1}{2}\}$.
	
	Note that $\sum_{\{i:\hsigma^0_i\neq \sigma_i\}} e^{\omega_i} \ge e^{-\underline{\omega}} n \ell(\sigma, \hsigma^0)$, we immediately get
	\begin{align*}
	\Prob\left(\ell(\sigma, \hsigma^0)>\gamma\right)\le \Prob\left(\sum_{\{i:\hsigma^0_i\neq \sigma_i\}} e^{\omega_i}>e^{-\underline{\omega}}  \gamma n\right) < n^{-(1+2C)}. 
	\end{align*}
This completes the proof.
\end{proof}

%


\section{Proof of Theorems \ref{thm:upper} and \ref{thm:lowbd}}
\label{sec:prooflowerupper}


\subsection{Combining the initial error and edge counting}

Let $\wh\sigma^{(-1,0)}$ be an $n$-dimensional vector one obtains after line $7$ of Algorithm \ref{alg:provable}.
The following Proposition~\ref{prop:init-error-edge-counting} gives an error bound for $\wh\sigma^{(-1,0)}$.

\begin{proposition}\label{prop:init-error-edge-counting}

	Suppose that Assumptions \ref{assump:alpha}, \ref{assump:tau} and \ref{assump:mu-H-mu} hold. Let $p(\alpha_1,z_1)$ and $q(\alpha_1,z_1)$ be quantities defined in \eqref{eqn:def-p-z0} and \eqref{eqn:def-q-z0} respectively, and
    \[
        I(\alpha_1,z_1)=-2\log\left(\sqrt{p(\alpha_1,z_1)q(\alpha_1,z_1)}+\sqrt{(1-p(\alpha_1,z_1))(1-q(\alpha_1,z_1))}\right).
    \]
    Assume $n_1, n_2 \in \bigl( (1- \delta_n)/2, (1+\delta_n)/2\bigr)$.
	For any $\epsilon>0$, define $\calB_\epsilon=\{z_1: \|z_1-\mu\|_2\le \sqrt{1-\frac{\epsilon}{4}} \rho\}$. Then there is an $n_\epsilon$ such that for all $n>n_\epsilon$,
	\begin{align}\label{eq:init-error-edge-counting}
	\Prob\left(\hsigma_{1}^{(-1,0)} \neq \sigma_1\right)\le \Expect_{\{\sigma_1=1\}}^{\alpha_1,z_1}\left[\indc{z_1\in \calB_\epsilon} \exp\left\{-\frac{n}{2}(1-\epsilon)I(\alpha_1,z_1)\right\}\right] + \exp\left\{-(1-\epsilon)\frac{\rho^2}{2\tau^2}\right\} + n^{-(1+C)}
	\end{align}
	for some constant $C>0$.
\end{proposition}

\begin{proof}
We start with some notation.
Let $J_{u}=\{i: \sigma_i=u, 2\le i\le n\}$, $n_u=|J_u|$,
$\widehat{J}_{u}=\{i: \hsigma_i^{(-1,0)}=u, 2\le i\le n\}$,
$m_u=|\widehat{J}_u|$ for $u\in \{1,2\}$, and
$J_{u_1 u_2}=\{i: \hsigma_i^{(-1,0)}=u_1,\sigma_i=u_2, 2\le i\le n\}$,
$m_{u_1 u_2}=|J_{u_1 u_2}|$ for $u_1,u_2\in\{1,2\}$.
For convenience, we suppress the superscript $(-1,0)$ from $\hsigma_i^{(-1,0)}$ in the rest of this proof.

Recall the definitions of $P_{ij}$ in \eqref{eq:model-1} and $p(\alpha,z)$ and $q(\alpha,z)$ in \eqref{eqn:def-p-z0} and \eqref{eqn:def-q-z0}. 
Define events
\begin{gather*}
    \bbC_1 := \Bigl\{ \max_{2 \le i \le n} \norm{z_i - \mu_i}_2 \le \eta \Bigr\},\\
    \bbD_1 := \left\{\sum_{\{i:\hsigma^{(-1,0)}_i\neq \sigma_i\}} e^{\omega_i}\le e^{-\underline{\omega}} \frac{\gamma}{2} (n-1)\right\}, \\
    \bbF_1 := \left\{ \bigg| \sum_{i \in J_1} P_{1i} - n_1 p(\alpha_1, z_1)\biggr|  \le  n_1 \epsilon' p(\alpha_1, z_1) \right\} \cap \left\{\biggl| \sum_{i \in J_2} P_{1i} - n_2 q(\alpha_1, z_1)  \biggr| \le  n_2 \epsilon' q(\alpha_1, z_1) \right\},\\
    E_1 := \bbC_1 \cap \bbD_1 \cap \bbF_1,
\end{gather*}
where $\eta = \tau\sqrt{12\log n}$ as in the proof of Proposition~\ref{prop:init-error}, $\gamma>0$ and $\epsilon'>0$ are fixed constants that will be specified later. 
Note that $\bbC_1$, $\bbD_1$, $\bbF_1$ and $E_1$ are all measurable with respect to the $\sigma$-algebra generated by $\{ \alpha_i, z_i\}_{i=1}^n$ and $A^{(-1)}$.
The proof of Proposition \ref{prop:init-error} implies that $\Prob(\bbC_1)\ge 1-(n-1)^{-2}\ge 1-n^{-3/2}$ and $\Prob(\bbD_1)\ge 1-(n-1)^{-(1+2C_1)}\ge 1-n^{-(1+C_1)}$ for some constant $C_1>0$ that depends on $\gamma$.

Conditional on $\alpha_1$ and $z_1 \in \calB_\epsilon$, we provide a probabilistic bound for $\bbF_1$ on event $\bbC_1$. 
With slight abuse of notation, let $\bbE$ denote the expectation with respect to the measure of $z$'s restricted on $\bbC_1$.
When $i \in J_2$ and $\sigma_1 = 1$, we note 
\begin{align*}
    \bbE[ P_{1i}^2 \mid \alpha_1, z_1] & \le  e^{2\alpha_1} \bbE[ \exp(2 \alpha_i + 2 z_1^\top H z_i) \mid \alpha_1, z_1] \\
    & = e^{2 \alpha_1 + 2 \alphaavg- 2z_1^\top H \mu }\bbE[ e^{2 \omega_i}]{\bbE}[ \exp\bigl( 2z_1^\top H (z_i + \mu) \bigr) \mid z_1]\\
    & \le C_1' e^{ 2 \alpha_1 + 2 \alphaavg } \exp\bigl(\tau^2 \| H z_1\|_2^2/2\bigr)\\
    & \le C_2' e^{ 2 \alpha_1 + 2 \alphaavg } ( 1 +\tau^2),  \qquad \text{for } n \text{ sufficiently large}.
\end{align*}
The first inequality in the preceding display holds as a result of $S(x) \le e^x$ for $x\in \bbR$.
In the second inequality, we use Assumption~\ref{assump:alpha} to bound $\bbE[e^{2w_i}]$, apply Lemma~\ref{lem:limit-mgf} and note that both $z_1$ and $z_i$ are bounded on $\bbC_1$ and $\{ z_1 \in \calB_\epsilon \}$.  
The last inequality holds for $n$ sufficiently large as $\tau \rightarrow 0$ as $n \rightarrow \infty $.  
We proceed to bound $P_{1i}$ on $\alpha_1, z_1 \in \calB_\epsilon$
\begin{align*}
    P_{1i} & \le \exp(\alpha_1 + \alphaavg + \omega') \exp\bigl( z_1^\top H z_i) \le C_3'  \exp(\alpha_1 + \alphaavg + \omega'),
\end{align*}
where we again apply $S(x) \le e^x$ for $x \in \bbR$ and $ z_1 $ is finite on $\{ z_1 \in \calB_\epsilon\}$.
On $\bbC_1 \cap \{ z_1 \in \calB_\epsilon\}$, by Assumption~\ref{assump:alpha}, we bound $q$ from below by 
\begin{equation*}
    q(\alpha_1, z_1) \ge C_4' \exp( \alpha_1 + \alphaavg - \underline{\omega}), \quad \text{for } n \text{ sufficiently large}.
\end{equation*}
We apply Bernstein's inequality and obtain 
\begin{align*}
    \bbP\left( \biggl| \sum_{i \in J_2} P_{1i} - n_2 q(\alpha_1, z_1)\biggr| \ge t \,\Big|\, \alpha_1, z_1\right) & \le 2\exp\left( - \frac{t^2}{2n_2 C_2' e^{ 2  \alpha_1 + 2\alphaavg } (1 + \tau^2) + (2/3) C_3' e^{\alpha_1 + \alphaavg + \omega'} t }\right). 
\end{align*}
Take $t = n_2 \epsilon' q(\alpha_1, z_1) \ge C_4' n_2 \epsilon' e^{\alpha_1 + \alphaavg - \underline{\omega}} $ and we further obtain, for some proper constants $C_5'$ and $C_2$,
\begin{equation}
    \label{eqn:bound-P1j}
\begin{aligned}
    \bbP\left( \bigg| \sum_{i \in J_2} P_{1i} - n_2 q(\alpha_1, z_1) \biggr| \ge  n_2 \epsilon' q(\alpha_1, z_1)  \,\Big|\, \alpha_1, z_1\right)  & \le 2 \exp\left( - \frac{C_4'^2 n_2 \epsilon'^2 e^{- 2 \underline{\omega}}}{2 C_2'  (1+\tau^2) + (2/3) C_3' C_4'  \epsilon' e^{ \omega' - \underline{\omega} } }\right) \\
    & \le 2 \exp\left( - C_5' n_2 \epsilon' e^{-\underline{\omega} - \omega'}\right) \\
    & \le \frac{1}{2} n^{- ( 1 + C_2)}.
\end{aligned}
\end{equation}
The second inequality in the preceding display holds as $e^{\omega' - \underline{\omega}} \gtrsim 1$ by Assumption~\ref{assump:alpha}.
We apply \eqref{eq:alpha5} in Assumption~\ref{assump:alpha} to obtain the last inequality. 
A similar argument yields that conditional on $\bbC_1$, for $z_1 \in \calB_\epsilon$
\begin{equation}
    \label{eqn:bound-P1j-p}
     \bbP\left( \biggl| \sum_{i \in J_1} P_{1i} - n_1 p(\alpha_1, z_1) \biggr| \ge  n_1 \epsilon' p(\alpha_1, z_1) \,\Big|\, \alpha_1, z_1\right) \le \frac{1}{2} n^{- ( 1 + C_2)}.
\end{equation}
Combining \eqref{eqn:bound-P1j} and \eqref{eqn:bound-P1j-p}, we obtain that conditional on $\alpha_1$ and $z_1\in \calB_\epsilon$, 
\[
    \bbP( \bbF_1^c \mid \bbC_1 ) \le n^{-(1+C_2)}.
\]
Together with the probablistic bound on $\bbC_1$, for some constant $C_2''$, we have conditional on $\alpha_1$ and $z_1\in \calB_\epsilon$,
\begin{equation}
    \bbP( \bbF_1^c )  \le \bbP( \bbF_1^c \mid \bbC_1 ) \bbP(\bbC_1) + \bbP( \bbF_1^c \mid \bbC_1^c ) \bbP(\bbC_1^c) \le \bbP( \bbF_1^c \mid \bbC_1 ) + \bbP(\bbC_1^c) \le n^{-(1+C_2)}  + n^{-3/2} \le n^{ - ( 1 + C_2'')},
    \label{eqn:bound-F1}
\end{equation}
Inspection of the above argument reveals that as long as $z_1 \in \calB_\epsilon$, the constant $C_2''$ in the preceding display does not depend on $\alpha_1$ and $z_1$, whence we obtain 
\begin{align}
\Prob(E_1^c \mid z_1 \in \calB_\epsilon )\le \Prob(\bbC_1^c) + \Prob(\bbD_1^c) +\Prob(\bbF_1^c) \le n^{-3/2} + n^{-(1+C_1)} + n^{-(1+C_2'')} \le n^{-(1+C)}, \label{eq:bound-CDF}
\end{align}
with $0<C<\min\{\frac{1}{2}, C_1, C_2''\}$.
It will be useful at the end of the proof to give a probalistic bound on $E_1$ without conditioning on $\{ z_1 \in \calB_\epsilon\}$ 
\begin{equation}
    \bbP(E_1^c) \le \bbP(E_1^c \mid z_1 \in \calB_\epsilon) + \bbP(z_1 \in \calB_\epsilon) \le n^{ -(1 + C)} + \exp\Bigl(- (1 -\epsilon/2)\frac{\rho^2}{ 2 \tau^2} \Bigr),
    \label{eqn:E1c-bound}
\end{equation}
where the last inequality follows from \eqref{eq:ball_bound} in Lemma \ref{lem:edgecount-rate}.

Next observe that
\begin{equation}
\begin{aligned}
& \Prob_{\{\sigma_1=1\}}(\hsigma_1=2 \text{\ and\ } E_1) \\
& =   \Prob_{\{\sigma_1=1\}}\left(\frac{1}{m_1}\sum_{i\in \widehat{J}_1} A_{1,i}\le \frac{1}{m_2}\sum_{i\in \widehat{J}_2} A_{1,i} \text{\ and\ } E_1 \right)
\\
& = \Expect_{\{\sigma_1=1\}}^{\alpha_1, z_1} \left[\indc{z_1\in\calB_\epsilon} \Prob\left(\frac{1}{m_1}\sum_{i\in \widehat{J}_1} A_{1,i}\le \frac{1}{m_2}\sum_{i\in \widehat{J}_2} A_{1,i} \text{\ and\ } E_1\bigg|\alpha_1,z_1 \right)\right]
\\
&~~~ + \Expect_{\{\sigma_1=1\}}^{\alpha_1, z_1} \left[\indc{z_1\in\calB_\epsilon^c} \Prob\left(\frac{1}{m_1}\sum_{i\in \widehat{J}_1} A_{1,i}\le \frac{1}{m_2}\sum_{i\in \widehat{J}_2} A_{1,i} \text{\ and\ } E_1\bigg|\alpha_1,z_1 \right)\right].
\end{aligned}
\label{eq:bound-decomp}
\end{equation}
We deal with the first term in the above display. Assume $z_1\in \calB_\epsilon$ in the following. We then have
	\begin{align}
        & \Prob\left(\frac{1}{m_1}\sum_{i\in \widehat{J}_1} A_{1,i}\le \frac{1}{m_2}\sum_{i\in \widehat{J}_2} A_{1,i} \text{\ and\ } E_1\,\bigg|\,\alpha_1,z_1 \right) \nonumber \\
        & = \bbE \left( \bbE\left[ \indc{E_1}  \Indc\left(\frac{1}{m_1}\sum_{i\in \widehat{J}_1} A_{1,i}\le \frac{1}{m_2}\sum_{i\in \widehat{J}_2} A_{1,i}\right) \, \bigg|\, \{ \alpha_i,z_i\}_{i=1}^n  \right] \, \bigg| \, \alpha_1, z_1 \right) \nonumber \\
        & \le \bbE \left( \bbE\left[ \indc{E_1}  \Indc\left(\frac{1}{m_1}\sum_{i\in {J}_{11}} A_{1,i}\le \frac{1}{m_2}\sum_{i\in {J}_{22}} A_{1,i} + \frac{1}{m_2}\sum_{i \in J_{21}} A_{1,i} \right) \, \bigg|\, \{ \alpha_i,z_i\}_{i=1}^n  \right] \, \bigg| \, \alpha_1, z_1 \right) . \label{eq:bound-decomp2}
	\end{align}
The equality holds because of the tower property of conditional expectations. 
We now consider the conditional expectaion inside the round brackets in the preceding display.
Conditional on $\{ \alpha_i,z_i\}_{i=1}^n$, we define for $i\in [n]$
\begin{gather*}
    W_i \stackrel{\text{ind}}{\sim} \mathrm{Bernoulli}({P}_{1i}).
\end{gather*}
Contional on  $\{ \alpha_i,z_i\}_{i=1}^n$, $(A_{1i})_{2\le i \le n}$ are mutually independent and independent of $A^{(-1)}$, whence we have, for any $t>0$ measurable with respect to the $\sigma$-algebra generated by $\{ \alpha_i, z_i\}_{i=1}^n$ and $A^{(-1)}$,
\begin{align}
    &\quad \bbE\left[ \indc{E_1}  \Indc\left(\frac{1}{m_1}\sum_{i\in {J}_{11}} A_{1,i}\le \frac{1}{m_2}\sum_{i\in {J}_{22}} A_{1,i} + \frac{1}{m_2}\sum_{i \in J_{21}} A_{1,i} \right) \, \bigg|\, \{ \alpha_i,z_i\}_{i=1}^n  \right] \nonumber\\
    &  = \bbE\left[ \indc{E_1}  \Indc\left(\frac{1}{m_1}\sum_{i\in {J}_{11}} W_i \le \frac{1}{m_2}\sum_{i\in {J}_{22}} W_{i} + \frac{1}{m_2}\sum_{i \in J_{21}} W_{i} \right) \, \bigg|\, \{ \alpha_i,z_i\}_{i=1}^n  \right] \nonumber \\
&  = \bbE\left[ \indc{E_1}  \bbE\left[\Indc\left(\frac{1}{m_1}\sum_{i\in {J}_{11}} W_i \le \frac{1}{m_2}\sum_{i\in {J}_{22}} W_{i} + \frac{1}{m_2}\sum_{i \in J_{21}} W_{i} \right) \, \bigg| \, \{ \alpha_i,z_i\}_{i=1}^n, A^{(-1)} \right] \, \bigg|\, \{ \alpha_i,z_i\}_{i=1}^n  \right] \nonumber \\
&  \le \bbE\left[ \indc{E_1}  \prod_{i\in J_{22}} \bigl( P_{1i} e^{t/m_2} + 1 - P_{1i} \bigr) \prod_{i\in J_{21}} \bigl( P_{1i} e^{t/m_2} + 1 - P_{1i} \bigr) \prod_{i\in J_{11}} \bigl( P_{1i} e^{-t/m_1} + 1 - P_{1i} \bigr) \, \bigg|\, \{ \alpha_i,z_i\}_{i=1}^n  \right]\nonumber \\
& \le \bbE\left[ \indc{E_1}  \exp\left( \sum_{i\in J_{22}} P_{1i} (e^{t/m_2}  -1 ) + \sum_{i\in J_{21}}  P_{1i} (e^{t/m_2}  - 1) + \sum_{i\in J_{11}}  P_{1i} (e^{-t/m_1}  - 1) \right) \, \bigg|\, \{ \alpha_i,z_i\}_{i=1}^n  \right].
\label{eq:penultimate-bound}
\end{align}
The second equality in the preceding display holds by the tower property of conditional expectations and because $E_1$ is measurable with respect to the $\sigma$-algebra generated by \(\{ \alpha_i,z_i\}_{i=1}^n\) and \(A^{(-1)}\).
In the first inequality, we apply the Chernoff bound and note that $m_1$, $m_2$, $P_{1i}$'s and $(J_{u_1 u_2})_{u_1, u_2 \in [2]}$ are all measurable with respect to the $\sigma$-algebra generated by \(\{ \alpha_i,z_i\}_{i=1}^n\) and \(A^{(-1)}\).
The second inequality holds as we note $1 + x \le e^x$ for all $x \in \bbR$.
Write $p = p(\alpha_1, z_1)$ and $q = q(\alpha_1, z_1)$ as shorthands. 
Define the following quantities
\begin{gather*}
    K_1 : = \exp\bigl( m_2 ( e^{t/m_2} -1) q + m_1  (e^{- t/m_1 } -1) p \bigr),\\
    K_2 : = \exp\left( ( e^{t/m_2} -1) \Bigl( \sum_{i\in J_{22}} P_{1i} - m_2 q\Bigr) \right),\\
    K_3 : = \exp\left( ( e^{-t/m_1} -1) \Bigl( \sum_{i\in J_{11}} P_{1i} - m_1 p\Bigr) \right),\\
    K_4 : = \exp\Bigl( ( e^{t/m_2} -1)  \sum_{i\in J_{21}} P_{1i} \Bigr).
\end{gather*}
We note that \eqref{eq:penultimate-bound} is the same as
\(
\bbE\bigl[ \indc{E_1} K_1 K_2 K_3 K_4 \mid \{\alpha_i,z_i\}_{i=1}^n  \bigr].
\)  
Set $t=\frac{m_1 m_2}{m_1+ m_2}\log(\frac{p}{q})$. 
Next we deal with $K_1$, $K_2$, $K_3$ and $K_4$ separately.

Before we proceed, we note the following useful facts. For any fixed $\gamma>0$, we make $n$ sufficiently large so that $\delta_n<\gamma$. Hence, $n_1, n_2\in \left[(1-\gamma)\frac{n}{2}, (1+\gamma)\frac{n}{2}\right]$. On event $E_1 \subset \bbD_1$, we have 
\(|{\{i:\hsigma^{(-1,0)}_i\neq \sigma_i\}}| \le e^{\underline{\omega}}\sum_{\{i:\hsigma^{(-1,0)}_i\neq \sigma_i\}} e^{\omega_i} \le \frac{\gamma}{2} (n-1)<\frac{\gamma}{2} n.\)
Therefore, we get $m_{12}\le \frac{\gamma}{2} n$, $m_{21}\le \frac{\gamma}{2} n$, and hence $m_1, m_2\in \left[(1-2\gamma)\frac{n}{2}, (1+2\gamma)\frac{n}{2}\right]$. Furthermore, for $z_1\in \calB_\epsilon$, the lower bound \eqref{eq:z0-H-mu-bound} holds for $z_1^\top H \mu$. We denote $\underline{\xi}=e^{\frac{\epsilon}{8}\mu^\top H\mu}$. For $z_1\in \calB_\epsilon$, on event $E_1$, both $e^{z_1^\top H z_i}$ and $e^{z_1^\top H \mu}$ are bounded above by some constant $\overline{\xi}$, which is larger than 1 since $z_1^\top H \mu>0$ when $z_1\in \calB_\epsilon$. 


First we deal with the main term $K_1$. Since \eqref{eq:p-decomposition}, \eqref{eq:q-decomposition}, \eqref{eq:Dp-upper-bound}, \eqref{eq:Dq-upper-bound}, \eqref{eq:D-bound}, \eqref{eq:Dp-lower-bound} and \eqref{eq:Dq-lower-bound} continue to hold for $p$ and $q$, we obtain
\begin{align}
	(1-\kappa) e^{2\alphaavg} e^{z_1^\top H \mu} D(\omega_1,z_1) & \le p\le e^{2\alphaavg} e^{z_1^\top H \mu} D(\omega_1,z_1), \quad \mbox{and} \label{eq:p_bound} \\
	(1-\kappa) e^{2\alphaavg} e^{-z_1^\top H \mu} D(\omega_1,z_1) & \le q \le e^{2\alphaavg} e^{-z_1^\top H \mu} D(\omega_1,z_1), \label{eq:q_bound}
\end{align}
where $\kappa=1-\frac{1}{4}(1+\underline{\xi}^{-1})^2\in (0,1)$ and $0<e^{-2\underline{\omega}} \underline{D} \le D(\omega_1,z_1)\le e^{\omegabar} \overline{D}$. For this particular choice of $\kappa$, we have 
\begin{align}
\sqrt{1-\kappa} e^{z_1^\top H\mu}-1 \ge \frac{1}{2}(1+\underline{\xi}^{-1})\underline{\xi} -1=\frac{1}{2}(\underline{\xi} -1)>0. \label{eq:kappa}
\end{align}
By direct calculation,
\begin{align}
m_2 q (e^{t/m_2}-1)+m_1 p (e^{-t/m_1}-1)= & -\left(m_1 p+m_2 q-(m_1+m_2)p^{\frac{m_1}{m_1+ m_2}}q^{\frac{m_2}{m_1+ m_2}}\right) \nonumber \\
\le & -\frac{n}{2}\left(p+q-2\gamma(p-q)-2\left(\frac{p}{q}\right)^{\gamma}\sqrt{pq}\right). \label{eq:exponent_bound11}
\end{align}
We aim to show that the term inside the round brackets of the last display and $-\frac{n}{2}I(\alpha_1,z_1)$ are close.
To this end, first note that
\begin{align*}
\frac{p+q-2\gamma(p-q)-2\left(\frac{p}{q}\right)^{\gamma}\sqrt{pq}}{(\sqrt{p}-\sqrt{q})^2} = & 1- 2\gamma\frac{\sqrt{p}+\sqrt{q}}{\sqrt{p}-\sqrt{q}}- 2\left[\left(\frac{p}{q}\right)^{\gamma}-1\right]\frac{\sqrt{pq}}{(\sqrt{p}-\sqrt{q})^2}
\end{align*}
Using \eqref{eq:p_bound}, \eqref{eq:q_bound} and \eqref{eq:kappa}, we obtain
\begin{align*}
2\gamma\frac{\sqrt{p}+\sqrt{q}}{\sqrt{p}-\sqrt{q}} & \le 2\gamma\frac{e^{\frac{1}{2}z_1^\top H\mu}+e^{-\frac{1}{2}z_1^\top H\mu}}{\sqrt{1-\kappa} e^{\frac{1}{2}z_1^\top H\mu}-e^{-\frac{1}{2}z_1^\top H\mu}}= 2\gamma\frac{e^{z_1^\top H\mu}-1}{\sqrt{1-\kappa}e^{z_1^\top H\mu}-1}\le 4\gamma \frac{\overline{\xi}-1}{\underline{\xi}-1}, \\
\left(\frac{p}{q}\right)^{\gamma}-1 & \le \left(\frac{e^{z_1^\top H \mu}}{(1-\kappa)e^{-z_1^\top H \mu}}\right)^{\gamma}-1\le \left(\frac{\overline{\xi}^2}{1-\kappa}\right)^{\gamma}-1, \quad \mbox{and} \\
\frac{\sqrt{pq}}{(\sqrt{p}-\sqrt{q})^2} & \le \frac{1}{ (\sqrt{1-\kappa} e^{\frac{1}{2}z_1^\top H\mu} -e^{-\frac{1}{2}z_1^\top H\mu} )^2} = \frac{e^{z_1^\top H\mu}}{(\sqrt{1-\kappa} e^{z_1^\top H\mu} -1)^2} 
{\le} 
\frac{4\overline{\xi}}{(\underline{\xi}-1)^2}.
\end{align*}
We choose $\gamma$ such that the second last and third last displays are sufficiently small.
Hence, for sufficiently small  constant $\gamma>0$,
\begin{align}
\frac{p+q+2\gamma(p-q)-2\left(\frac{p}{q}\right)^{\gamma}\sqrt{pq}}{(\sqrt{p}-\sqrt{q})^2} \ge  1-\frac{\epsilon}{4}. \label{eq:exponent_bound12}
\end{align}
Also note that
\begin{align*}
	I(\alpha_1,z_1)= & -2\log\left(1-\frac{1}{2}(\sqrt{p}-\sqrt{q})^2-\frac{1}{2}(\sqrt{1-p}-\sqrt{1-q})^2\right).
\end{align*}
Let
\begin{align*}
	\beta & =   \frac{1}{2}(\sqrt{p}-\sqrt{q})^2+\frac{1}{2}(\sqrt{1-p}-\sqrt{1-q})^2 \\
	& = \frac{1}{2} (\sqrt{p}-\sqrt{q})^2 \left[1+\frac{(\sqrt{p}+\sqrt{q})^2}{(\sqrt{1-p}+\sqrt{1-q})^2}\right].
\end{align*}
By \eqref{eq:p_bound}, \eqref{eq:q_bound}, \eqref{eq:alpha1} of Assumption \ref{assump:alpha}, and that $e^{z_1^\top H\mu}\le \overline{\xi}$, we have $p,q\le \frac{3}{4}$.
Thus
\begin{align*}
\frac{(\sqrt{p}+\sqrt{q})^2}{(\sqrt{1-p}+\sqrt{1-q})^2} & \le  \frac{e^{2\alphaavg}\left(e^{\frac{1}{2}z_1^\top H \mu}+e^{-\frac{1}{2}z_1^\top H \mu}\right)^2 D(\omega_1,z_1)}{(\frac{1}{2}+\frac{1}{2})^2} \\
& \le  e^{2\alphaavg+\omegabar}(\overline{\xi}^{\frac{1}{2}}+\underline{\xi}^{-\frac{1}{2}})^2 \overline{D},
\end{align*}
which goes to 0 as $2\alphaavg+\omegabar\to -\infty$ by \eqref{eq:alpha1} of Assumption \ref{assump:alpha}. Consequently,
\begin{align*}
	\beta & \le \frac{1}{2} e^{2\alphaavg+\omegabar} \left(e^{\frac{1}{2}z_1^\top H \mu}-\sqrt{1-\kappa}e^{-\frac{1}{2}z_1^\top H \mu}\right)^2 \overline{D} \left[1+ e^{2\alphaavg+\omegabar}\left(\overline{\xi}^{\frac{1}{2}}+\underline{\xi}^{-\frac{1}{2}}\right)^2 \overline{D}\right] \\
	& \le \frac{1}{2} e^{2\alphaavg+\omegabar} \left(\overline{\xi}^{\frac{1}{2}}- \frac{1}{2}(1+\underline{\xi}^{-1}) \overline{\xi}^{-\frac{1}{2}}\right)^2 \overline{D} \left[1+ e^{2\alphaavg+\omegabar} \bigl(\overline{\xi}^{\frac{1}{2}}+\underline{\xi}^{-\frac{1}{2}} \bigr)^2 \overline{D}\right],
\end{align*}
which also goes to 0 as $2\alphaavg+\omegabar\to -\infty$.
Since $\log(1-\beta)\ge -\beta-\beta^2$ for all $0<\beta<1/2$, we obtain $I(\alpha_1,z_1) \le  2\beta+2\beta^2$. Therefore,
\begin{align*}
\frac{I(\alpha_1,z_1)}{(\sqrt{p}-\sqrt{q})^2}
& \le   (1+ \beta)\left(1 + \frac{(\sqrt{p}+\sqrt{q})^2}{(\sqrt{1-p}+\sqrt{1-q})^2} \right).
\end{align*}
Since the limits of $\beta$ and $\frac{(\sqrt{p}+\sqrt{q})^2}{(\sqrt{1-p}+\sqrt{1-q})^2}$ are both zeros, we have for large values of $n$ that
\begin{align}
	\frac{I(\alpha_1,z_1)}{(\sqrt{p}-\sqrt{q})^2} \le 1+\frac{\epsilon}{4}. \label{eq:exponent_bound13}
\end{align}
We combine \eqref{eq:exponent_bound11}, \eqref{eq:exponent_bound12} and \eqref{eq:exponent_bound13} to obtain
\begin{align}
K_1 \le \exp\left\{-\frac{n}{2}\,
\frac{1-{\epsilon}/{4}}{1+{\epsilon}/{4}} \,
I(\alpha_1,z_1)\right\}
\le \exp\left\{-\frac{n}{2}\left(1-\frac{\epsilon}{2}\right)
I(\alpha_1,z_1)\right\}.
\label{eq:exponent_bound1}
\end{align}

To bound $K_2$, we have the decomposition 
\begin{align*}
K_2=\exp\left( ( e^{t/m_2} -1) \Bigl( \sum_{i\in J_{2}} P_{1i} - n_2 q -\sum_{i\in J_{12}}P_{1i}+(n_2-m_2)q\Bigr) \right).
\end{align*}
By \eqref{eq:p_bound} and \eqref{eq:q_bound}, we can bound $e^{t/m_2}-1$ by a constant
\begin{align}
    e^{t/m_2}-1=\left(\frac{p}{q}\right)^{\frac{m_1}{m_1+m_2}}-1 \le \left(\frac{p}{q}\right)^{\gamma+\frac{1}{2}}-1 \le \left(\frac{e^{z_1^\top H \mu}}{(1-\kappa)e^{-z_1^\top H \mu}}\right)^{\gamma+\frac{1}{2}}-1 
    {\le} 
    \frac{\overline{\xi}^2}{1-\kappa}-1. \label{eq:exp-tm-bound}
\end{align}
We then bound $|\sum_{i\in J_{2}} P_{1i} - n_2 q|$, $\sum_{i\in J_{12}}P_{1i}$ and $(n_2-m_2)q$ one by one. By definition, on event $E_1$ we have 
\begin{align}
\bigg|\sum_{i\in J_{2}} P_{1i} - n_2 q\bigg| \le n_2 \epsilon' q <  n\epsilon' q. \label{eq:K211_bound}
\end{align}
We use $\log(1-x)\le -x$ for $0<x<1$ to obtain
\begin{align*}
I(\alpha_1,z_1) & \ge  2\beta  =  (\sqrt{p}-\sqrt{q})^2 \left[1+\frac{(\sqrt{p}+\sqrt{q})^2}{(\sqrt{1-p}+\sqrt{1-q})^2}\right].
\end{align*}
By \eqref{eq:p_bound} and \eqref{eq:q_bound}, we get
\begin{align}
\frac{I(\alpha_1,z_1)}{(\sqrt{p}-\sqrt{q})^2 } & \ge
1+\frac{1}{4}\left[(1-\kappa) e^{2\alphaavg-2\underline{\omega}}(e^{\frac{1}{2}z_1^\top H \mu}+ e^{-\frac{1}{2}z_1^\top H \mu})^2 \underline{D}\right] \nonumber \\
& \ge 1+ \frac{1}{4}\left[(1-\kappa) e^{2\alphaavg-2\underline{\omega}}(\underline{\xi}^{\frac{1}{2}}+ \overline{\xi}^{-\frac{1}{2}})^2 \underline{D}\right] \to 1, \label{eq:K212_bound}
\end{align}
as $\alphaavg\to -\infty$. 
Following \eqref{eq:kappa}, we also have
\begin{align}
\frac{q}{(\sqrt{p}-\sqrt{q})^2}=\frac{1}{\left(\sqrt{p/q}-1\right)^2}\le \frac{1}{(\sqrt{1 - \kappa} e^{z_1^\top H \mu}-1)^2} \le \frac{4}{(\underline{\xi}-1)^2}. \label{eq:K213_bound}
\end{align}
Putting \eqref{eq:exp-tm-bound}, \eqref{eq:K211_bound}, \eqref{eq:K212_bound} and \eqref{eq:K213_bound} together, for a suitably chosen $\epsilon'$, we obtain
\begin{equation}
\label{eq:K21_bound}
\begin{aligned}
(e^{t/m_2}-1) \bigg|\sum_{i\in J_{2}} P_{1i} - n_2 q \bigg| & \le \left(\frac{\overline{\xi}^2}{1-\kappa}-1\right) n \epsilon' I(\alpha_1, z_1) \frac{(\sqrt{p}-\sqrt{q})^2}{I(\alpha_1,z_1)} \frac{q}{(\sqrt{p}-\sqrt{q})^2}  \\
& \le \frac{\epsilon}{32} n I(\alpha_1,z_1). 
\end{aligned}
\end{equation}
Note 
\(
P_{1i} \le e^{\alpha_1+\alpha_i+z_1^\top H z_i}\le \overline{\xi} e^{2\alphaavg+\omega_1}e^{\omega_i},
\)
then on event $E_1$ we have
\begin{align*}
    \sum_{i\in J_{12}} P_{1i} \le \overline{\xi} e^{2\alphaavg+\omega_1}\sum_{i\in J_{12}} e^{\omega_i} \le \overline{\xi} e^{2\alphaavg+\omega_1}\sum_{ \{i:\hsigma_i^{(-1,0)}\neq \sigma_i\} } e^{\omega_i} \le \overline{\xi} e^{2\alphaavg + \omega_1} e^{-\underline{\omega}} \frac{\gamma}{2} (n-1) .
\end{align*}
By \eqref{eq:q_bound}, the definition of $D(\omega_1,z_1)$ and Assumption~\ref{assump:alpha}, we see 
\(
    q \gtrsim {e^{2 \alphaavg + \omega_1}}.
\)
In view of \eqref{eq:K212_bound} and \eqref{eq:K213_bound}, we make $\gamma$ small enough such that
\begin{equation}
\label{eq:K22_bound}
\begin{aligned}
    (e^{t/m_2}-1)\sum_{i\in J_{12}} P_{1i} & \le \left(\frac{\overline{\xi}^2}{1-\kappa}-1\right)\overline{\xi} e^{2\alphaavg + \omega_1} e^{-\underline{\omega}} \frac{\gamma}{2} (n-1)\\
    &\le
    \left(\frac{\overline{\xi}^2}{1-\kappa}-1\right) \overline{\xi} \frac{e^{2\alphaavg + \omega_1}}{q} 
    \frac{q}{(\sqrt{p} - \sqrt{q})^2}
    \frac{(\sqrt{p} - \sqrt{q})^2}{I(\alpha_1,z_1)}
    e^{-\underline{\omega}} \frac{\gamma}{2} n I(\alpha_1, z_1)\\
    & \le  \frac{\epsilon}{32} n I(\alpha_1,z_1).
\end{aligned}
\end{equation}
Since $n_2-m_2\le \frac{3}{2}\gamma n$, combining \eqref{eq:exp-tm-bound}, \eqref{eq:K212_bound} and \eqref{eq:K213_bound} we obtain
\begin{align}
(e^{t/m_2}-1) (n_2-m_2)q & \le \left(\frac{\overline{\xi}^2}{1-\kappa}-1\right) \frac{3}{2}\gamma n I(\alpha_1, z_1) \frac{(\sqrt{p}-\sqrt{q})^2}{I(\alpha_1,z_1)} \frac{q}{(\sqrt{p}-\sqrt{q})^2} \nonumber \\
& \le \frac{\epsilon}{32} n I(\alpha_1,z_1)\label{eq:K23_bound}
\end{align}
for small enough $\gamma$.
Combining \eqref{eq:K21_bound}, \eqref{eq:K22_bound} and \eqref{eq:K23_bound}, we obtain
\begin{align}
K_2\le \exp\left\{\frac{3\epsilon}{32} n I(\alpha_1,z_1)\right\} \label{eq:exponent_bound2}
\end{align}


The same bound for $K_3$ is obtained in a similar way to bound $K_2$
\begin{align}
K_3\le \exp\left\{\frac{3\epsilon}{32} n I(\alpha_1,z_1)\right\}. \label{eq:exponent_bound3}
\end{align}

Lastly, the following bound for $K_4$ is obtained by the same argument as in establishing \eqref{eq:K22_bound}
\begin{align}
K_4\le \exp\left\{\frac{\epsilon}{32} n I(\alpha_1,z_1) \right\}. \label{eq:exponent_bound4}
\end{align}

Combining \eqref{eq:penultimate-bound}, \eqref{eq:exponent_bound1}, \eqref{eq:exponent_bound2}, \eqref{eq:exponent_bound3}, \eqref{eq:exponent_bound4}, we get
\begin{align*}
&\bbE\left[ \indc{E_1} \Indc\left(\frac{1}{m_1}\sum_{i\in {J}_{11}} A_{1,i}\le \frac{1}{m_2}\sum_{i\in {J}_{22}} A_{1,i} + \frac{1}{m_2}\sum_{i \in J_{21}} A_{1,i} \right) \, \bigg|\, \{ \alpha_i,z_i\}_{i=1}^n  \right] \\
& \qquad \qquad \le \exp\left\{-\frac{n}{2} (1-\frac{15}{16}\epsilon) I(\alpha_1,z_1)  \right\} \le \exp\left\{-\frac{n}{2} (1-\epsilon) I(\alpha_1,z_1)  \right\}.
\end{align*}
Since the rightmost side of the above display depends only on $(\alpha_1,z_1)$, by \eqref{eq:bound-decomp2} we obtain for $z_1\in\calB_\epsilon$
\begin{align*}
& \Prob\left(\frac{1}{m_1}\sum_{i\in \widehat{J}_1} A_{1,i}\le \frac{1}{m_2}\sum_{i\in \widehat{J}_2} A_{1,i} \text{\ and\ } E_1\,\bigg|\,\alpha_1,z_1 \right) \le \exp\left\{-\frac{n}{2} (1-\epsilon) I(\alpha_1,z_1)\right\}.
\end{align*}
By \eqref{eq:bound-decomp}, we further have
\begin{align*}
\Prob_{\{\sigma_1=1\}}(\hsigma_1=2 \text{\ and\ } E_1) & \le \Expect_{\{\sigma_1=1\}}^{\alpha_1,z_1}\left[\indc{z_1\in \calB_\epsilon} \exp\left\{-\frac{n}{2}(1-\epsilon)I(\alpha_1,z_1)\right\}\right] +\Prob_{\{\sigma_1=1\}}(z_1\in \calB_\epsilon^c) \\
& \le \Expect_{\{\sigma_1=1\}}^{\alpha_1,z_1}\left[\indc{z_1\in \calB_\epsilon} \exp\left\{-\frac{n}{2}(1-\epsilon)I(\alpha_1,z_1)\right\}\right] + \exp\left\{-(1-\epsilon/2)\frac{\rho^2}{2\tau^2}\right\},
\end{align*}
where the last inequality is due to \eqref{eq:ball_bound} in Lemma \ref{lem:edgecount-rate}. Finally, in view of \eqref{eqn:E1c-bound}, we have
\begin{align*}
\Prob_{\{\sigma_1=1\}}(\hsigma_1=2) & \le \Prob_{\{\sigma_1=1\}}(\hsigma_1=2 \text{\ and\ } E_1) + \Prob_{\{\sigma_1=1\}}(E_1^c) \\
& \le \Expect_{\{\sigma_1=1\}}^{\alpha_1,z_1}\left[\indc{z_1\in \calB_\epsilon} \exp\left\{-\frac{n}{2}(1-\epsilon)I(\alpha_1,z_1)\right\}\right] + \exp\left\{-(1-\epsilon)\frac{\rho^2}{2\tau^2}\right\} + n^{-(1+C)}.
\end{align*}
\bigskip

\end{proof}

\subsection{Proof of Theorem \ref{thm:upper}}
The proof strategy here is similar to that used in the proof of Theorem 2 in \cite{gao2015achieving}.
For $i\in [n]$ there is a permutation $\pi_i$ such that
\begin{equation*}
	\ell(\sigma, \wh\sigma^{(-i,0)}) = \frac{1}{n} \sum_{j=1}^n \indc{\sigma_j \neq \pi_i(\hat\sigma_j^{(-i,0)})}.
\end{equation*}
Without loss of generality, we may assume that $\pi_1 = \mathrm{Id}$ is the identity permutation.
Then by Proposition \ref{prop:init-error} and Lemma 4 in \cite{gao2015achieving}, we obtain that for some constant $C>0$, for each $i=2,\cdots, n$ with probability at least $1 - n^{-(1+C)}$,
\begin{equation*}
	\wh\sigma_i = \pi_i(\wh\sigma^{(-i,0)}_i).
\end{equation*}
Together with Proposition \ref{prop:init-error-edge-counting}, we obtain that for $i=1,\dots, n$,
\begin{align}
\Prob\{ \sigma_i \neq \wh\sigma_i \}
& \leq \Prob\{ \sigma_i \neq \pi_i(\wh\sigma^{(-i,0)}_i),~  \wh\sigma_i = \pi_i(\wh\sigma^{(-i,0)}_i) \}
+ \Prob\{ \wh\sigma_i \neq \pi_i(\wh\sigma^{(-i,0)}_i) \}  \nonumber \\
& \leq \countrateuppereps{\epsilon'} + 2 n^{-(1+C)}.
\label{eq:upper-proof-1}
\end{align}
Here, for any fixed $\epsilon \in (0, 1/2)$, we pick
\begin{equation*}
	\epsilon' = \frac{\epsilon}{2}. 
\end{equation*}
By Markov's inequality, We have
\begin{align*}
\Prob\left\{ \ell(\sigma, \wh\sigma) > \countrateuppereps{\epsilon} \right\}
& \leq \frac{1}{\countrateuppereps{\epsilon}} \cdot
\frac{1}{n}\sum_{i=1}^n \Prob\{\sigma_i\neq \wh\sigma_i\} \\
& \leq
\frac{\countrateuppereps{\epsilon'}}{\countrateuppereps{\epsilon}}
+ \frac{2n^{-(1+C)}}{\countrateuppereps{\epsilon}} .
\end{align*}
We divide the remaining proof into two cases depending on the relative magnitude of $\countrateuppereps{\epsilon}$ and $n^{-(1+C/2)}$.

\textbf{Case 1} $\quad$
If $\countrateuppereps{\epsilon} \ge n^{-(1+C/2)}$, then 
\begin{align*}
\Prob\left\{ \ell(\sigma, \wh\sigma) > \countrateuppereps{\epsilon} \right\}
\leq
\frac{\countrateuppereps{\epsilon'}}{\countrateuppereps{\epsilon}}
+ {2n^{-C/2}}.
\end{align*}
To control the ratio $\countrateuppereps{\epsilon'}/\countrateuppereps{\epsilon}$,
we further divide into two subcases. 

\textbf{Subcase 1.1} $\quad$
In this subcase, we assume that 
\begin{equation}
	\label{eq:subcase1}
	e^{-(1-\epsilon)\frac{\rho^2}{2\tau^2}} \ll 
	\Expect_{H_0}^{\alpha_0,z_0}\left[\indc{z_0\in \mathcal{B}_{\epsilon}}e^{-(1-\epsilon) \frac{n}{2} I(\alpha_0, z_0)} \right].
\end{equation}
We then have
\begin{align}
& \Expect_{H_0}^{\alpha_0,z_0}\left[\indc{z_0\in \mathcal{B}_{\epsilon'}}e^{-(1-\epsilon') \frac{n}{2} I(\alpha_0, z_0)} \right]
\nonumber \\
& \leq 
\Expect_{H_0}^{\alpha_0,z_0}\left[\indc{z_0\in \mathcal{B}_{\epsilon}}e^{-(1-\epsilon')\frac{n}{2} I(\alpha_0, z_0)} \right]
+ C e^{-(1-\epsilon)\frac{\rho^2}{2\tau^2}}
\label{eq:upperbd-1}\\
& =
\Expect_{H_0}^{\alpha_0,z_0}\left[\indc{z_0\in \mathcal{B}_{\epsilon}} e^{-(1-\epsilon) \frac{n}{2} I(\alpha_0, z_0)} 
e^{-(\epsilon - \epsilon') \frac{n}{2} I(\alpha_0, z_0)}
\right]
+ C e^{-(1-\epsilon)\frac{\rho^2}{2\tau^2}} 
\nonumber \\
& = o(1) \cdot
\Expect_{H_0}^{\alpha_0,z_0}\left[\indc{z_0\in \mathcal{B}_{\epsilon}} e^{-(1-\epsilon) \frac{n}{2} I(\alpha_0, z_0)} \right]
+ C e^{-(1-\epsilon)\frac{\rho^2}{2\tau^2}} 
\label{eq:upperbd-2}\\
& \ll \countrateuppereps{\epsilon}. 
\label{eq:upperbd-3}
\end{align}
Here, \eqref{eq:upperbd-1} holds since $e^{-(1-\epsilon')m I(\alpha_0, z_0)}\leq 1$ and $\bbP_{H_0}( z_0 \in \calB_{\epsilon'} \setminus \calB_\epsilon) \le \Prob_{H_0}(z_0\notin \mathcal{B}_{\epsilon}) \leq C e^{-(1-\epsilon)\frac{\rho^2}{2\tau^2}}$. 
In \eqref{eq:upperbd-2},
the equality holds since $\epsilon > \epsilon'$ and $n I(\alpha_0,z_0)$ is bounded from below uniformly when $z_0\in \mathcal{B}_{\epsilon}$ by a sequence that diverges to infinity.
Finally, \eqref{eq:upperbd-3} holds since both terms in \eqref{eq:upperbd-2} are $o(\countrateuppereps{\epsilon})$ as $n\to\infty$ under \eqref{eq:subcase1}.
Hence, 
\begin{align}
	\label{eq:upperbd-case1}
\Prob\left\{ \ell(\sigma, \wh\sigma) > \countrateuppereps{\epsilon} \right\}
\leq
\frac{\countrateuppereps{\epsilon'}}{\countrateuppereps{\epsilon}}
+ {2n^{-C/2}} = o(1).
\end{align}

\textbf{Subcase 1.2} $\quad$
In this case, we consider the situation complemental to \eqref{eq:subcase1}, namely
\begin{equation}
	\label{eq:subcase2}
	\Expect_{H_0}^{\alpha_0,z_0}\left[\indc{z_0\in \mathcal{B}_{\epsilon}}e^{-(1-\epsilon) \frac{n}{2} I(\alpha_0, z_0)} \right]
	\lesssim e^{-(1-\epsilon)\frac{\rho^2}{2\tau^2}}.
\end{equation}
Note that \eqref{eq:subcase2} leads to
\begin{align}
\Expect_{H_0}^{\alpha_0,z_0}
\left[e^{-(1-\epsilon) \frac{n}{2} I(\alpha_0, z_0)} \right]	
& 
\leq 
\Expect_{H_0}^{\alpha_0,z_0}\left[\indc{z_0\in \mathcal{B}_{\epsilon}}e^{-(1-\epsilon) \frac{n}{2} I(\alpha_0, z_0)} \right]
+ \Prob_{H_0}(z_0\notin \mathcal{B}_\epsilon)
\nonumber
\\
& \lesssim e^{-(1-\epsilon)\frac{\rho^2}{2\tau^2}}.
\label{eq:subcase2-upper}
\end{align}
For the first term in $\countrateuppereps{\epsilon'}$, we have
\begin{align}
& \Expect_{H_0}^{\alpha_0,z_0}\left[\indc{z_0\in \mathcal{B}_{\epsilon'}}e^{-(1-\epsilon') \frac{n}{2} I(\alpha_0, z_0)} \right]
\nonumber
\\
& 
= \Expect_{H_0}^{\alpha_0,z_0}\left[\indc{z_0\in \mathcal{B}_{\epsilon'}}e^{-(1-\epsilon) \frac{n}{2} I(\alpha_0, z_0)} 
e^{-(\epsilon - \epsilon') \frac{n}{2} I(\alpha_0, z_0)} \right]
\nonumber \\
& 
= o(1)\, \Expect_{H_0}^{\alpha_0,z_0}\left[\indc{z_0\in \mathcal{B}_{\epsilon'}}e^{-(1-\epsilon) \frac{n}{2} I(\alpha_0, z_0)}  \right]
\label{eq:upperbd-4} 
\\
& = o(1)\, \Expect_{H_0}^{\alpha_0,z_0}
\left[e^{-(1-\epsilon) \frac{n}{2} I(\alpha_0, z_0)}  \right]
\nonumber \\
& 
\ll e^{-(1-\epsilon)\frac{\rho^2}{2\tau^2}}.
\label{eq:upperbd-5}
\end{align}
Here \eqref{eq:upperbd-4} holds since $nI(\alpha_0,z_0)$ is bounded from below uniformly when $z_0\in \mathcal{B}_{\epsilon'}$ by a sequence that diverges to infinity and $\epsilon > \epsilon'$.
The bound \eqref{eq:upperbd-5} is due to \eqref{eq:subcase2-upper}.

Under \eqref{eq:subcase2}, we then have 
\begin{align*}
\countrateuppereps{\epsilon'}
=
\Expect_{H_0}^{\alpha_0,z_0}\left[\indc{z_0\in \mathcal{B}_{\epsilon'}}e^{-(1-\epsilon') \frac{n}{2} I(\alpha_0, z_0)} \right]	
+ e^{-(1-\epsilon')\frac{\rho^2}{2\tau^2}} 
\ll e^{-(1-\epsilon)\frac{\rho^2}{2\tau^2}} 
\lesssim \countrateuppereps{\epsilon}.
\end{align*}
Hence the desired bound \eqref{eq:upperbd-case1} continues to hold.

\medskip

\textbf{Case 2} $\quad$
When
\begin{equation}
	\label{eq:super-small-rate}
\countrateuppereps{\epsilon} < n^{-(1+C/2)} < n^{-1},
\end{equation}
then
\begin{align*}
\Prob\left\{ \ell(\sigma, \wh\sigma) > \countrateuppereps{\epsilon} \right\}
& = \Prob\left\{ \ell(\sigma, \wh\sigma) > 0 \right\} \\
& \leq \sum_{i=1}^n \Prob\{\sigma_i\neq \wh\sigma_i\} \\
& \leq n \countrateuppereps{\epsilon} + 2 n^{-C} \\
& \leq n^{-C/2} + 2 n^{-C} = o(1).
\end{align*}
Here, the second inequality is a union bound. The third inequality is due to \eqref{eq:upper-proof-1} and the last inequality holds due to \eqref{eq:super-small-rate}.
This completes the proof. 

\subsection{Proof of Theorem \ref{thm:lowbd}}

The lower bound can be established by adapting some arguments spelled out in Section 3 of \cite{gao2018minimax}.
We include them below for the manuscript to be self-contained.

Note that for any $0 < \epsilon_2 < \epsilon_1 < 1/2$, we have
\begin{equation*}
\countratelowereps{\epsilon_1} \leq \countratelowereps{\epsilon_2}
\quad \text{and}
\quad
\frac{\countratelowereps{\epsilon_1}}{\countratelowereps{\epsilon_2}} \to 0.
\end{equation*}
Therefore, for any fixed $\epsilon \in (0,1/2)$, we may choose a fixed $\epsilon' > 0$ and a sequence $\delta'=\delta'_n$ such that
\begin{equation}
	\label{eq:delta-require}
	\frac{1}{n}\ll \delta' \ll 1 \quad \text{and} \quad
	\delta'\, \countratelowereps{{\epsilon'}} \gtrsim
	\countratelowereps{{\epsilon}}.
\end{equation}
Then, we choose a $\sigma^*\in[2]^n$ such that $n_u(\sigma^*)\in\left[(1- \delta')n/2,\,  (1 + \delta')n/2\right]$ for $u=1,2$.
Let $\mathcal{C}_u(\sigma^*)=\{i\in[n]:\sigma^*_i=u\}$.
Then we choose some $\wt{\mathcal{C}}_1\subset\mathcal{C}_1(\sigma^*)$ and $\wt{\mathcal{C}}_2\subset\mathcal{C}_2(\sigma^*)$ such that $|\wt{\mathcal{C}}_1|=|\wt{\mathcal{C}}_2|=\ceil{(1-\delta')n/2}$. Define
$$
T=\wt{\mathcal{C}}_1\cup \wt{\mathcal{C}}_2
\quad\text{and}\quad
\mathcal{Z}_T=\left\{ \sigma \in[2]^n: \sigma_i=\sigma^*_i\text{ for all }i\in T\right\}.
$$
The set $\mathcal{Z}_T$ corresponds to a sub-problem that we only need to estimate the clustering labels $\{\sigma_i\}_{i\in T^c}$.

Given any $\sigma \in\mathcal{Z}_T$, the values of $\{\sigma_i\}_{i\in T}$ are known.
Now, we define the subspace
$$
\calP_n^0 = \left\{\calM_n(\sigma, H,\mu,\tau, F_\alpha)
\in \calP_n: \sigma \in \mathcal{Z}_T\right\}.
$$
We have $\calP_n^0 \subset \calP_n$ by the construction of $\mathcal{Z}_T$. This gives the lower bound
\begin{equation}
\inf_{\wh{\sigma}}\sup_{\calP_n}\mathbb{E}\ell(\sigma,\wh{\sigma})
\geq \inf_{\wh{\sigma}}\sup_{\calP_n^0}\mathbb{E}\ell(\sigma,\wh{\sigma})
=
\inf_{\wh{\sigma}}\sup_{\sigma\in\mathcal{Z}_T}\frac{1}{n}\sum_{i=1}^n\mathbb{P}\{\wh{\sigma}_i\neq \sigma_i\}.
\label{eq:reduction}
\end{equation}
The last equality above holds because for any $\sigma^1, \sigma^2\in\mathcal{Z}_T$,
we have $\frac{1}{n}\sum_{i=1}^n\indc{\sigma^1_i\neq \sigma^2_i}=O(\delta') = o(1)$
so that $\ell(\sigma^1, \sigma^2)=\frac{1}{n}\sum_{i=1}^n
\indc{\sigma^1_i\neq \sigma^2_i}$.
Continuing from (\ref{eq:reduction}), we have
\begin{align}
\nonumber \inf_{\wh{\sigma}}\sup_{\sigma \in\mathcal{Z}_T}\frac{1}{n}\sum_{i=1}^n\mathbb{P}\{\wh{\sigma}_i\neq \sigma_i\}
&\geq
\frac{|T^c|}{n}\inf_{\wh{\sigma}}\sup_{\sigma \in\mathcal{Z}_T}\frac{1}{|T^c|}\sum_{i\in T^c}\mathbb{P}\{\wh{\sigma}_i\neq \sigma_i\} \\
\label{eq:ratio}
&\geq \frac{|T^c|}{n}\frac{1}{|T^c|}
\sum_{i\in T^c}\inf_{\wh{\sigma}_i}\ave_{\sigma \in\mathcal{Z}_T}
\mathbb{P}\{\wh{\sigma}_i\neq \sigma_i\}.
\end{align}
Note that for each $i\in T^c$,
\begin{align}
\nonumber & \inf_{\wh{\sigma}_i}
\ave_{\sigma \in\mathcal{Z}_T}\mathbb{P}\{\wh{\sigma}_i\neq \sigma_i\} \\
\label{eq:to-test} &\geq  \ave_{\sigma_{-i}}\inf_{\wh{\sigma}_i}\left(\frac{1}{2}\mathbb{P}_{(\sigma_{-i},\sigma_i=1)}\left(\wh{\sigma}_i\neq 1\right)+\frac{1}{2}\mathbb{P}_{(\sigma_{-i},\sigma_i=2)}\left(\wh{\sigma}_i\neq 2\right)\right).
\end{align}

Now consider any fixed pair
$(\Prob_{(\sigma_{-i},\sigma_i=1)},\Prob_{(\sigma_{-i},\sigma_i=2)} )$.
Let $m_1$ and $m_2$ be the number of nodes with label $1$ and $2$ in $\sigma_{{-i}}$, respectively.
Let $\bar{m} = m_1\vee m_2$.
By the construction of $\mathcal{Z}_T$, we have
\[
\left|\bar{m} - \frac{n}{2} \right|\leq \frac{\delta' n}{2}.
\]
By data processing inequality, the total variation distance between this pair of distributions satisfies
\begin{equation}
	\label{eq:tv-ineq}
\mathrm{TV}(\Prob_{(\sigma_{-i},\sigma_i=1)},\Prob_{(\sigma_{-i},\sigma_i=2)} )
\geq
\mathrm{TV}(\Prob^0_{\bar{m}},\Prob^1_{\bar{m}} ),
\end{equation}
where $\Prob^0_{\bar{m}}$ and $\Prob^1_{\bar{m}}$ refer to the null and the alternative distributions in \eqref{eq:onenode-testing} with $\bar{m}$ observations from either community.
Continuing \eqref{eq:tv-ineq}, we further obtain from Lemmas \ref{lemma:likelihood-ratio-test-equiv} and
\ref{lem:edgecount-rate} that
\begin{equation*}
\mathrm{TV}(\Prob_{(\sigma_{-i},\sigma_i=1)},\Prob_{(\sigma_{-i},\sigma_i=2)} )
\geq
\mathrm{TV}(\Prob^0_{\bar{m}},\Prob^1_{\bar{m}} )
\geq \countratelowereps{\epsilon''},\quad \text{for any $\epsilon'' \in (0, 1/2)$},
\end{equation*}
where we have used the second last display and the fact that $\delta'  = o(1)$.
Together with \eqref{eq:reduction} and \eqref{eq:ratio}, this implies that
\begin{equation*}
	\inf_{\wh{\sigma}}\sup_{\calP_n}\mathbb{E}\ell(\sigma,\wh{\sigma})
    \gtrsim \delta'\, \countratelowereps{\epsilon''},\quad \text{for any $\epsilon'' \in (0,1/2)$}.
\end{equation*}
We complete the proof by observing \eqref{eq:delta-require}.

\end{document}